%% file: main_arxiv.tex
\documentclass{article}

\usepackage{iclr2026_conference}
\usepackage{times}
\iclrfinalcopy

\usepackage{microtype}

\usepackage{setspace}

\usepackage{amsmath}
\usepackage{amssymb}
\usepackage{algorithm}
\usepackage{algorithmicx} 
\usepackage{algpseudocode} 
\usepackage{graphicx} 
\usepackage{amsthm}
\usepackage{physics}
\usepackage{xcolor}
\usepackage{hyperref}
\usepackage{booktabs}
\usepackage{rotating}
\usepackage{xurl}
\usepackage{mathtools}
\usepackage{thmtools}

\usepackage{cleveref}
\usepackage{bm}
\usepackage{wrapfig}

\usepackage{nicefrac}
\usepackage{relsize}

\usepackage{tikz}
\usetikzlibrary{arrows,positioning, shapes.symbols,shapes.callouts,patterns}
\usetikzlibrary{decorations.pathmorphing}
\usetikzlibrary{arrows.meta}

\usepackage{pgfplots}

\usepgfplotslibrary{fillbetween}

\usepackage[shortlabels]{enumitem}

\newtheorem{theorem}{Theorem}[section]

\newtheorem{remark}[theorem]{Remark}

\newtheorem{definition}[theorem]{Definition}

\renewcommand{\paragraph}[1]{\textbf{#1}}

\newcommand{\R}{\mathbb{R}}
\newcommand{\X}{\mathcal{X}}
\renewcommand{\H}{\mathcal{H}}

\renewcommand{\P}{\mathbb{P}}

\newcommand{\set}[1]{\{#1\}}
\newcommand{\todo}[1]{\textcolor{red}{TODO #1}}

\newcommand{\ccolor}[1]{{\leavevmode #1}}
\newcommand{\sgn}{\text{sign}}
\renewcommand{\smash}[1]{#1}

\usepackage{bbm}
\newcommand{\1}{\mathbbm{1}}

\renewcommand{\outer}{h_\text{outer}}
\newcommand{\inner}{h_\text{inner}}

\renewcommand{\outer}{h_\text{out}}
\renewcommand{\inner}{h_\text{inn}}

\newcommand{\HCinner}{\HC_{\text{inner}}}
\newcommand{\HCouter}{\HC_{\text{outer}}}

\renewcommand{\HCinner}{\HC_{\text{inn}}}
\renewcommand{\HCouter}{\HC_{\text{out}}}

\newcommand{\deps}{\#\text{deps}}
\newcommand{\yearsw}{\text{wait yrs}}
\newcommand{\yearsg}{\text{lyg}}

\usepackage{color-edits}
\addauthor{hh}{red}

\DeclareMathOperator*{\Prob}{\mathbb{P}}

\renewcommand{\exp}{\operatorname{exp}} %TODO why?

\newcommand{\Y}{\mathcal{Y}}

\newcommand{\HC}{\mathcal{H}}
\newcommand{\F}{\mathcal{F}}

\DeclareMathOperator{\logistic}{logistic}
\DeclareMathOperator{\logit}{logit}

\DeclareMathOperator{\interior}{int}

\newcommand{\wv}{\bm{w}}

\title{Towards Cognitively-Faithful Decision-Making Models to Improve AI Alignment}
\author{Cyrus Cousins\thanks{Correspondence to \texttt{originalcyruscousins@gmail.com}}\\Duke University
\And
Vijay Keswani\\IIT Delhi
\And
Vincent Conitzer\\CMU
\And
Hoda Heidari\\CMU
\AND
Jana Schaich Borg\\Duke University
\And
Walter Sinnott-Armstrong\\Duke University
}

\renewcommand{\todo}[1]{}

\newif\iflongversion
\longversionfalse
\begin{document}

\maketitle

\begin{abstract}
Recent AI trends seek to align AI models to learned \emph{human-centric objectives}, such as personal preferences, utility, or societal values.
%Recent AI work  trends towards incorporating human-centric objectives, with the explicit goal of aligning AI models to personal preferences and societal values.
Using standard preference elicitation methods, researchers and practitioners build models
%to simulate
of
human decisions and judgments,
to which AI models are aligned.
% which are then used %utilized %downstream
% to align AI behavior with that of humans. 
% 
However,  %the hypothesis classes
% models
% commonly used in such elicitation processes
standard elicitation
%models
methods
often
%do not
fail to
capture the
%true
cognitive processes
%human use to make their decisions,
behind human decision making,
such as %when people
%the use of
\emph{heuristics} or simplifying \emph{structured thought patterns}. % associated with decision making. 
% As a result, models learned from people's decisions often do not align %well with human 
% with their cognitive processes, and can not be used to validate the learning framework for generalization to other decision-making tasks.
% 
To address this failure, we take an axiomatic approach to learning \emph{cognitively faithful} decision processes from pairwise comparisons. 
% Motivated by prior extensive work in psychology on decision-making heuristics and using natural axioms, 
Building on the %vast literature characterizing
literature analyzing %on the
cognitive processes that
shape
%contribute to
human decision-making,
%, and recent work characterizing such processes in pairwise comparison tasks,
%and pairwise comparisons,
we derive a
%class of models
model class
in which
%individual
features are
first \emph{processed} with learned rules, % and compared across options, and then the
%then aggregated, processed features are
%\emph{processed}
then \emph{aggregated} via a fixed rule, such as the Bradley-Terry rule, to produce a decision. % aggregation or a tallying heuristic. 
This structured processing of information ensures that such models are realistic and feasible candidates to represent underlying human decision-making processes.
We demonstrate the efficacy of this modeling approach by learning interpretable models of human decision making in a kidney allocation task, and show that our proposed models match or surpass the accuracy of prior models of human pairwise decision-making.
% for synthetic and real-world data from several participants 

\todo{[fewer adjectives] [interpretable vs. cognitively faithful] [less hedging]}

\end{abstract}

% \clearpage
\section{Introduction}
% \vspace{-0.05in}

% There are several advantages to building computational models of human cognition.
% 
%They
Computational models of human cognition
allow us to quantify and study the factors that influence our decisions, and when accurate, 
they help explain commonalities in human decision processes across different tasks \citep{gigerenzer2000simple, crockett2016computational}.
% contribute to AI  that is more generalizable to new decision-making contexts than AI that less faithfully reflects humans' true decision-making processes.
% 
% They allow us to quantify and study the factors that influence our choices
% \citep{gigerenzer2000simple, crockett2016computational},
% These heuristics may not always be obvious; they are used in situations that require rapid decisions, yet are honed with life experience, and understanding them can provide insights into human decision processes.
% 
% Computational models can also 
% and can also act as explanatory frameworks,
% % for human decisions, 
% allowing people to validate the factors affecting their own and others' decisions \citep{fresco2012explanatory}.
% and can even serve as a \textit{sounding board} to assess and refine one's decision-making process.
% 
% , the recent surge in AI capabilities comes with AI's deployment to assist human decisions.
% in domains of moral import, e.g., decisions in healthcare, autonomous vehicles, resource allocation, policing, and criminal justice settings.
% 
This may also be why computational models to predict human behavior have found applications in personalized AI tools in various settings, e.g., healthcare, autonomous vehicles, and resource allocation, to increase users' confidence that AI will be aligned with their preferences %\citep{ji2023ai,ng2000algorithms,christiano2017deep,kim2018computational,noothigattu2018voting,noothigattu2019teaching,lee2019webuildai,freedman2020adapting}.
\citep{ji2023ai,kim2018computational,noothigattu2019teaching}. %Cut 4 repeated citations 
% 
% 
% Such models have recently found applications in the domain of personalized AI tools as well.
% % of these models is associated with the surge in AI capabilities.
% With AI being used to assist with human decisions in various settings, e.g., healthcare, autonomous vehicles, and resource allocation, recent works put forward computational models of human decision-making as blueprints for AI tools that are better aligned with humans \citep{ji2023ai,ng2000algorithms,christiano2017deep,kim2018computational,noothigattu2018voting,noothigattu2019teaching,lee2019webuildai,freedman2020adapting}. 
% experts make moral decisions as well as provide mechanisms to guardrail against potential unethical AI behavior.

% 
Despite
%these
promising use cases, current AI models of human decision-making typically do not attempt to faithfully replicate
%human decision
cognitive
processes.
% 
% Limited computational 
% arise from misspecified optimization formulations (e.g., \textit{reward hacking} by RL agents), 
Modern AI tools built using preference elicitation \citep{lee2019webuildai,awad2018moral,johnston2023deploying,freedman2020adapting,rafailov2023direct}
% , preference optimization \citep{rafailov2023direct}, 
and reinforcement learning
%methods
\citep{ng2000algorithms,christiano2017deep,kaufmann2023survey} implicitly assume that a reward/objective model based on a prespecified hypothesis class can accurately predict human decisions, but such tools are agnostic about how faithful these models are to
%humans' actual decision
cognitive
processes.
However, recent works %cast doubt on this approach, highlighting
highlight the lack of suitable 
% reward functions 
% or 
hypothesis classes 
to capture human choices or their values \citep{stray2021you,boerstler2024stability,leike2018scalable}.
Unsuitable modeling classes introduce the possibility of inappropriate optimization formulations and   \textit{reward hacking} \citep{pan2022effects,hubinger2019risks}, and can lead to arbitrarily erroneous inferential models learned from human responses \citep{zhuang2020consequences,pan2022effects}.
% theoretically and empirically noted several problems with such assumptions, such as the lack of suitable reward functions to model human choices, the possibility of \textit{reward hacking} due to misspecified rewards  
% lacking in prominent ways that limit their applicability in real-world scenarios.
% 
% Focusing on our \textit{inductive inferential capabilities} in comparative settings \citep{brighton2006robust}, the hypothesis classes used to simulate the inferential processes that individuals employ when choosing between options are rarely grounded in theories of human cognition.
% 
% Crucially, AI models developed in prior works to simulate human decisions do not seem to process information in a manner similar to humans.
% 
% They assume that people have stable and static moral preferences, which significantly deviates from reality, given that our moral preferences are continuously evolving.
% 
% Inappropriate hypothesis classes often lead to erroneous inferential models learned from human responses,
Additionally, such misaligned models are limited in their explanatory power, 
% and \textit{trustworthiness}, 
% 
% As suggested by \citet{jacovi2021formalizing}, misalignment between AI and human reasoning processes 
undermining any \textit{intrinsic trust} placed in the AI's decisions \citep{jacovi2021formalizing}.
% due to misalignment with human decision processes and, hence, reduce the trust placed in their decisions.
% 
% 
% of models trained using such classes due to several sources of misalignment with the relevant user.
% Hypothesis classes that have been 
% 
These problems are further exacerbated when AI is used in high-stakes domains, including moral domains such as healthcare and sentencing \citep{sinnott2021ai,kim2018computational},
% (with actions significantly harming/benefiting others \citep{sinnott2021ai,kim2018computational}), 
where 
% recent works have found that 
stakeholders often expect an AI to justify its decision in a similar manner and to the same extent as humans \citep{lima2021human}.
\ccolor{Modeling done in prior moral decision-making contexts either favors simple model classes --- e.g., linear models and decision trees \citep{lee2019webuildai,noothigattu2018voting,freedman2020adapting}, whose fit for human decision-making is highly context-specific \citep{ganzach2001nonlinear,zorman1997limitations}, 
% 
% which do not accurately represent human decision process, 
or employs uninterpretable classes --- e.g., neural networks or random forests \citep{wiedeman2020modeling,ueshima2024discovering} --- whose decision processes cannot be validated due to their uninterpretable designs.%
\todo{move to related work?}}
% 
%A recent qualitative study on kidney transplant allocation decisions by \citet{keswani2025can} noted this \textit{process misalignment} as a crucial issue in trusting AI to decide who gets an available kidney, with one study participant expressing the following concern about AI: 
%``\textit{my concern, they don’t think like a human [..] what they might think should be ranked as priority I may not think the same way}'' \citep{keswani2025can}.
A qualitative study by \citet{keswani2025can} note this \textit{process misalignment} as a crucial issue around AI-trust in morally-salient tasks like kidney transplant allocation, with one study participant expressing concern: % about AI: 
%``\textit{my concern, they don’t think like a human [..] what they might think should be ranked as priority I may not think the same way}'' \citep{keswani2025can}
\emph{``%my concern,
they don’t think like a human\ldots{} what they might think should be ranked as priority, I may not''.}
% my concern, they don’t think like a
% human [..] what they might think should be ranked as priority I may not think the same way”

%The above examples show
Such human concerns highlight
that, for certain real-world applications, building trustworthy AI tools requires accurately simulating human decision processes.
% 
%Yet, the challenge here is that
However,
human decision processes
%often differ from the mechanisms represented
are often not captured
by standard hypothesis classes.
% 
%In particular,
For example,
consider humans' use of decision \textit{heuristics}, defined by \citet{gigerenzer2011heuristic} as a strategy ``\textit{that ignores part of the information, with the goal of making decisions more quickly, frugally, and/or accurately than more complex methods}.''
% 
%The qualitative study on kidney allocation decisions quoted above documents
\Citet{keswani2025can} document several heuristics people use when deciding which patient should get an available kidney.
% 
%For instance, their
Their
study participants often used a form of \textit{hiatus heuristic}, employing \emph{thresholds} to isolate relevant patient feature information; e.g.,
some treated ``number of dependents'' as a binary (none or any), disregarding the actual \emph{quantity}.
%some transformed patients' number of dependents into a binary feature that simply captured whether the patient has dependents or not.
%
%Others
Similarly, rather than contrast patients holistically, some participants used simple aggregation methods, such as the \textit{tallying heuristic}, which simply counts the number of features favoring each patient to make
%the final choice.
their choices.
% (i.e., compare two or more options by simply counting the number of factors favoring each option)
% Examples can include simple additive methods, such as the \textit{tallying heuristic} (i.e., compare two or more options by simply counting the number of factors favoring each option), or non-linear decision rules, such as  \textit{fast-and-frugal trees} \citep{gigerenzer2000simple}.
% 
% While the use of these heuristics in human decision processes is well established, supervised learning approaches used over human decision datasets focus only on predicting the decisions and not on learning the heuristics they use.
While the role of such heuristics in human decision processes is well established \citep{gigerenzer2011heuristic,shah2008heuristics,gigerenzer2000simple}, 
% Despite their prominent role, 
models learned from human decisions using standard hypothesis classes often don't capture or explain the heuristics people use to make decisions.
% are rarely evaluated in their ability to capture these heuristics.
% 
For example, if a decision-maker uses the threshold-based heuristic mentioned above, linear models would fail to capture it, while neural networks or multivariate monotonic regressors may learn the heuristic, but would fail to explain its role in the decision process, instead representing it as some approximately-equivalent but hopelessly opaque process.
% \citep{keswani2025can}.
% 
% This is apparent from hypothesis classes used in prior works that are rarely justified in their ability to capture the underlying human decision-making process.
% 
% Nevertheless, 
To explain or simulate human decision processes accurately, we need methods to learn and represent these heuristics (and other decision-making nuances).
% from human decisions.
% 
To that end, this paper explores hypothesis classes that more accurately capture the cognitive processes people use to make their decisions in pairwise comparison settings.

This work presents
a learning framework for computational models of pairwise human decision-making, % in  pairwise comparisons,
which is motivated by
%, and aligned with,
prior work on human cognition.
% We present
%a set of
% natural decision-making axioms (\cref{sec:models}), and from them, we derive a class of feasible \emph{cognitively faithful} decision-making models. %``cognitively faithful decision making models.''
% 
% [motivated by prior e work, heuristics, existing ai don't capture without complexity explosion.
% explainability? learnability?
% ]
Classical \emph{theoretical} axiomatic models of choice %, such as expected utility theory,
impose strong structural assumptions on preferences, %including full transitivity, independence of irrelevant alternatives, and stable cardinal utility representations
e.g., strong \emph{transitivity} concepts and \emph{independence of irrelevant alternatives} \citep{vonNeumann1944theory,luce1959individual}.
% This would seem form a strong basis for this work --- 
However, decades of \emph{empirical work} demonstrate that these assumptions are systematically violated in human decision-making, particularly in comparative and context-dependent settings \citep{tversky1969intransitivity,tversky1974judgment%,erev2010choice
}.
%Rather than reinstating these classical axioms or abandoning axiomatic structure altogether, %our approach adopts a third position:
%we
We thus
propose a set of \emph{substantially weaker} axioms (\cref{sec:models}) that constrain how comparative judgments are made, and from these axioms, we derive a class of feasible \emph{cognitively faithful} decision-making models.
% , without committing to global optimality, linear utility, or context independence. 
This allows us to retain interpretability and theoretical grounding
%while remaining compatible with
without contradicting
empirically observed heuristic decision processes.
Indeed, because our axioms %are, for the most part, relatively simple and intuitive, 
are weaker than classical alternatives, 
% and relatively simple, 
%thus they
%do not %explicitly 
%yield explicit 
%and consequently,
they do not
\emph{fully specify} a decision making process, but rather,
%we show that our axioms
they \emph{constrain the feasible space} of decision making% models
.
The approach is pragmatic: 
%To some extent, our axioms are prescriptive,
Our axioms are somewhat prescriptive,
as they represent normatively-desirable attributes of decision making, but %there is also
they also have a descriptive element, as they relax existing %models of decision making
classical axioms
that were themselves attempts
%attempt
to describe human behavior. % are based on human behavior.
Similarly, the violation of our axioms is \emph{not necessarily} a critical failure of human reasoning, but as relatively
%digestible conceptual
simple claims, the axioms can be understood more easily than complex models, and they can be studied qualitatively and assessed for their replicability in AI systems that 
% are trained to 
%learn from and
align to human decisions.

Given any pairwise comparison between two options $x_1, x_2 \in \R^d$,
%derive hypothesis classes that 
%we model the decision-maker's response to this comparison as %consisting of
we characterize the decision-maker's preference by
a
%sequential
two-stage process (\cref{sec:decision_rules}).
The first stage
%of our two-stage model
captures the \textit{editing rules} that the decision-maker employs over individual features $\smash{x_1^{(j)}\!,\, x_2^{(j)}}$ to process/simplify/transform the information presented in feature $j \in [d]$.
% 
%The decision rules employed can be different for
Each feature $j$ may have its own decision rule, and these rules work towards transforming $\smash{x_i^{(j)}}$ in a manner that reflects the feature's contribution to the decision-maker's final choice (e.g., whether the decision-maker %applies a threshold on this feature or ``zeroes'' it out
thresholds, ignores, linearly transforms, or otherwise processes the feature).
We allow this transformation to be \textit{contextual} in nature, whereby the value of one or more features can influence the transformation used for another feature.
Such \textit{conditional transformations} capture feature interactions in our hypothesis class, increasing the models' expressiveness while ensuring that we still learn feature-level decision rules.
% with potentially different rules for different $i \in [d]$, 
% 
The second stage captures the decision rule that aggregates the processed features to choose the \textit{preferred option}.
% in the pairwise comparison.
% 
This aggregation rule can be as simple as the \textit{tallying heuristic} \citep{gigerenzer2000simple}
% (i.e., counting how many processed features favor each option \citep{gigerenzer2000simple})
or as complicated as \textit{Bradley-Terry} (\citeyear{bradley1952rank}) \textit{aggregation}
% , commonly employed in the preference elicitation literature to model 
for probabilistic preferences. %\citep{bradley1952rank}. 
% 
% Crucially, the decision rules we consider in both stages are motivated by well-known heuristics from prior works on psychology, and 
% 
% 
Beyond the motivations from prior works in cognitive science and preference modeling,
% on cognitive science,
% heuristic decision-making, 
% we also show that the two-stage process emerges from natural assumptions on the decision-making process
% (e.g., symmetry and transitivity) 
% of human decision-making processes for pairwise comparisons 
% (\cref{sec:axioms}).
% 
% We first 
% we present a simple 
we show that our axiomatic basis 
% of human decision-making processes
% 
%We find that this simple
% The axiomatic basis 
% that 
yields the two-stage model class (\cref{sec:axioms}).
Moreover, more stringent assumptions yield special cases of interest, such as logistic regression, probit regression, and monotonic decision rules.
% 
%We also
Finally, we assess the empirical efficacy of our proposed framework over synthetic and real-world datasets in the kidney allocation domain (\cref{sec:experiments}), a domain where users have said cognitive process alignment is particularly important. %, as discussed earlier. 
% 
% We use datasets containing people's responses to kidney allocation scenarios for real-world data evaluations. 
%For this setting, 
%In this domain, we show
We find
that our method
%can be used to 
can
learn models that explain the decision rules that people use for kidney allocation decisions, while ensuring that the learned model has similar or better predictive accuracy than other modeling approaches.
%are a good fit for
% align with %well to
% their responses.

% 
% Overall, in this work we put forward a hypothesis class for learning from human responses to pairwise comparisons, with the aim that models in this class are better aligned with human decision processes compared to models from other standard hypothesis classes, e.g., linear models or neural networks.
% % 
% Unlike prior works that focus on specific heuristics \citep{brighton2006robust,csimcsek2015learning}, learning with our proposed hypothesis class allows us to discover the heterogeneous decision rules individuals use in their decision processes and, as a result, explain their decision processes more accurately.
% % 
% We also assess the empirical efficacy of our proposed framework over synthetic and real-world datasets in the kidney allocation domain (\cref{sec:experiments}). 
% % 
% For this setting, we show that our method can be used to learn models that infer and explain the decision rules and heuristics that people use when making kidney allocation decisions, while ensuring that the learned rules are a good fit for their responses.

%\paragraph{Related Work}
%\subsection{Related Works}
\vspace{-.125cm}
\section{Related Work}
\vspace{-.125cm}
% 
% Our work contributes to the literature on learning from pairwise comparisons data \citep{chen2004survey,borodinov2020method}.
% 
%Methods for learning models based on
Methods for learning preferences from
%choices over options
comparative choices
% among options
have been proposed in many contexts, such as for recommender systems \citep{kalloori2018eliciting, guo2010real, chen2004survey}, language model personalization \citep{rafailov2023direct,jiang2024survey,ziegler2019fine,ouyang2022training},
% in computer science domains to
and
%to build
explanatory frameworks 
% for decision-making 
in psychology \citep{lichtenstein2006construction,ben2019foundations,charness2013experimental}.
% This paradigm of preference elicitation has received wide interest 
% , 
%Additional recent use cases 
% of decision modeling also 
%lie in the domain of
Recently, preference learning and elicitation of stakeholder preferences has flourished in the context of
\textit{human-%centered
centric AI} and AI alignment  %, to elicit and encode stakeholder preferences in AI
\citep{jiang2024survey, ji2023ai, capel2023human, feffer2023preference, kim2018computational, johnston2023deploying, sinnott2021ai, Liscio2024, lee2019webuildai}.
% as a step towards building AI systems whose behavior encodes the elicited preferences.
% 
% This approach has been explored to build \textit{morality} into AI systems as well, by 
% training the AI to simulate the elicited human moral judgments \citep{feffer2023preference, kim2018computational, johnston2023deploying, sinnott2021ai, Liscio2024, lee2019webuildai}.
% 
\todo{Across these use cases,}%
The usability of learned models in real settings relies on both their predictive accuracy and their alignment to humans' decision-making \citep{mukherjee2024optimal,keswani2025can, capel2023human,lima2021human}.
% 
% %Across these use cases, 
% In these settings,
% % these works on modeling human judgments have primarily focused on building computational models that the simulate the stated judgments, with their 
% learned models are evaluated by their predictive performance on held-out/future data \citep{chen2004survey, mukherjee2024optimal}. 
% % 
% % Here, the performance is judged by the predictive loss or accuracy of the learned model. 
% % 
% Beyond predictive accuracy, aligning an AI to human decision-making {\em processes}
% % (and not just the decisions) 
% has also been highlighted as
% %an important
% a key
% requirement for real-world acceptance and usability %of these tools
% \citep{keswani2025can, capel2023human,lima2021human}. 
% 
% This is specifically true in moral domains, where prior studies have shown that people expect an AI to justify its moral decisions to the same extent as humans \citep{lima2021human, keswani2025can}.
% 
%Yet, as mentioned earlier,
However, modeling strategies in most prior AI alignment work are agnostic about whether they actually capture human decision-making processes.
Our work fills this gap by proposing modeling classes explicitly motivated by the decision rules people
%have previously reported using
report for pairwise comparisons.
\iflongversion

\fi
Like this work, \citet{noothigattu2020axioms} study the class of binary decision rules
% for pairwise comparisons 
arising from %a dataset of
pairwise comparisons. However, their axioms
% , namely \emph{Pareto efficiency}, \emph{monotonicity}, \emph{pairwise majority consistency}, and \emph{separability}, 
concern MLE estimation over %a given dataset
data, and thus are \emph{distribution sensitive},
% to the given distribution, 
whereas our axioms dictate laws as to how probabilistic preferences must behave \emph{over %an entire
their domain}. 
% They find conditions under which maximum likelihood estimation of a random utility model satisfies Pareto efficiency and monotonicity, but not pairwise majority consistency or separability. 
%\citet{noothigattu2020axioms}
They also do not %fully
characterize model classes
%that arise
arising
from their axioms, but rather study whether known classes satisfy them, making their work more taxonomically descriptive than prescriptive\todo{is that fair?}.
\Citet{ge2024axioms} continue this line of inquiry, again with axioms related to estimation from datasets.

There have been other efforts to computationally model human decision-making.
\citet{bourgin2019cognitive} propose models bounded by theoretical properties of human decision-making that are fine-tuned with real-world data,
%but their approach  aims
but they aim
for accurate predictions, %without necessarily resulting in
%not necessarily
%obtaining
rather than
cognitive fidelity.
% or
% interpretable models.
% 
% The work of Rosenfeld and Kraus on “hybrid models” shares similar motivations as ours, in terms of developing methods that accurately model human cognitive processes. Their use of “Aspiration Adaptation Theory” is very relevant in 
\iflongversion
\ccolor{
todo: rephrase \citet{rosenfeld2012combining} model agents solving optimization problems, sharing our motivation of accurately simulating cognitive processes. For the settings they consider, the optimization objectives are known. In contrast, we model human decision-making in comparative settings where the decision objectives are unavailable to the elicitation framework and can differ across users. 
}\fi%
\citet{peterson2021using} use neural networks to implement and test cognitive models of
%over data containing
participant choices for %several
pairs of gambles.
This analysis provides information on the \textit{fit} of various cognitive models, but does not provide a mechanism to \emph{learn} the decision-making process from an individual's data.
% They compare and report the fit of various cognitive models encoded i
% 
\ccolor{\citet{plonsky2017psychological} and \citet{payne1988adaptive} use
%theories of human psychology to
psychological theories to
select feature transformations or simulate prespecified heuristics %appropriately
to obtain gains in predictive power by curating the learning tasks to be more cognitive aligned with human decision processes.
% in risk-based choice settings before using them.
% 
However, the applicability of these methods is limited 
% in settings where 
when feature transformations are unknown \emph{a priori} and vary across individuals. 
}%
% for downstream learning tasks.
% 
Our work also uses relevant works in psychology to identify appropriate decision rule assumptions, but %additionally learns
we learn feature transformations from human decisions. %response data.
% further extends the work of \citet{plonsky2017psychological} by learning transformations through data.
% 
Other works use neural networks alongside theory-driven cognitive models to simulate human decision-making \citep{fintz2022using,lin2022predicting},
%but the use of neural networks makes it difficult to interpret the learned decision processes.
but such models of learned decision processes are largely uninterpretable.
% used by the decision-makers.
% 
% Focusing on heuristic decision-making, 
Several studies also illustrate 
% show 
that simple heuristic-based models predict human responses
%very
well in classification %tasks 
\citep{holte1993very,csimcsek2015learning, brighton2006robust, czerlinski1999good,dawes1979robust}, supporting our goal of learning
%these heuristics from people's decisions.
heuristics from decisions.
Yet, the heterogeneity of heuristics used by different people makes it difficult to find a single model to accurately simulate the responses for an entire population.
Our work addresses this issue by learning %the decision rules for any given individual decision-maker in a supervised manner. 
individual-level decision rules from data.

% Focusing on heuristic decision-making, \citet{brighton2006robust} show how ``fast-and-frugal heuristics'', a set of cognitive heuristics put forward by \citet{gigerenzer2000simple}, outperform several ML approaches in simulating pairwise-comparison responses. 
% However, fast-and-frugal heuristics primarily determine the dominant option 
% (i.e., the second stage of our model) 
% and do not provide insight into contextual feature transformations that people use.
% % 
% Many other works have also recognized how simple heuristic-based models predict human responses very well in classification tasks \citep{holte1993very,csimcsek2015learning, brighton2006robust, czerlinski1999good}, which supports our goal of learning the decision rules people use.
% % 
% Yet, the heterogeneity of  heuristics used by different people makes it difficult to find
% % put forward 
% a single accurate model to simulate the responses for an entire population \citep{keswani2025can}.
% % 
% Our work contributes towards addressing this gap by learning the decision rules for any given individual decision-maker in a supervised manner. 
% Other works also discuss 

% Also related is the explainability and interpretability literature. 

\section{%
%Our Model
A Model of Cognitively-Faithful Decision-Making%
} \label{sec:models}
% \vspace{-0.1in}

%Consider the setting where
In our setting, a decision-maker is presented with a pairwise comparison $(x_1, x_2)$ between two options
from the domain $\X \subseteq \R^d\!$, i.e. each with $d$ descriptive features. %, where $d > 0$ denotes the number of features describing each option.
% 
% Suppose that each element in the comparison is described using $d \in \R_{>0}$ features. 
For each $i \in [d]$, let $\X_i \subseteq \R$ denote the input space for feature $i$, such that $\X \doteq \X_1 \times \cdots \times \X_d$.
% 
% 
% 
% 
% Our work operates in the domain where a decision-maker is presented with a moral dilemma posed as a pairwise comparison between two options, with the choice of the decision-maker reflecting potential benefits/harms caused to others in the given moral domain.
% 
% Suppose one 
% is presented with 
% A pairwise comparison is between two elements from the same underlying input/feature space and the decision-maker is asked to express a preference by deciding which one we prefer.
% 
% Our goal is to understand how the decision-maker processes the presented features to reach their decision in the given moral domain.
% 
% \paragraph{Decision-Making for Pairwise Comparisons.}
% 
% 
% With this setup, given any pairwise comparison $(p_1, p_2)$, each element in this comparison is represented by a $d$-dimensional vector from the space $\X \doteq (\X_1, \cdots, \X_d)$.
% 
The decision-maker's 
% stochastic 
response function 
% to pairwise comparisons 
%is represented by\linebreak[1]
$H: \X \times \X \rightarrow [0,1]$
denotes the probability of choosing the first option $x_1$ for any given pairwise comparison $(x_1, x_2)$. 
% 
% We assume the output to lie between 0 and 1, accounting for decision stochasticity, uncertainty, or lack of confidence in one's response.
% 
% Our goal is to decipher $H$ using the decision-maker's responses to several pairwise comparisons.
% 
% 
% \paragraph{Dataset.}
% Function $H$ 
% % used by one for decision-making 
% will usually be unknown in most settings.
% % 
% Instead, what we generally have access to is one's response to several pairwise comparisons.
% 
% 
% In particular, 
To learn $H$ from observed data, suppose we have a dataset $S$ containing the decision-maker's responses to $N$ pairwise comparisons of the kind $(x_1, x_2, r)$, 
% denoted by $S = (x_1^{(j)}, x_2^{(j)}, r^{(j)})_{j=1}^N$, 
where $r \in \{0,1\}$,
is the binary decision, with $r=1$ if the decision-maker deterministically chose $x_1$, and 0 otherwise.
% from the decision-maker.
% (based on thresholding $H(p_1^j, p_2^j)$ at a certain (unknown) value).
% 
%To focus on a singular data-generating process, we assume that the dataset $S$ contains responses from only one decision-maker and not a mixture of responses from many decision-makers.
% 

% Additionally, to 

% \paragraph{Learning and Evaluation.}
% 
Given dataset $S$, we aim to learn a model $\smash{\hat{H}}$ that accurately simulates the decision-maker's response function.
Taking a general learning approach to this problem, given a hypothesis class of models $\H$, we can search for a model from $\H$ that minimizes the predictive loss over $S$\iflongversion
; i.e., 
% \[
$\smash{\hat{H}} \doteq \smash{\arg\min_{H' \in \H}} \P_{(x_1, x_2, r) \in S}[H'(x_1, x_2) \neq r].$
TODO product???
\else
.
\fi
% \enspace.
% \]
% 
% 
% The general ML approach to solving this problem is starting with a hypothesis class $\H$, a predictive loss function $\L$, and searching for a model from $\H$ that minimizes the predictive loss over $S$ with respect to $\L$. That is 
% \[\hat{H} \doteq \arg\min_{H' \in \H} \E_{(p_1, p_2, r) \in S}\L (H'(p_1, p_2), r).\]
% % 
% Of course, the primary complexity often lies with defining an appropriate hypothesis class $\H$ and loss function $\L$ so that the learned model accurately reflects the processes used by the (unknown) function $H$.
% % 
% In this work, we will focus on appropriate definitions of $\H$ that can capture the way humans process the given features to make their decision.
% % 
% To that end, for the sake of brevity, assume that $$\L (H'(p_1, p_2), r) = \mathbf{1}(H'(p_1, p_2) = r).$$
% 
% 
A common approach to learning $\smash{\hat{H}}$ is to start with \textit{standard} hypothesis classes, such as classes of linear functions,
%$\H_{\text{linear}}$,
decision trees,
%$\H_{\text{tree}}$,
or neural networks. % $\H_{\text{NN}}$. %$\H_{\text{neural}}$.
% 
% While these might be able to fit the data fairly well (depending on the feature space and model parameters), 
However, as discussed earlier, the assumptions of these classes are likely not faithful to the cognitive processes people truly use to reach their own decisions \citep{keswani2025can}.
As such, when using these classes, the resulting model can be difficult to validate and may lead to erroneous predictions in cases where the model class cannot capture the computation required to reach the ``correct'' decision.
We aim to identify a hypothesis class that yields the most accurate model while faithfully representing humans' actual decision-making processes. Such models have the added benefit of being interpretable for human users.
To come up with cognitively faithful computational models,
% of human decision-making, 
we first discuss the computational properties commonly observed in human decision-making processes, and then use these properties to construct an appropriate hypothesis class.
% for our setting.

\subsection{Rule-Based Decision-Making Models} \label{sec:decision_rules}
% 
%As discussed earlier,
Prior works on cognitive models for human decision-making point to a reliance on \textit{decision rules} and \textit{heuristics} in comparative settings \citep{gigerenzer2000simple,gigerenzer2011heuristic,brandstatter2006priority,kahneman2013prospect,kahneman2005model}.
% 
% Our own qualitative findings suggest similar decision rules also being applied in moral decision-making settings.
% when they have to make decisions relying on their intuitions.
% 
% This prior evidence of the use of decision rules is our motivation for forwarding 
Based on this literature, we construct hypothesis classes that consist of hierarchical decision rules.
Our proposed strategy models decision-making for pairwise comparison of options $(x_1, x_2)$ as a two-step process. 
In the first step, the decision-maker processes each feature in $x_1$ and $x_2$ to \textit{edit} or transform the information presented by this feature.
In the second step, the decision-maker combines the edited information from all features to select the \textit{dominant} option. 
Our hypothesis class will consist of functions that reflect this hierarchical process.

%(1)
\paragraph{Editing Rules} The decision rules or heuristic functions used in the first step are referred to as \textit{editing rules}.
    {Editing rules} operate at the level of individual features of each element in the given pair and operationalize how a decision-maker processes the given feature value.
    Let $\inner^{i}: \X_i \rightarrow \X_i'$ denote the editing rule for feature $i$, with $\X_i'$ the output domain post-editing.
    Examples of editing include feature simplification (e.g., zeroing out features considered irrelevant), transformation (e.g., log transformation for features with ``diminishing returns'' or scaling), or leaving the feature unchanged.
    % 
    % Let $\inner^{i,\omega}: \X_i \rightarrow \X_i'$.
    % 
    % Here $\X_i'$ is the output domain for feature $i$ post-editing and $\omega$ denotes the \textit{context} in which this editing rule is applied.

    While these rules are often \emph{structurally simple}, reflecting their use to reduce cognitive load \citep{gigerenzer2011heuristic},
    % (as noted in the examples below), 
    there can be significant heterogeneity in the editing used in different contexts. %TODO repeated below
    For example, prior works %have noted
    note \textit{feature interactions}, where the editing rule used for the $i$th feature $x^{(i)}$ can depend on the values of other features of $x$ \citep{keswani2025can}.
    To account for this, we will consider the choice of editing rule conditional on the \textit{decision context} features, denoted by the values of features in a set $\omega \subseteq [d]$. 
    %
\iffalse
    In other words, with context $\omega$, the editing rule for any feature $i$ can be different for different 
    % values of $\omega$ features 
    values $x^\omega$.
\fi
    % In other 
    % Here, $\omega$ denotes the information from $x_i$ that is relevant to choosing the editing rule for any feature value $x_1^j$.
    % additional information of option $p_i$ that the decision-maker employs when editing any feature of this option.
    % 
    % 
    %In the weakest case,
    On one extreme, $\omega = \emptyset$, implying that each feature is edited independently (i.e., no feature interactions), as $\omega$ grows, model complexity increases as conditional interactions involve more features, and
    %In the strongest case
    the other extreme, $\omega = [d]$, implies that the choice of editing function for each feature in $x$ can depend on the values of all other features. % of $x$.
    \todo{Hence, in our simulations, we will limit $\omega$ to just one feature.}%
    %
    %\footnote{Note that the larger the context $\omega$, the larger the number of editing rules to be learned, and the higher the learning complexity. Hence, in our simulations, we will limit $\omega$ to just one feature.}
% 
    % be the \textit{most important feature} in the decision-making process.}
    % 
    %
    To account for this context $\omega$, % (with slight abuse of notation),
    for any option $x \in \X$, we will denote the contextual editing rule operating over feature $i$ by $\smash{\inner^{i,x^{\omega_i}}}: \X_i \rightarrow \X_i'$, where $\omega_i = \omega \setminus \{i\}$.
    % 
    % For simplicity, The editing rules $\inner^i$ for
    % 
\todo{TODO ?}
    
    \begin{remark}[Examples of Editing Rules]
        The use of editing rules is well-supported by 
        %several works
        literature
        in behavioral economics and social psychology.
        \citet{montgomery1983decision} describes %it as
        the phase of ``separating relevant information from less relevant information which can be discarded,%\ldots,% in subsequent information processing.
        '' 
        and \citet{kahneman2013prospect} argue that %the editing phase allows %one to make the decision
        this allows for
        decision-making based only
        on the most essential information. 
        % 
        % Beyond reducing cognitive load \citep{gigerenzer2000simple,rieskamp1999people}, 
        \iflongversion
        Several works %have also noted
        also note
        significant heterogeneity in editing rules.
        % used by different decision-makers.
        \fi
        In some settings, %they
        editing rules
        represent \emph{feature importance assignment} 
        % or determination of %attributes considered irrelevant
        % \emph{irrelevant attributes} 
        \citep{payne1976task, gigerenzer2000simple}.
        In other cases, editing reflects \emph{feature transformations} to isolate
        %the information that is relevant
        information relevant
        to the %final choice
        task \citep{ajzen1996social, shah2008heuristics}, e.g.,
        %explicitly using thresholds to discretize the feature
        thresholding to discretize features,
        or %implicitly considering log transformation
        $\log$-transformation
        to model diminishing returns %associated with increasing feature values 
        \citep{tversky1972elimination, kubanek2017optimal}.
        % 
    % 
        % In most real-world scenarios, the editing functions are known to work toward reducing the complexity of the presented information.
        % % 
        % This is partly to reduce the cognitive effort associated with fast decision-making when presented with information along multiple dimensions and also to ensure that one focuses on the information that is most pertinent to the final decision \citep{gigerenzer2000simple,rieskamp1999people}.
        % % 
        % Hence, in most cases, we will assume that the output domain $\X_i'$ is similar or smaller in size than $\X_i$.
    \end{remark}
% 
%(2)
\paragraph{Dominance Testing}
In the second step of the decision-making process, the edited features are compared across the two options ($x_1, x_2$) and then combined to reach the final decision on which option is the \textit{dominant one}.
\todo{Contradicts what we do: ours is more like:}%
% In the second step of the decision-making process, the edited features are \emph{combined}, then \emph{contrasted} across the two options ($x_1, x_2$) and to reach the final decision on which option is the \textit{dominant one}.
    % 
    We will refer to this second step as the \textit{dominance testing rule}.
    Let $\outer: \X' \times \X' \rightarrow [0,1]$ denote the dominance test function, where $\X' = \X_1' \times \cdots \times \X_d'$.
    % 
    % With this setup in mind, the overall strategy we use to model a decision maker's response to any pairwise comparison is to first apply editing functions $\inner_{i, \omega}$ over feature $i$, for all $i$ and both options and then combine the edited information to reach the final decision using the dominance testing rule $\outer$.

    \begin{remark}[Examples of Dominance Testing rules]
        % Several human decision-making processes for pairwise comparisons studied in prior literature demonstrate the use of dominance testing strategies
        % , e.g., to maximize the expected utility of their decision 
        % \citep{montgomery1983decision}.
        % 
        Several dominance testing rules have been studied in prior literature, including the \textit{tallying up} heuristic mentioned earlier 
        % (count how many edited features favor each element in the pair and choose the one with a higher count) 
        \citep{czerlinski1999good,gigerenzer2011heuristic} and the \textit{prominent feature} heuristic (choose the element favored by the most prominent feature for which there is non-zero difference in edited feature value% across the options
        ) \citep{persson2022prominence,tversky1988contingent}.
        % 
        % Generally, the use of heuristic functions for this phase could also involve simple decision rules, like the ones mentioned above. 
        % 
        Alternately, one could create additive or probabilistic functions to capture various dominance testing rules observed in practice, e.g., Bradley-Terry aggregation (\citeyear{bradley1952rank}). % \citep{bradley1952rank}
        % (discussed % more
        % in %the next section
        % \cref{sec:axioms}).
        % 
        % We will demonstrate such generic dominance testing models in \cref{}.
    \end{remark}
With this setup, we model a decision maker's response to any pairwise comparison as first applying editing functions $\smash{\inner^{i, x^\omega_i}}$ over each feature $i$ for both options, given context $\omega$ and $\omega_i = \omega \setminus \{i\}$, and then combining the edited information to reach the final decision using the dominance testing rule $\outer$.
Hence, our proposed hypothesis class contains such two-stage models, namely %as formalized below.
% that follow the above described two-stage process.
% 
% That is, 
% 
% the two-stage mode we use to model the decision maker's response to any pairwise comparison is the following: $\smash{\hat{H}}(x_1, x_2) = \outer \left(\hat{x_1}, \hat{x_2} \right),$
% where $\hat{x_i} = \left(\inner_{1,\omega} (x_{i}^{(1)}), \inner_{1,\omega} (x_{i}^{(1)}, \ldots, \inner_{d,\omega}(p_{j,d}) \right)$.
% 
{
\[
\HC \doteq \left\{ x_1, x_2 \mapsto
  \outer\left (\forall i \in [d], x^{(i)}_{1} \mapsto \smash{\inner^{i, x_1^{\omega_i}}(x^{(i)}_{1})}, x^{(i)}_{2} \mapsto \smash{\inner^{i, x_2^{\omega_i}}(x^{(i)}_{2})} \right)
  %\outer\left( \sum_{i=1}^{d} h_{i}(x^{(1)}_{i}) - h_{i}(x^{(2)}_{i}) \right)
  \,\middle|\, \inner^{\cdot, \cdot} \in \HCinner, \outer \in \HCouter
  %, \outer \in \HC_{\text{outer}}
  \right\}.
% \enspace.
\]}%
The properties that characterize 
% the hypothesis classes for editing and dominance testing functions, 
classes $\HC_{\text{inner}}$ and $\HC_{\text{outer}}$ are discussed in the next section.

\subsection{Axiomatic Characterization of Decision Rules} \label{sec:models:axioms}\label{sec:axioms}
% \vspace{-0.07in}

The above-described
% above description of the 
two-stage process is motivated by extensive decision-making literature.
% on the heuristics and decision rules people use for pairwise comparison.
% 
% An 
% option and 
An additional compelling motivation for this hypothesis class comes from an axiomatic characterization of human decision processes for pairwise comparisons.
% 
% In this section, we show how natural axioms on the decision-maker's response function $H$ lead to the two-stage process described above.
% 
% 
The axioms below describe simple properties expected to be satisfied in an ideal comparative choice model (independent of its functional form).
% of a binary predictive model $H(\cdot, \cdot)$, i.e., 
% Importantly, these can be stated independently of the functional form of $H(\cdot, \cdot)$.
% 
% In general, we are also interested in properties of the hypothesis class $\HC$, but an axiom applies to $\HC$  if it applies to each $h \in \HC$.
% Each is described intuitively in plain English, then formally in mathematics.
% 
% 
% 
\begin{definition}[Axioms of Binary Choice% Models
]
\label{def:axioms}
%
%Consider
% Given a binary preference function $H(x_{1}, x_{2}): \X^{2} \to \Y$.\todo{x1, x2 notation!}\todo{Is this $h$ actually $H$? Or should we write $\Prob(x_{1} \geq x_{2})$?}
%We define the following axioms on $H(\cdot, \cdot)$.
The following axioms describe a binary preference function $H(x_{1}, x_{2}): \X^{2} \to \Y$.
Unless otherwise stated, each must hold \emph{for all} $x_{1}, x_{2}, x_{3} \in \X$.
\begin{enumerate}
\item\label{def:axioms:comp} \textbf{Complementarity}:
%Complementarity is essentially the principle that 
The order in which two options are presented should not impact the option that is eventually selected.
Formally, we require that %for all $x_{1}, x_{2} \in \X$,
%for all $h \in \HC$,
$H(x_{1}, x_{2}) = 1 - H(x_{2}, x_{1})$.

\item\label{def:axioms:wt} \textbf{Weak Transitivity} (WT): The comparison $H(x_{1}, x_{3})$
can be
%computed
expressed
as a
continuous
function of
the probabilities
$H(x_{1}, x_{2})$ and $H(x_{2}, x_{3})$, i.e., for some continuous
$f: %[0, 1] \times [0, 1]
 [0, 1]^{2} \to [0, 1]$, it holds that \scalebox{1}{$H(x_{1}, x_{3}) = f(H(x_{1}, x_{2}), H(x_{2}, x_{3}))$.\hspace{-1cm}\null}
%independent of the actual values of $x_{1}, x_{2}, x_{3}$ except through the predictions $H(x_{1}, x_{2})$, % and
%$H(x_{2}, x_{3})$.

%\item \textbf{Half-Indistinguishability}: If $H(x_{2}, x_{3}) = \frac{1}{2}$, then $H(x_{1}, x_{2}) = H(x_{1}, x_{3})$

%\item \textbf{Transitive Commutativity}: Suppose $x_{1}, x_{2}, x_{2}', x_{3}$, such that $H(x_{1}, x_{2}) = H(x_{2}, x_{2}')$ and $...$

\todo{Reversibility commented}
\iffalse
\item\label{def:axioms:pr} %\textbf{Path Reversal}:
\textbf{Reversibility}:
%For any $x_{1}, x_{2}, x_{3} \in \X$,
There exists some $x_{2}' \in \X$ such that $H(x_{1}, x_{2}') = H(x_{2}, x_{3})$ and \scalebox{1}{$H(x_{2}', x_{3}) = H(x_{1}, x_{2})$.\hspace{-2cm}\null}

\textcolor{brown!50!black}{%
Alternative axiom: \textbf{Commutative Transitivity}: As in WT, we may compute $H(x_{1}, x_{3}) = f( H(x_{1}, x_{2}), H(x_{2}, x_{3}) )$, and moreover, the function $f$ is order-invariant.
}
\fi

\item\label{def:axioms:cs} \textbf{Codomain Span}: \scalebox{0.98}[0.99]{For any $p \in (0, 1)$ and $x_{1} \in \X$, there exists $x_{2} \in \X$ such that $H(x_{1}, x_{2}) = p$.}

\todo{Is there exists some $x_{1} \in \X$ strong enough?}

\todo{
\textcolor{brown}{We used to have a weaker form:
\textbf{Codomain Span}: For any $p \in (0, 1)$, there exist $x_{1}, x_{2} \in \X$ such that $H(x_{1}, x_{2}) = p$.
}
}

\item\label{def:axioms:cic} \textbf{%Counterfactual Intervention 
Noninteractive Compositionality} (NC):
Suppose some $x_{1} \in \X$, and $x_{2}',x_{2}'',x_{2}''' \in \X$ such that $x_{2}'$ and $x_{2}''$ are obtained by changing different features of $x_{1}$, and $x_{2}'''$ is produced by applying both changes to $x_{1}$.
%Suppose
We then require that $H(x_{1}, x_{2}''')$ can be computed from $H(x_{1}, x_{2}')$ and $H(x_{1}, x_{2}'')$. %, i.e., independently of the actual values $x_{1}, x_{2}', x_{2}'', x_{2}'''$.

\item\label{def:axioms:ci} \textbf{Conditionally Interactive Compositionality} (CIC): %Interactivity}:
Given a set of %(possibly empty)
\emph{condition features} $\omega \subseteq [d]$, suppose some $x_{1} \in \X$, and $x_{2}',x_{2}'',x_{2}''' \in \X$ such that $x_{2}'$ and $x_{2}''$ are obtained by changing different features (%both not
neither in $\omega$) of $x_{1}$, and $x_{2}'''$ is produced by applying both changes to $x_{1}$.
%Suppose
We then require that $H(x_{1}, x_{2}''')$ can be computed from $H(x_{1}, x_{2}')$ and $H(x_{1}, x_{2}'')$. %, i.e., independently of the actual values $x_{1}, x_{2}', x_{2}'', x_{2}'''$.

\end{enumerate}
\end{definition}

% TODO cut 2nd half of PR.
% OR strengthen DS:
% For any $x_{1}$, for any $p \in (0, 1)$, there exists $x_{2}$ such that ...

% \paragraph{On disrcete and bounded domains}
%TODO

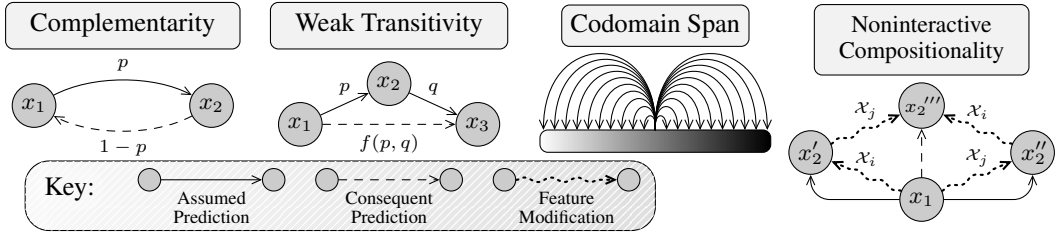
\begin{figure}
\begin{centering}
\begin{tikzpicture}[
    xscale=2.36,yscale=1.25,
    xi/.style={circle,draw=black,fill=gray!40!white,font=\small,inner sep=1pt,minimum size=1.65em},
    xiblank/.style={xi,minimum size=0.85em},
    xp/.style={circle,draw=black,fill=gray!50!green},
    gn/.style={rectangle,draw=black,fill=gray!50!orange},
    output/.style={rectangle,draw=black,fill=white},
    %myarrowhead/.style={thick,-{Latex[length=1.5mm,width=1.0mm]}},
    myarrowhead/.style={-{Straight Barb[length=1.2mm,width=1.1mm]}},
    arr/.style={draw=black,myarrowhead,tips=proper,line cap=round},
    arrmod/.style={arr,dotted,thick,decorate,decoration={name=snake,post length=1mm,amplitude=0.3mm}},
    propname/.style={rectangle,rounded corners=3pt,draw=black,fill=gray!10!white},
    keyitem/.style={font=\scriptsize},
]

\iffalse
\begin{scope}[shift={(0, 0)}]

\node[propname] (ref) at (0, 0) {\begin{tabular}{c} Reflexivity %\\
%$H(x_{1}, x_{1}) = \mathsmaller{\frac{1}{2}}$
%$H(x, x) = \mathsmaller{\frac{1}{2}}$
\end{tabular}};

\node[xi] (ref-x1) at (0, -0.8) {$x_{1}$};

\draw[arr,dashed] (ref-x1) edge[arr,
    %out=315, in=225,
    out=300, in=240,
    loop,every loop/.append style={myarrowhead}] node[below]{\scriptsize $\nicefrac{1}{2}$} (ref-x1);

%\draw[arr,dashed] (ref-x1.300) .. controls ++(-0.25,-0.5) and ++(0.25,-0.5) .. node[below]{\scriptsize $\nicefrac{1}{2}$} (ref-x1.240);

%\draw[arr,dashed] (ref-x1.300) edge[bend left=60] node[below]{\scriptsize $\nicefrac{1}{2}$} (ref-x1.240);

% \draw[arr,dashed,loop below] (ref-x1) edge[] node[below]{\scriptsize $\nicefrac{1}{2}$} (ref-x1);

\end{scope}
\fi

\begin{scope}[shift={(0, 0)}]

\node[propname] (com) at (0, 0) {\begin{tabular}{c} Complementarity %\\ \small $H(x_{1}, x_{2}) = 1 - H(x_{2}, x_{1})$ 
\end{tabular}};

\begin{scope}[shift={(0, 0.15)}]

\node[xi] (com-x1) at (-0.5, -1) {$x_{1}$};
\node[xi] (com-x2) at (0.5, -1) {$x_{2}$};

\draw[arr] (com-x1) edge[bend left=45] node[above]{\scriptsize $p$} (com-x2);
\draw[arr,dashed] (com-x2) edge[bend left=45] node[below]{\scriptsize $1-p$} (com-x1);

\end{scope}

\end{scope}

\begin{scope}[shift={((1.5,0))}]

\node[propname] (wt) at (0, 0) {\begin{tabular}{c} Weak Transitivity %\\ \small $H(x_{1}, x_{3}) = f\bigl( H(x_{1}, x_{2}), H(x_{2}, x_{3}) \bigr)$
\end{tabular}};

\begin{scope}[shift={((0,-1.05))}]

\node[xi] (wt-x1) at (-0.5, 0) {$x_{1}$};

\node[xi] (wt-x2) at (0, {0.5*sin(60)}) {$x_{2}$};

\node[xi] (wt-x3) at (0.5, 0) {$x_{3}$};

%\draw[arr] (wt-x1) edge (wt-x2) edge (wt-x3);
\draw[arr] (wt-x1) edge node[above] {\scriptsize $p$} (wt-x2);
\draw[arr] (wt-x2) edge node[above] {\scriptsize $q$} (wt-x3);

\draw[arr,dashed] (wt-x1) edge node[below] {\scriptsize $f(p, q)$} (wt-x3);

\end{scope}

\end{scope}

\iffalse
\begin{scope}[shift={((2,-2))}]

\node[propname] (wt) at (0, 0) {\begin{tabular}{c} Reversibility (TODO: Now unused) \\ %$H(x_{1}, x_{3}) = f\bigl( H(x_{1}, x_{2}), H(x_{2}, x_{3}) \bigr)$ 
\end{tabular}};

\begin{scope}[
    shift={((0,-1.0))}, %-1.2
    xscale=1.4,
    ]

%\newcommand{\revxtwooffset}{0.25}
\newcommand{\revxtwooffset}{0.15}

\node[xi] (rev-x1) at (-0.5, 0) {$x_{1}$};

\node[xi] (rev-x2) at (-\revxtwooffset, 0.433012702) {$x_{2}$};
\node[xi,dashed,postaction={pattern={north east lines},pattern color=white}] (rev-x2p) at (\revxtwooffset, -0.433012702) {$x_{2}'$};

\node[xi] (rev-x3) at (0.5, 0) {$x_{3}$};

%\draw[arr] (wt-x1) edge (wt-x2) edge (wt-x3);
\draw[arr,draw=red!67!black] (rev-x1) edge node[above] {\scriptsize $\!p$} (rev-x2);
\draw[arr,draw=blue!67!black] (rev-x2) edge node[above] {\scriptsize $q$} (rev-x3);
\draw[arr,dashed,draw=black!40!gray] (rev-x1) edge (rev-x3);

\draw[arr,dashed,draw=blue!67!black] (rev-x1) edge node[below] {\scriptsize $q$} (rev-x2p);
\draw[arr,dashed,draw=red!67!black] (rev-x2p) edge node[below] {\scriptsize $p\!$} (rev-x3);

\end{scope}
\end{scope}
\fi

\begin{scope}[shift={((3,0))}]

\node[propname] (ds) at (0, 0) {Codomain Span};

\begin{scope}[
    shift={{(0,-1.3)}},
    xscale=1,yscale=0.95,
    ]

\newcommand{\dsrad}{0.65}
\newcommand{\narrs}{10}

\draw[draw=black,rounded corners=3pt,left color=white,right color=black] (-\dsrad,-0.05) rectangle (\dsrad,0.2);

\foreach \i [evaluate=\i as \dsvo using 1.13*(\i/\narrs)^0.75]  in {1,2,...,\narrs} {

%\draw[arr] ({{-0.99*\i*\dsrad/\narrs}},0.2) edge[-] ++(0,0.15) edge[bend left=90] (0.99*\i*\dsrad/\narrs,0.2);

%\newcommand{\voffset}{0.11}

% \draw ({{-0.99*\i*\dsrad/\narrs}},0.2) edge[-] ++(0,0.05);
% \draw[arr] ({{-0.99*\i*\dsrad/\narrs}},{{0.2+0.05}}) edge[bend left=90] (0.99*\i*\dsrad/\narrs,0.2);

%(0.99*\i*\dsrad/\narrs,0.2+0.11) edge ++(0,-0.11);

%\newcommand{\dsvo}{1.0}

\newcommand{\dsxf}{0.97}

%Old ``rainbow:
%\draw[arr] ({{-\dsxf*\i*\dsrad/\narrs}},0.2) -- ({{-\dsxf*\i*\dsrad/\narrs}},0.25) .. controls ({{-\dsxf*\i*\dsrad/\narrs}},{{0.2+\dsvo}}) and ({{\dsxf*\i*\dsrad/\narrs}},{{0.2+\dsvo}}) .. ({{\dsxf*\i*\dsrad/\narrs}},0.2);

%New ``single origin"
\draw[arr] (0,0.2) -- (0,0.25) .. controls (0,{{0.2+\dsvo}}) and ({{\dsxf*\i*\dsrad/\narrs}},{{0.2+\dsvo}}) .. ({{\dsxf*\i*\dsrad/\narrs}},0.2);
\draw[arr] (0,0.2) -- (0,0.25) .. controls (0,{{0.2+\dsvo}}) and ({{-\dsxf*\i*\dsrad/\narrs}},{{0.2+\dsvo}}) .. ({{-\dsxf*\i*\dsrad/\narrs}},0.2);

}

\end{scope}
\iffalse
\begin{scope}[shift={((-0.15,0.2))},yscale=0.4,xscale=0.4]

\draw[draw=black,rounded corners=2pt,top color=black,bottom color=white] (-0.25, -4) rectangle (0.25, -1);

\foreach \i in {0,1,...,6} {

%\draw[arr] {{(0.25,-4)}} edge[bend right=45] {{(0.25,-2)}};
\draw[arr] (0.25,{{-2.5-\i*0.24}}) edge[bend right=60] (0.25,{{-2.5+\i*0.24}});

}

\end{scope}
\fi

\end{scope}

\begin{scope}[shift={((4.5,-0.15))}]

\node[propname] (ds) at (0, 0) { \small \begin{tabular}{c} 
 %Counterfactual Intervention \\ Compositionality
 Noninteractive \\ Compositionality
 \end{tabular}};

\begin{scope}[
    shift={((0,-1.2))},
    xscale=1.25,
    ]

\node[xi] (cic-x1) at (0, -0.5) {$x_{1}$};

\node[xi] (cic-x2p) at (-0.5, 0) {\small$x_{2}'$};
\node[xi] (cic-x2pp) at (0.5, 0) {\small$x_{2}''\!$};

\node[xi] (cic-x2ppp) at (0, 0.5) {\scriptsize$x_{2}\!'''\!$};

% \draw[arr] (cic-x1) edge[bend left=50] (cic-x2p);
% \draw[arr] (cic-x1) edge[bend right=50] (cic-x2pp);

\draw[arr,rounded corners=1.6ex] (cic-x1) -| (cic-x2p);
\draw[arr,rounded corners=1.6ex] (cic-x1) -| (cic-x2pp);

%\draw[arr,dashed] (cic-x2p) edge node[above] {\tiny $\X_{j}$} (cic-x2ppp);
%\draw[arr,dashed] (cic-x2pp) edge node[above] {\tiny $\X_{i}$} (cic-x2ppp);

\draw[arr,dashed] (cic-x1) edge (cic-x2ppp);

\draw[arrmod] (cic-x1) -- node[above] {\tiny $\X_{i}$} (cic-x2p);
\draw[arrmod] (cic-x1) -- node[above] {\tiny $\X_{j}$} (cic-x2pp);

\draw[arrmod] (cic-x2p) -- node[above] {\tiny $\X_{j}$} (cic-x2ppp);
\draw[arrmod] (cic-x2pp) -- node[above] {\tiny $\X_{i}$} (cic-x2ppp);

%\draw[arrmod] (cic-x1) -- (cic-x2ppp);

\end{scope}

\end{scope}

\begin{scope}[
    shift={{(-0.75, -1.64)}}
    ]

\draw[draw=black,rounded corners=2ex,postaction={pattern={north east lines},pattern color=white},left color=white!95!gray,right color=white!75!gray] 
    %(-0.5, -0.65) rectangle (4.0, 0.25);
    (0.2, -0.54) rectangle (3.75, 0.20);

% \node[propname,fill opacity=0.25,text opacity=1,fill=white,draw=none]
%     (ds) at (0, 0) {Key:};

\node[propname,fill opacity=0.25,text opacity=1,fill=white,draw=none,rounded corners=1.5ex]
    (ds) at (0.45, -0.1) {Key:};

\newcommand{\keyoffset}{0.1}

\node[xiblank] (k1-x1) at (1.0-\keyoffset, 0.0) {\ }; %{$x_{1}$};
\node[xiblank] (k1-x2) at (1.5+\keyoffset, 0.0) {\ }; %{$x_{2}$};

%\draw[arr] (1.0, 0) edge node[below] { \small \begin{tabular}{c} Assumed \\ Prediction \end{tabular} } (1.5, 0);
\draw[arr] (k1-x1) edge node[below,keyitem] { \raisebox{0.25ex}{\begin{tabular}{c} Assumed \\[-0.3ex] Prediction \end{tabular}} } (k1-x2);

\node[xiblank] (k2-x1) at (2.0-\keyoffset, 0.0) {\ }; %{$x_{1}$};
\node[xiblank] (k2-x2) at (2.5+\keyoffset, 0.0) {\ }; %{$x_{2}$};

\draw[arr,dashed] (k2-x1) edge node[below,keyitem] { \raisebox{0.25ex}{\begin{tabular}{c} Consequent \\[-0.3ex] Prediction \end{tabular}} } (k2-x2);

\node[xiblank] (k3-x1) at (3.0-\keyoffset, 0.0) {\ }; %{$x_{1}$};
\node[xiblank] (k3-x2) at (3.5+\keyoffset, 0.0) {\ }; %{$x_{2}$};

\draw[arrmod] (k3-x1) -- node[below,keyitem] { \raisebox{0.25ex}{\begin{tabular}{c} Feature \\[-0.3ex] Modification \end{tabular}} } (k3-x2);

%\usetikzlibrary{decorations.pathmorphing}

\end{scope}

\end{tikzpicture}
\\
\end{centering}
%
% \vspace{-0.25cm}
% \vspace{-0.4cm}
%
\caption{
% \small%
Visualization of %all axioms presented in 
the axioms of \cref{def:axioms}.
Solid arrows represent assumed choice probabilities, dashed arrows represent choice probabilities %that can be derived from assumed judgment probabilities
implied by an axiom, and in \emph{noninteractive compositionality}, dotted %snake
arrows represent %some fixed modification to the indicated feature
%modifications to some feature
feature modifications.
%Where applicable, arrows are
Arrows may be 
%Some arrows are
labeled with choice probabilities.
% , and their lengths roughly correlate with choice probabilities. 
\todo{Reflexivity, or the principle $H(x, x) = \mathsmaller{\frac{1}{2}}$, is omitted from \cref{def:axioms}, as it %is a trivial consequence of 
follows from complementarity.} %: If order does not matter, then $x$ must be chosen over itself 50\% of the time.
\iffalse
Complementarity,
%weak transitivity,
WT,
and
%noninteractive compositionality
NC
are depicted straightforwardly, % of their definitions
reversibility adds a consequent arrow from $x_{1}$ to $x_{3}$ (concluded via weak transitivity), and codomain span represents a continuous spectrum of $\X$ such that all possible probabilistic predictions arise.
\fi%
%, and noninteractive compositionality illustrates two independent feature modifications and the consequent prediction that arise from them.%
\todo{
These axioms are roughly arranged by complexity:\todo{ reflexivity is a single point condition,}
Complementarity describes the relationship between two points,
%weak transitivity
WT describes a triangular relationship between three points, %reversibility in some sense describes a parallelogram, where the diagonal is implied by WT, and
%noninteractive compositionality
and NC
is
%also
a four-point condition, where the two assumed preferences imply the remaining preferences.
}
\todo{Conditional comp?}
}
\label{fig:axioms}
% \vspace{-0.5cm}\null
% \vspace{-0.75cm}\null
\end{figure}

\todo{Graphic illustration}
\todo{Domain reduction.}
These axioms are visualized in \cref{fig:axioms}.
% Axioms 1--3 are prescriptive. 
%at times a bit technical in nature
%a bit technical.
%We hope that 
\emph{Complementarity} (\ref{def:axioms:comp}) %merely %is an apparently uncontroversial statement
ensures
that the order in which options are presented should not matter, and
%probabilistic predictions
predicted probabilities
sum to $1$.\todo{footnote{If order effects \emph{are} desirable, or one wishes not to normally assume them to not exist, this assumption is not actually limiting:  one can simply augment the input space to represent left/right context explicitly. OR augment $\cal X$ with its order, i.e.,treat the comparison problem as $(x_{L}, L)$ vs. $(x_{R}, R)$ instead of just $x_{L}$ vs $x_{R}$.} Parity problems? transitivity is about 3s, and our arguments require flopping things around? So this would require illegal comparisons where R / L context is flipped or duplicated.}
%No hidden variables is a very weak form of transitivity, requiring
\emph{Weak transitivity} (\ref{def:axioms:wt}) requires that comparisons between $(x_1, x_2)$ and $(x_2, x_3)$ carry all the information about the items $x_1, x_2, x_3$ to also compare $(x_1, x_3)$, i.e., the first two pairwise comparisons suffice to ``complete the triangle'' of all three pairwise comparisons.
% to make a third comparison.
% 
% Weak transitivity requires only that any two comparisons between items carry all information about the items to make a third comparison.
% , i.e., to ``complete the triangle'' of pairwise comparisons. %said items that is relevant to comparisons.
% 
%Path reversal
\todo{Reversibility is a bit subtle, but it essentially requires that a path going from $x_{1}$ to $x_{2}$ to $x_{3}$ has a corresponding reverse path from $x_{3}$ to $x'_{2}$ to $x_{1}$, with corresponding jumps in relative probabilities.}%
%The following theorem
\emph{Codomain span} (\ref{def:axioms:cs}) is a
%trivial
technical condition, essentially describing a type of continuity, as to specify the decision rule's transitivity law for \emph{all possible probabilities}, we must ensure that all possible probabilities \emph{actually arise} in comparisons over the domain $\X$.
% and transitivity relationships can be constructed as necessary. % in the domain.
\Cref{thm:factoring} item~\ref{thm:factoring:trans} shows that these properties work together to induce a strong form of transitivity in probabilistic predictions, and the discrete form given as item~\ref{thm:factoring:disc} relaxes the \emph{codomain span} (\ref{def:axioms:cs}) assumption.
Axioms %(5) and (6)
\ref{def:axioms:cic}~\&~\ref{def:axioms:ci} characterize the extent to which different features interact in the decision process.
%
%\emph{Counterfactual intervention compositionality}
\emph{Noninteractive compositionality} encodes a form of \textit{feature independence},
% is a bit more complex, but it essentially 
requiring that the
impact of changing two different features is compositional, %changing two features can be expressed additively inside the sigmoid as the sum of two individual changes,
which is less restrictive than assuming linearity, but restricts interactions between features in a similar manner.
Essentially, this axiom characterizes a \emph{generalized additive model} (GAM) \citep{hastie2017generalized}, as it dictates how information is synthesizes \emph{across features}.\todo{Citation!}
\emph{Conditional interactivity} 
then generalizes this idea, allowing for non-additive \emph{conditional feature interactions}.
We propose this axiom to account for a larger class of models that consider the decision context, where the impact of changing two features is additively compositional, \emph{except if} one of the features is a \emph{condition feature} (part of the \emph{context} $\omega$ described in \cref{sec:decision_rules}).

Axioms 1--3 are prescriptive, describing properties that may be considered desirable by many, whereas
%the latter two axioms
4--5
are not 
% prescriptive 
assumed to hold universally (potentially varying even within an individual), but rather they suffice characterize a class of simple decision-making rules when they do hold.
This stands in contrast to the strong assumptions made by classical axiomatic models of decision making. %In contrast, our axioms do not aim to characterize rational choice or optimal decision-making. Instead, they formalize minimal structural properties of comparative judgments --- namely, how local pairwise decisions relate to one another and how feature-level changes compose.
%Importantly, these axioms
Moreover, they are agnostic to the functional form of the decision rule and do not assume linear utility, stable preferences, or optimality.
%%

\iffalse %Redundant with intro
{\color{brown}
Our axiomatic approach differs
%fundamentally
substantially
from classical %axiomatizations
%characterizations
analyses
of preference. Traditional frameworks, 
%including
such as
expected utility theory and %its
probabilistic variants, characterize preferences by imposing strong global consistency conditions, such as complete transitivity and independence of irrelevant alternatives \citep{vonNeumann1944theory,luce1959individual}. These axioms yield strong conclusions, but they are known to be descriptively inadequate for human choice behavior, especially in settings involving bounded rationality, heuristics, and context effects \citep{tversky1969intransitivity,gigerenzer2000simple,erev2010choice}. 
% 
In contrast, our axioms do not aim to characterize rational choice or optimal decision-making. Instead, they formalize minimal structural properties of comparative judgments --- namely, how local pairwise decisions relate to one another and how feature-level changes compose.
%Importantly, these axioms
Moreover, they are agnostic to the functional form of the decision rule and do not assume linear utility, stable preferences, or optimality.
}
\fi
%
{%\color{brown}
%It is also important to
We emphasize that our axioms are strictly weaker than those commonly used in classical choice theory. For example, \emph{WT} does not assume that preferences are globally transitive or derivable from a single latent utility function; it merely requires that a comparison between two options can be constructed from intermediate comparisons. Similarly, \emph{NC} is weaker than additive utility or independence of irrelevant alternatives: it constrains how feature modifications combine without assuming linearity or ruling out contextual dependence. \emph{CIC} further relaxes this requirement by explicitly allowing structured feature interactions. 
As a result, our framework accommodates heuristic and context-sensitive decision processes documented in empirical studies, while still enabling formal analysis and learnability guarantees. 
We now show that the decision-making processes that satisfy these axioms follow the two-stage process outlined in \cref{sec:decision_rules}.
% 
% Hence, leading to computational models.
% 
}
% This positions our axioms as a principled alternative to classical rationality assumptions, rather than a reassertion of them.}

% % Axioms 1--3 are prescriptive  Moreover, 
% All of our axioms %, while remaining only
% are simple statements about concrete judgments over pairwise choices (rather than more abstract properties, like estimation
% %processes
% or a specific decision mechanism).
% % that are observed in either observed in people’s reasoning for their decisions.}
% Furthermore, they are quite weak: %the axioms are quite weak: %do not specify any particular relationship: %, but only that some relationship exists;
% WT assumes that transitive relationships \emph{exist}, %\emph{reversibility} assumes that intermediaries exist, 
% and NC and CIC assume that the impact of multiple independent feature modifications \emph{can be computed}, %but none specify exactly \emph{how} these relationships are to work.
% without specifying any \emph{particular} relationship.
% %
% %Note that these
% % These
% \ccolor{Note that we do not claim that our axioms characterize rationality, but
% % Rather, they are descriptive properties of human decision-making.
% % }
% % 
% % Hence, 
% %Rather, these axioms characterize
% rather
% an ideal binary choice process \emph{in the absence of cognitive biases}; limitations of this assumption are discussed in \cref{sec:limitations}.}
% % 
% %
% We next show that the decision-making processes that satisfy these axioms follow the two-stage process outlined in \cref{sec:decision_rules}.
% 

%TODO: Still, ..., infinite family.

%[NEW THEOREM]

\begin{restatable}[Axiomatic Factoring Characterization]{theorem}{thmfactoring}
\label{thm:factoring}
We now explore the consequences of
%the axioms of \cref{def:assumptions}
our axioms
on characterizing the response function $H(\cdot, \cdot)$ of a decision-maker.
% \quad 
\begin{enumerate}
\item\label{thm:factoring:disc} \textbf{Atomic Model}
Suppose %axioms 1--2
axiom~\ref{def:axioms:comp}, and
%furthermore assume
also
that
\todo{We used to assume this: the transitivity law $f(\cdot, \cdot)$ is symmetric,
i.e., $H(x_{1}, x_{3}) = f( H(x_{1}, x_{2}), H(x_{2}, x_{3}) ) = f( H(x_{2}, x_{3}), H(x_{1}, x_{2}) )$ for all $x_{1}, x_{2}, x_{3} \in \X$.}%
$\X$ is a countable domain.
%there exist some transitivity rule $H(x_{1}, x_{3}) = f( H(x_{1}, x_{2}), H(x_{2}, x_{3}) )$ consistent with all $x_{1}, x_{2}, x_{3} \in \X$ such that $f$ is a
%continuous and
%symmetric function, i.e., $f(p, q) = f(q, p)$.
Then there exists an \emph{atomic rule} $\inner: \X \to \R$ and $\outer: \R \to [0, 1]$ s.t. $H(\cdot, \cdot)$ can be factored as\todo{Isn't $\inner$ a bivariate function?}
\[
H(x_{1}, x_{2}) = \outer \left( \inner(x_{1}) - \inner(x_{2}) \right) \enspace.
\]

\item\label{thm:factoring:trans} \textbf{$\sigma(\cdot)$-Transitivity}
Suppose\todo{ a decision rule $H(\cdot, \cdot): \X \times \X \to [0, 1]$ that obeys} axioms~\ref{def:axioms:comp}--\ref{def:axioms:cs} %, and moreover,
({Complementarity}, {WT}, and {Codomain Span}) hold.
%hold for some \emph{continuous} transitivity law $f(\cdot, \cdot)$. % (1--4) \cref{def:axioms}~item~\ref{def:axioms:ci}
%and moreover that in the NHV characterization, $H(x_{1}, x_{3}) = f( H(x_{1}, x_{2}), H(x_{2}, x_{3}) )$, $f$ is a \emph{continuous function}.
Then there exists a symmetric continuous random variable with full support and CDF $\sigma(\cdot)\iflongversion: \R \to (0, 1)\fi$ s.t.
%sigmoid function $\sigma(\cdot): \R \to (0, 1)$ %and model function $h: \X \to \R$ %, i.e., (...)
%such that
%\vspace{-0.2cm}
\[
%f(p, q) = \sigma( \sigma^{-1}(p) + \sigma^{-1}(q) ) 
H(x_{1}, x_{3}) = \sigma \left( \sigma^{-1}( H(x_{1}, x_{2} ) ) + \sigma^{-1}( H(x_{2}, x_{3} ) ) \right) \enspace. \todo{TODO axiom 1?}%\enspace,
\]
%and moreover
%and $\sigma(0) = \frac{1}{2}$\todo{ (equivalently, $\sigma^{-1}(\frac{1}{2}) = 0$)}, $\lim_{u \to \infty} \sigma(u) = 1$, $\lim_{u \to -\infty} \sigma(u) = 0$, and $\sigma(-u) = 1 - \sigma(u)$.
Moreover, there exists some \emph{atomic rule} $\inner: \X \to \R$ such that $H(\cdot, \cdot)$ can be factored as\todo{Isn't $\inner$ a bivariate function?}
\[
H(x_{1}, x_{2}) = \sigma \left( \inner(x_{1}) - \inner(x_{2}) \right) \enspace.
\]
\todo{
Moreover, there exists some ``neutral baseline'' $x_{0} \in \X$, such that for all $x \in \X$, it holds
%\[
$
H(x, x_{0})
=
\sigma \circ \inner(x)
$,
%\enspace,
%\]
i.e., the operator $(\outer \circ \inner)(\cdot)$ ``compares'' a given $x$ to the neutral baseline.}%
\todo{Future draft: path reversibility is sufficient but not necessary for symmetry?}
\todo{Add uniqueness characterization:
Moreover, this representation is unique up to
linear (but not affine) transformation,
%affine transformation,
i.e., the set of all such $\sigma$ that satisfy the above is
\[
\Sigma = \{ \sigma'(u) = \sigma( \alpha u %+ \beta
    ) \, | \, \alpha \in \R %, \beta \in \R
    \}
\enspace.
\]
%Furthermore, 
Hence for some choice of $\alpha$ %, $\beta$
above, there exists a sigmoid curve $\sigma(u)$ such that %$\sigma(0) = \frac{1}{2}$ (equivalently, $\sigma^{-1}(\frac{1}{2}) = 0$, 
$\frac{\partial}{\partial u} \sigma(u) = 1$ at $u = 0$, which we term the \emph{canonical form}\footnote{This multiplicative scale is often referred to as the \emph{inverse temperature}, by analogy with statistical physics models.
Some authors canonicalize to $\frac{\partial}{\partial u} \sigma(u) = \frac{1}{4}$ at $u = 0$, as per the logistic function, or for the Gaussian case $\Phi$, we have $\frac{\partial}{\partial u} \sigma(u) = \frac{1}{\sqrt{2\pi}}$.}.
Additionally, item~1 holds (non-uniquely) with $\outer(\cdot) = \sigma(\cdot)$.
}

\todo{
[why it's called transitivity: these imply that $H(x_{1}, x_{3}) \geq \max( H(x_{1}, x_{2}), H(x_{2}, x_{3}) )$ if $H(x_{1}, x_{2}) \geq \frac{1}{2}$ and $H(x_{1}, x_{3}) \geq \frac{1}{2}$.
Also $H(x_{1}, x_{3}) \geq \frac{1}{2}$ iff $\frac{1}{2}( H(x_{1}, x_{2}) + H(x_{2}, x_{3}) ) \geq \frac{1}{2}$, i.e., the ``stronger'' predicition wins.
If the two predictions are on opposite sides of $\frac{1}{2}$, then $\min(H(x_{1}, x_{2}), H(x_{2}, x_{3})) \leq H(x_{1}, x_{3}) \leq \max(H(x_{1}, x_{2}), H(x_{2}, x_{3}))$
]
}

\iffalse

\item\label{thm:factoring:atom} \textbf{Atomic Model}:
\scalebox{0.98}[0.99]{Suppose as in item~\ref{thm:factoring:trans}.
Then there exists some \emph{model function} $h: \X \to \R$ such that}
\[
H(x_{1}, x_{2}) = \sigma\left( h(x_{1}) - h(x_{2}) \right) \enspace.
\]
Moreover, there exists some ``neutral baseline'' $x_{0} \in \X$, such that for all $x \in \X$, it holds
\[
\sigma \circ h(x) = H(x, x_{0}) \enspace,
\]
i.e., the operator $(\sigma \circ h)(\cdot)$ ``compares'' any given $x$ to the neutral baseline.

\textcolor{brown}{Superseded by (1)?}

\fi

\item\label{thm:factoring:uncond} \textbf{Unconditional Factor Model}:
Suppose
% as in item~\ref{thm:factoring:trans}, 
%that
% $H(x_{1}, x_{2}) = \outer \left( \inner(x_{1}) - \inner(x_{2}) \right)$ %, i.e., 
% (as concluded by item \ref{thm:factoring:disc}~or~\ref{thm:factoring:trans}),
% and also %\cref{def:axioms}~item~\ref{def:axioms:cic}.
% axiom~\ref{def:axioms:cic} (NC).
axioms~\ref{def:axioms:comp}--\ref{def:axioms:cs} %\ref{def:axioms:cic} (i.e., add NC).
and \ref{def:axioms:cic} (NC).
Then there exist \emph{inner functions} $\inner^{i}: \X_{i} \to \R$ for $i \in [d]$ %$ 1, \dots, k$ 
such that the atomic rule $\inner(\cdot)$ and decision rule
% $H(x_1, x_2)$
$H(\cdot, \cdot)$
%can be factored as
factor as
% \vspace{-0.2cm}
\[
\inner(x) = \smash{\sum_{i=1}^{d}} \inner^{i}(x^{(i)})
\enspace, \quad
H(x_1, x_2) = \outer\left( \vphantom{\sum}\smash{\sum_{i=1}^{d}} \inner^{i}(x^{(i)}_{1}) - \inner^{i}(x^{(i)}_{2}) \right)
\enspace.
\todo{TODO: sigma houter?}
\]
% \[
% H(x^{(1)}, x^{(2)}) = \sigma\left( \sum_{i=1}^{d} \inner^{i}(x^{(1)}_{i}) - \inner^{i}(x^{(2)}_{i}) \right)
% \enspace.
% \]
\iffalse
Equivalently, the atomic model of item~\ref{thm:factoring:atom} can be expressed as
\[
h(x) = \left( \sum_{i=1}^{d} \inner^{i}(x^{(i)}) \right)
\enspace.
\]
\fi
% TODO: INDEXING

\item\label{thm:factoring:cond} \textbf{Conditional Factor Model}:
Suppose
%as in item~\ref{thm:factoring:trans},
% and also 
% that $H(x_{1}, x_{2}) = \outer \left( \inner(x_{1}) - \inner(x_{2}) \right)$ (as concluded by item \ref{thm:factoring:disc}~or~\ref{thm:factoring:trans}),
axioms~\ref{def:axioms:comp}--\ref{def:axioms:cs}
and
%also %\cref{def:axioms}~item~\ref{def:axioms:ci}.
%axiom~
\ref{def:axioms:ci} (CIC).
Then there exist \emph{conditional inner functions} $\smash{\inner^{i,\cdot}:} \X_{i} \to \R$ for $i \in [d]$ %$ 1, \dots, k$ 
and \emph{condition features} (context) $\omega \subseteq [d]$, $\omega_{i} = \omega \setminus \{i\}$,
% $C \in \X_{C}$, 
%such that
s.t. %the atomic rule $\inner(\cdot)$ and decision rule $H(x_1, x_2)$ can be factored as
% \vspace{-0.15cm}
\[
H(x_1, x_2) = \outer\left( \vphantom{\sum}\smash{\sum_{i=1}^{d}} \inner^{i,x_1^{\omega_i}}(x^{(i)}_{1}) - \inner^{i,x_2^{\omega_i}}(x^{(i)}_{2}) \right)
\enspace. \todo{TODO: conditioning notation}
\]
%where $\omega_i = \omega \setminus \{i\}$, for any $i$.
% \[
% H(x^{(1)}, x^{(2)}) = \sigma\left( \sum_{i=1}^{d} \inner^{i,x^{(1)}_{C}}(x^{(1)}_{i}) - \inner^{i,x^{(2)}_{C}}(x^{(2)}_{i}) \right)
% \enspace.
% \]aao
%
\todo{Is there a version that specifies the conditioning?}
\end{enumerate}

\end{restatable}
% 
% TODO: conjecture, $\outer$ should be a sigmoid, just not uniquely, in item 1?
% 
% 
%NB: commentary, complementarity axioms
Item~\ref{thm:factoring:disc} shows that 
%axioms 1--2 are
axiom~1
%is sufficient
suffices
to reduce binary choices to a function of a difference of atomic predictions for continuous transitivity laws $f(\cdot, \cdot)$, i.e., to $H(x_{1},x_{2}) = \outer(\inner(x_{1}) - \inner(x_{2}))$,\linebreak[3] thus recovering the two-stage modeling process described in \cref{sec:decision_rules}.
However, there is little structure beyond this: Such a factoring \emph{exists}, but it is by no means \emph{unique}, and it may well be \emph{intransitive}.
Item~\ref{thm:factoring:trans}
%Axioms 3
then imposes weak transitivity through axiom~\ref{def:axioms:wt}, additional \emph{continuity structure} through axiom~\ref{def:axioms:cs}, and paired with continuity of the transitivity law $f(\cdot, \cdot)$, we may conclude that $\outer(\cdot)$ is a sigmoid function (CDF)\todo{\footnote{Moreover, there exist domains that satisfy axioms 1--2 that do not admit a sigmoidal $\outer(\cdot)$, i.e., that are $\sigma$-intransitive.}}. % (fixing $\HC_{\text{outer}}$).
Axioms 4--5 are then needed to control feature interactions, and dictate a two-stage model structure wherein each feature is processed individually or conditionally.

\todo{Longer version:
Naturally, one wonders as to the importance of the stronger axioms, and whether the conclusion of item~\ref{thm:factoring:trans} is really so much stronger than that of item~\ref{thm:factoring:disc}. 
It may be surprising that imposing some kind of continuity law, i.e., codomain span, over the domain $\X$ suffices to have such an impact, but
%we show that signal transitivity does not follow from axioms 1 and 2 to see this,
the following example brings its importance into contrast.
Consider a trivial three element domain that obeys an antimonotonic rocks-paper-scissors relationship, defined by $H(x_{1}, x_{2}) = H(x_{2}, x_{3}) = H(x_{3}, x_{1}) = p$ for any $p \neq \frac{1}{2}$. From here complementarity and weak transitivity require that $f(p, p) = p$, $f(1-p,1-p) = 1-p$, $f(p, 1-p) = \frac{1}{2}$, which completely specifies the decision rule while respecting these axioms.
%Here…, none of which violates any of the assumptions of item we don't see that 
Here we see that weak transitivity does not imply transitivity of weak preferability, i.e., the relation $H(x_{1}, x_{2}) \geq \frac{1}{2} \wedge H(x_{2}, x_{3}) \geq \frac{1}{2} \implies H(x_{3}, x_{1}) \geq \frac{1}{2}$, however it is easy to see that this property \emph{does follow} from $\sigma$-transitivity. %Moreover,
Consequently, this intransitive preference ordering can not be expressed as a $\sigma$-transitive relationship. % due to its anti-monotonicity.
}

Intuitively, with $\sigma$-transitivity, we require that transitive probabilities can be computed via addition within the sigmoid.
Therefore, transitivity along chains of probabilities above $\smash{\frac{1}{2}}$ grow ever closer to $1$ as $\sigma(\cdot)$ is applied to a growing sum.
\iflongversion
TODO
Thus, if both predictions of probability at least $\smash{\frac{1}{2}}$, a prediction with probability $p \geq \smash{\frac{1}{2}}$
produces a positive value of $\sigma^{-1}(p)$; the certainty of transitive predictions grows when both values have the same sign. TODO english!
\fi
\iflongversion%
Several valid $\sigma$ sigmoid functions commonly employed in ML are contrasted in \cref{fig:sigmoid}.\todo{Discuss?}
\fi%
In the parlance of generalized linear models, $\sigma^{-1}$ plays the role of a link function.
Bradley-Terry models, such as logistic regressors, employ the logistic sigmoid\todo{discuss log odds}, but other sigmoids, such as the Gaussian CDF (with probit link function)
%are also popular,
also see use,
e.g., in GLMs \citep{mccullagh1989generalized}.\todo{Discuss computational learnability, log-concavity, and cross entropy loss?}
%Furthermore, another way to think of these axioms is simply that, if we want some type of transitivity, we must define precisely how predictions involving $x_{1}, x_{2}, x_{3} \in \X$ are related mathematically, and even if we are not ready to pick a specific sigmoid curve $\sigma$, perhaps we can accept that such a some such curve would indeed define any given precise concept of transitivity.

\todo{Discuss sigmoids more?}

We now describe \emph{strong assumptions} 
% which are not necessarily well motivated from a cognitive reasoning perspective, but
that may only be appropriate in specific real-world domains, 
% or under certain assumptions, and 
but nevertheless produce specific properties for the hypothesis class of editing function $\HC_{\text{inner}}$. 
Unlike our axioms of \cref{def:axioms}, these assumptions are highly situational, and they describe behavior that may be considered reasonable in some circumstances, but their absence should rarely be viewed as surprising or undesirable in and of itself.

\begin{definition}[Assumptions on Binary Preference Models]
%Suppose a function
%$h(x_{1}, x_{2}): \X \to \Y$.
%$\inner: \X_{i} \to \R$. 
\todo{x1, x2 notation!}
\label{def:assumptions}

We state assumptions on $H$
% on binary preference models 
that are suitable for discrete or continuous domains.
These discrete properties must apply to all $x_{1}, x_{2}
%, x'_{1}, x'_{2}
\in \X$. 
\begin{enumerate}[resume]
\item\label{def:assumptions:d-lin} \textbf{$\sigma$-Linearity}: %$h(x) = wx$ for some $x \in \R$.
%, it holds that
There exists some $\wv \in \R^{d}$ such that\todo{forall is here}
\todo{Old form: $\sigma^{-1} \circ H(x_{1}, x_{2}) - \sigma^{-1} \circ H(x_{1}', x_{2}') = \wv \cdot \bigl( (x_{1} - x_{1}') - (x_{2} - x_{2}') \bigr)$.}%
$\sigma^{-1} \circ H(x_{1}, x_{2}) = \wv \cdot (x_{1} - x_{2} )$.

\todo{Simplified:
\[
\sigma^{-1} \circ H(x_{1}, x_{2}) = \wv \cdot (x_{1} - x_{2} )
\]
or even
\[
\sigma^{-1} \circ H(x_{1}, x_{2}) - \sigma^{-1} \circ H(x_{1}, x_{2}') = \wv \cdot (x_{2}' - x_{2})
\]
}
\item\label{def:assumptions:d-mono} \textbf{Monotonicity}:
\todo{Old: If $x_{1} \preceq x_{1}'$ and $x_{2} \succeq x_{2}'$, then $H(x_{1}, x_{2}) \leq H(x'_{1}, x'_{2})$.}%
If $x_{1} \preceq x_{2}$, then $
H(x_{1}, x_{2}) \leq \frac{1}{2}$.

\todo{Simplified:
If $x_{1} \preceq x_{2}$, then
\[
H(x_{1}, x_{2}) \leq \frac{1}{2}
\enspace.
\]
or
if $x_{2} \preceq x'_{2}$, then
\[
H(x_{1}, x_{2}) \leq H(x_{1}, x'_{2})
\enspace.
\]
or
\[
\sigma^{-1} \circ H(x_{1}, x_{2}) \leq 0
\]
or even
\[
\sigma^{-1} \circ H(x_{1}, x_{2}) - \sigma^{-1} \circ H(x_{1}, x_{2}') = \wv \cdot (x_{2}' - x_{2})
\]
}
\end{enumerate}
\end{definition}

  \vspace{-0.5ex}
  We now combine these stronger assumptions with the results of the axiomatic analysis of \cref{thm:factoring}.
Recall that %we obtain
our basic axioms imply the factoring
%\smash{$
\[
\HC = \left\{ x_{1}, x_{2} \mapsto 
%\right.$\linebreak[2]$\left.
  %\outer
  \sigma\!\left(  \inner(x_{1}) - \inner(x_{2}) \right)
  %\outer\left( \sum_{i=1}^{d} h_{i}(x^{(1)}_{i}) - h_{i}(x^{(2)}_{i}) \right)
  \,\middle|\, \inner \in \HCinner \right\}
  \enspace.
 \]
 % The following theor
%$.}
  % \vspace{-0.5ex}

\begin{restatable}[Axiomatic Models]{theorem}{thmmodels}
\label{thm:models}
%
%Suppose as in \cref{thm:factoring}~item~\ref{thm:factoring:uncond}.
Suppose axioms \ref{def:axioms:comp}--\ref{def:axioms:cic}.
\iflongversion
Then the class of feasible decision rules can be factored as
\vphantom{-0.333cm}
{\small%
% \[
% \outer(a, b) = \sigma \smash{\left( \sum_{i=1}^{d} a_{i} - b_{i} \right)}
% \enspace,
% \quad
% \HC = \left\{ x^{(1)}, x^{(2)} \mapsto
%   \outer\left ( i \mapsto \inner^{i}(x^{(i)}_{1}), i \mapsto \inner^{i}(x^{(i)}_{2}) \right)
%   %\outer\left( \sum_{i=1}^{d} h_{i}(x^{(1)}_{i}) - h_{i}(x^{(2)}_{i}) \right)
%   \,\middle|\, \inner^{i} \in \HCinner^{(i)} \right\}
% \enspace.
% \]
\[
\HC = \left\{ x_{1}, x_{2} \mapsto
  %\outer
  \sigma\left ( \vphantom{\sum}\smash{\sum_{i=1}^{d}} \inner^{i}(x^{(i)}_{1}) - \inner^{i}(x^{(i)}_{2}) \right)
  %\outer\left( \sum_{i=1}^{d} h_{i}(x^{(1)}_{i}) - h_{i}(x^{(2)}_{i}) \right)
  \,\middle|\, \inner^{i} \in \HCinner^{(i)} \right\}
\enspace.
\]
}
%\vphantom{-0.25cm}
%
%Suppose $\HC$ can be factored as $f $
\todo{Suppose also that [canonicalization].
Then
%$f(a, b) = \sigma(\sum_{i=1}^{d} a_{i} - b_{i})$.
this choice of $\outer$ is required.
}%
\iffalse
For a given aggregator function $f(\cdot, \cdot)$,
consider the class of all possible factored models, i.e., take each
\[
\HCinner^{(i)}: \X_{i} \to \R
\enspace.
\]
\fi

The following additional assumptions further restrict each $\HCinner^{(i)}$.

TODO edit me
\else
%Then
\iffalse
$\HC = \left\{ x_{1}, x_{2} \mapsto
  %\outer
  \sigma\left ( \vphantom{\sum}\smash{\sum_{i=1}^{d}} \inner^{i}(x^{(i)}_{1}) - \inner^{i}(x^{(i)}_{2}) \right)
  %\outer\left( \sum_{i=1}^{d} h_{i}(x^{(1)}_{i}) - h_{i}(x^{(2)}_{i}) \right)
  \,\middle|\, \inner^{i} \in \HCinner^{(i)} \right\}$.
\fi
The following %additional assumptions
% then %further
restrict $\HCinner$ or each $\HCinner^{(i)}$.
\fi
\begin{enumerate}
\item\label{thm:models:lin} \scalebox{0.96}[0.99]{Suppose \emph{$\sigma$-linearity}\iflongversion\ (\cref{def:assumptions}~items~\ref{def:assumptions:d-lin}~or\ref{def:assumptions:c-lin})\fi.
Then $\HCinner^{(i)} = \{ x \mapsto w x | w \in \R \}$, %and moreover
thus $\HC = \{ x_{1}, x_{2} \mapsto \sigma(\bm{w} \cdot (x_{1} - x_{2}) ) \,|\, \bm{w} \in \R^{d} \}
$.\!}

\item\label{thm:models:umono} Suppose \emph{monotonicity}\iflongversion\ (\cref{def:assumptions}~item~\ref{def:assumptions:d-mono})\fi.
Then $\HCinner^{(i)} = \{ h: \R \to \R \,|\, x \leq y \implies h(x) \leq h(y) \}$.

% \item \scalebox{0.98}[0.99]{Suppose \emph{continuous monotonicity}\iflongversion\ (\cref{def:assumptions}~item~\ref{def:assumptions:c-mono})\fi.
% Then $\HCinner^{(i)} = \{ h: \R \to \R \,|\, x \leq y \implies h(x) \leq h(y), h \text{ is continuous} \}$.}

\item\label{thm:models:mmono} \textbf{Multivariate Monotonic Models}: If we assume monotonicity but relax %the $\sigma$-noninteractivity axiom entirely,
noninteractive compositionality, then
$\HC = \{ x_{1}, x_{2} \mapsto \sigma( h(x_{1}) - h(x_{2}) ) \,|\, h: \X \to \R \text{ s.t. } \vec{x} \preceq \vec{y} \implies h(\vec{x}) \leq h(\vec{y}) \}$.

\todo{continuous matters here.}
\todo{relax transitivity, can we get a symmetric monotonic model? Does there exist some sigma such that it's the above (without assuming it)?}

\item\label{thm:models:condition-tree}
\textbf{Conditional GAM Tree}
If we relax %\emph{noninteractive compositionality} to \emph{conditionally interactive compositionality},
NC to CIC, then $\inner(\cdot)$ may be represented as a hybrid model of a \emph{decision tree} / GAM model, starting with a tree over $\X_{\omega}$, where each leaf contains a \emph{GAM} over $\X_{\setminus \omega}$.\todo{notation TODO $x^{(i)}$? Discuss below?} Formally, we have
%\[
\smash{$
\HCinner = \left\{ \vphantom{\sum}\smash{\sum_{i=1}^{d-\abs{\omega}}} \bigl(T_{i}(x^{\omega}) \bigr)\bigl((x^{\setminus \omega})^{(i)}\bigr) \, \middle| \, T \in \R^{\abs{\omega}} \to (\R \to \R)^{d-\abs{\omega}} \right\}
%\inner(x) = \sum_{i=1}^{d} \bigl(T_{i}(x_{\omega}) \bigr)(x_{\setminus \omega})
.$}
%\enspace.
%\]
%
%In other words,

\end{enumerate}

\iffalse
\begin{enumerate}
\item Suppose \emph{linearity} and \emph{noninteractivity}.
%then $\HCinner^{(i)} = \{ x \mapsto w \cdot x | w \in \}$.
% Then
% \[
% \HC = \{ x \mapsto w \cdot x | w \in \R^{d} \}
% \]
% up to invertible transformation.
%Moreover, additionally supposing strong Bradley-Terry yields the symmetric logistic regression class.
Suppose also that $f(a, b)$ is factor additive, monotonically increasing in $a$, and monotonically decreasing in $b$, then 
\[
\HC = \{ x^{(1)}, x^{(2)} \mapsto \logistic(\bm{w} \cdot (x^{(1)} - x^{(2)}) ) | \bm{w} \in \R^{d} \}
\enspace,
\]
up to invertible monotonic transformation.
In other words, the class is equivalent to symmetric logistic regression.\todo{Other GLM? Is Probit regression equivalent?}
\todo{Can we fix the sigmoid?}

\todo{Without noninteractivity: ?}

\todo{NB: Without fixing $f$: using $f(a, b) = \sin^{2}(\sum_{i=1}^{d} a_{i} - b_{i})$ could yield any function.}
\item\label{thm:models:mmono} Suppose \emph{monotonicity} and \emph{noninteractivity}.
Then
\[
\HC = ...
\enspace,
\]
up to invertible monotonic transformation.
Moreover, additionally supposing strong Bradley-Terry yields the symmetric logistic univariate monotonic regression class.

\item \textbf{Conditional Categorical Models}: [I think Bayes network theorem covers this]

\item \textbf{Kernel Models}: Something about RKHS? Assume axioms on kernel space $\Phi(\X)$, linearity replace with $K(\cdot, \cdot)$.
\end{enumerate}
\fi
%
\end{restatable}

\todo{[reorder item 4.]

[move text out of 2.6. xxx show how yyy restrict zzz]

[hybrid tree-gan model, where [tree] then [gam]]

[is gam clear? Suppose [axiom names]]
}

\todo{Aside on $\sigma = \sin^{2}$.}

From this perspective, we derive a few valuable insights:
\emph{Logistic regression} with nonnegative weights is abstractly characterized by Bradley-Terry aggregation ($\sigma(u) = \logistic(u)$) and linearity (and implicitly implies noninteractivity and monotonicity).
If linearity is relaxed to monotonicity, we obtain the class of univariate monotonic models, and subsequently when noninteractivity is relaxed, we obtain univariate monotonic models with conditional interactions\todo{Or multivariate}.
Similarly, \emph{probit regression} is characterized by taking $\sigma(u) = \Phi(u)$, i.e., the Gaussian CDF function, and linearity.
\iflongversion
TODO
NB the complementarity axiom constrains $\sigma$ to require $\sigma(u) = 1 - \sigma(-u)$, and as a result, popular asymmetric GLM, such as the complementary log-log, i.e., $\sigma(u) = 1 - \exp(-\exp(u))$, are inadmissible.
\fi

\iflongversion
TODO: don't say we use Luce!

Plan:

[probabilistic model first

Uniqueness of DFD: up to sign and translation: Yes among symmetric components, yes for Gaussian (uniquely among infinitely divisible), but no in general.

TODO: can there be an asymmetric decomp but no symmetric decomp?

For instance, the Laplace distribution may be represented either
as the di§erence of two symmetric Bessel distributions or as the di§erence of two
exponential distributions 

]

[then Luce?]
\fi
\paragraph{Extending Pairwise Comparisons to $n$-Way Choice}
We focus on binary preference elicitation because it is a well-studied task in the decision-making literature. 
% It is not immediately obvious how probabilistic binary preferences, i.e., $\Prob(x_{1} > x_{2})$ yield probabilistic rankings, i.e., preference distributions over a set of $n$ items.
% However, 
To extend our binary preference framework to learn probabilistic rankings over $n$ items, we use the \emph{choice axiom} of \citet{luce1959individual}:
%, who show that a simple axiomatic assumption suffices: 
\textit{Introducing additional items should not change the probability ratio of choosing $x_{1}$ to choosing $x_{2}$,}
\iffalse
i.e., $\smash{\frac{\Prob(A \text{ best})}{\Prob(B \text{ best})} = \frac{\Prob(A > B)}{\Prob(B > A)}}$ for all items $A, B \in \X$.
%or equivalently,
\fi%
% or equivalently, we could assume the \emph{conditional probability} condition
%of
% $\Prob(x_{1} \text{ best} | x_{1} \text{ best} \text{ or } x_{2} \text{ best}) = \Prob(x_{1} > x_{2})$. %of choosing a, given that a or b is chosen, must match the binary preference probability of choosing a over b, I eat x. 
Without %transitivity
WT,
there are $\binom{n}{2}$ degrees of freedom (DoF) to complementary binary preferences, and a probabilistic ranking that accords with these binary preferences in general does not exist (e.g., preferences over rocks-paper-scissors cannot accord with %binary preferences
the choice axiom).
However, %it is straightforward to see that
with %the $\sigma$-transitivity axiom,
%weak transitivity axiom,
WT, there are only $n-1$ DoF, as a %directed
tree of predictions spanning %all
$n$ items can be completed via transitivity\todo{ and complementarity}. \todo{e.g., the chain $\Prob(x_{1} > x_{2}), \Prob(x_{2} > x_{3}), \dots, \Prob(x_{n-1} > x_{n})$, or the star $\Prob(x_{1} > x_{2}), \Prob(x_{1} > x_{3}), \dots, \Prob(x_{1} > x_{n})$, suffice to compute the remaining $\binom{n}{2}-(n-1)$ prediction probabilities.}

\iflongversion
Thus, while at first glance, $n$-way comparisons would appear to require quadratic time (comparing all pairs), they are actually computable in linear time (assessing each item individually, then making $n-1$ comparisons).
We conjecture that while humans can easily handle binary comparisons (constant complexity), they become less accurate as the computation becomes more complex with additional items, and cognitive biases arise as a result.
However, this approach still motivates $\sigma$-transitivity, as the complexity of choosing among $n$ items is still greatly reduced over the general case, making it a feasible candidate for effective and efficient human decision making. %, and we find that for $n \geq 3$, the logistic sigmoid is the only option compatible with proportionality.
Moreover, while computing the preference distribution requires \emph{linear memory}, selecting the \emph{most likely} option in the preference distribution requires only \emph{constant memory} (by keeping track of both the current probability in the chain, and the highest denormalized probability thus observed).
\todo{Does this imply computational stuff, like for convex optimization?}
\fi

\begin{restatable}[Proportionality of %Higher-Order %Rankings
Choice]{theorem}{thmprop}
\label{thm:prop}
Suppose a probabilistic decision rule over $n \geq 3$ items %such that
that obeys 
the \emph{choice axiom}.
\iflongversion
TODO
%\emph{proportionality}, i.e., the selected item $X$ is a random variable over $x_{1}, \dots, x_{n}$, and for some complementary binary decision rule $H(\cdot, \cdot)$, it obeys
\[
\forall i, j \in 1, \dots, n: \quad
\frac{\Prob(X = x_{i})}{\Prob(X = x_{j})} = 
%\frac{\Prob(Y = i)}{\Prob(Y = j)} =
\frac{H(x_{i}, x_{j})}{H(x_{j}, x_{i})}
\enspace.
\]
\fi%
%If $n \geq 3$, then
Then
$H(\cdot, \cdot)$
%obeys $\sigma$ transitivity, and moreover, 
%Then [...]
admits a $\sigma$-transitive factoring with logistic $
\sigma(\cdot)$. %$\sigma(u) = \logistic(u)$. %, up to linear (but not affine) scaling.
\iflongversion
In other words, the choice axiom %implies that %the decision rule is a Bradley-Terry model.
requires a Bradley-Terry decision model.
\fi
\end{restatable}

\todo{Discuss noise and cognitive biases. This is more reasonable when probabilities aren't too extreme.}%
\todo{New theorem: noisy decision making.}%
%
%This theorem
\Cref{thm:prop}
uniquely characterizes the Bradley-Terry model, %(logistic $\sigma$),
but is only interesting insofar as the model is believable.
Alternative models of $n$-way decision making are also well-studied, such as noisy observation, where utility + noise is observed for each option, and the largest observation wins, resulting in a probability distribution over outcomes.
In particular, the Bradley-Terry model arises (non-uniquely) from homoskedastic Gumbel noise, but of course the Thurstone–Mosteller model \citep{thurstone1927law,mosteller1951remarks} arises from homoskedastic Gaussian noise (hence the probit link function), and from this perspective, $\sigma$-transitivity can be generalized to $n$-way decision models whenever the distribution with CDF $\sigma$
%can be expressed as
%arises from
decomposes as
the difference of two
%two
i.i.d.\ random variables, i.e., is in the \emph{difference-form decomposable} (DFD) family, although such decompositions \emph{are not %in general unique
always unique} \citep{carnal1989decomposition,ewerhart2025possibility}.
\todo{Not always the case?}
%TODO cite

Overall, with reasonable domain-specific assumptions, we recover other standard modeling classes from our two-stage modeling framework, highlighting the expressive power of our modeling class.
Importantly, assumptions like linearity and monotonicity are domain-dependent and need to be qualitatively justified for the problem context.
Classes, like logistic regressors, that implicitly encode these assumptions may not generate cognitively-faithful models when the human decision process violates these assumptions.
% , which can be true in the case of human decision processes \citep{stray2021you,boerstler2024stability,keswani2025can,}.
% 
Instead, we based our modeling class on generic axioms of human decision processes, which can be further constrained by domain-specific assumptions as necessary.
\section{Empirical Analysis on Kidney Allocation Data} \label{sec:experiments}
% \vspace{-0.05in}

\begin{figure}
    \centering
    \includegraphics[width=\linewidth]{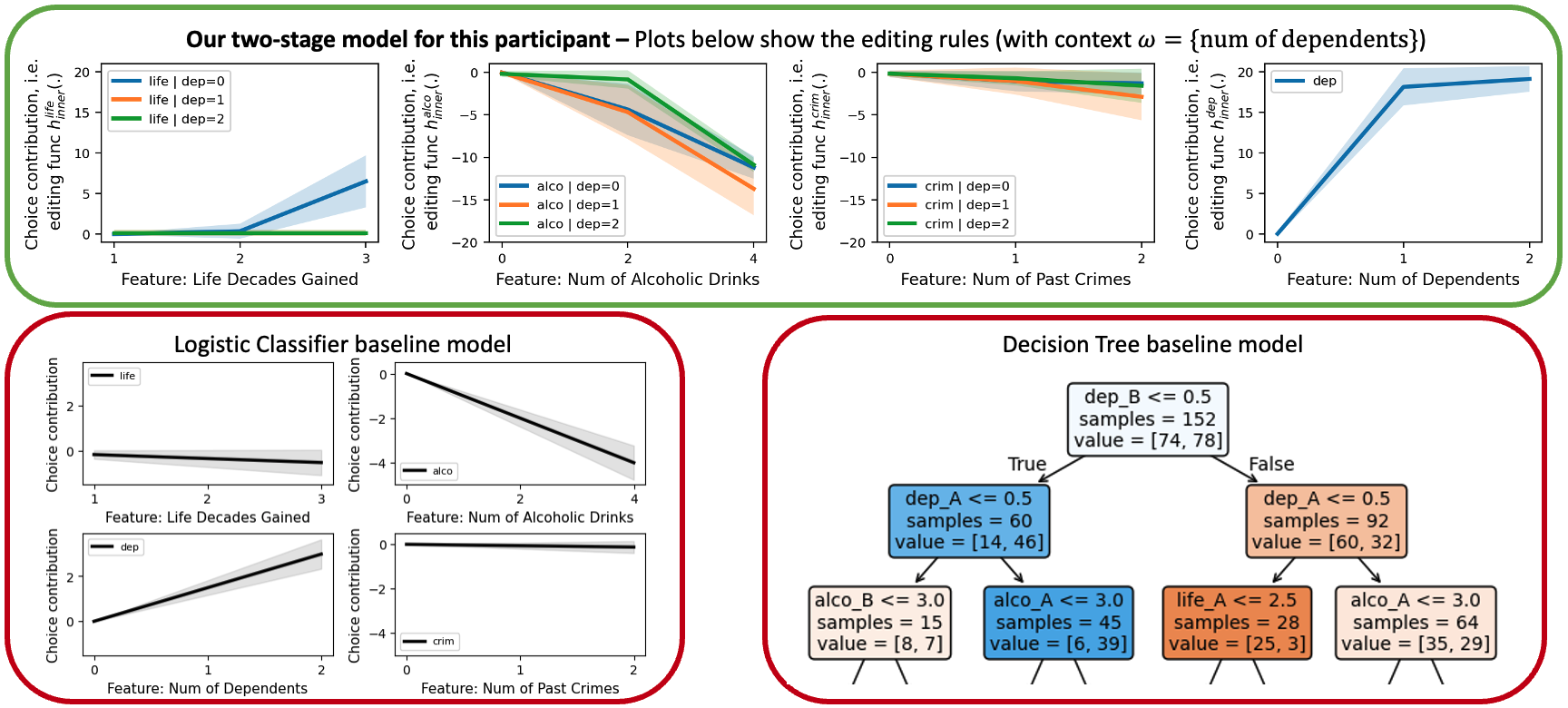}
    % \vspace{-0.6cm}\null
    \caption{Interpretable models learned for
    %Participant 4
    P4
    of %Kidney Allocation Study One
    %kidney allocation study one
    Study One.
    %using different interpretable strategies.
    For our two-stage model and logistic classifier, we plot the %summary
    contributions %of each value
    of each feature to the final choice, %can be
    %are individually computed (using $\inner(\cdot)$ in our case,
    %and using
    %or \emph{feature coefficients} for the logistic model) %, and these are
    i.e., we plot $\inner(\cdot)$.
    % , which by \cref{thm:models}~\cref{thm:models:lin} is linear for the logistic model).
    %For the decision tree, % model,
    We also show the decision tree, truncated to %show only
    the first three layers. %present only the first three layers of the tree. % learned in a single repetition. 
    Note the differences in
    model interpretation
    %interpretability
    between these
    %using different
    strategies, with our model providing insight into the heuristics used by
    %this participant.
    P4.\todo{''our 2-stage model'' logi  classifier?}%
    }
    \label{fig:example_models}
    % \vspace{-0.5cm}\null
\end{figure}

We test our modeling strategy to learn computational models of people's moral judgments regarding the allocation of kidneys to patients based on their medical attributes (e.g., transplant outcomes) and non-medical attributes (e.g., dependents and lifestyles) \citep{boerstler2024stability,chan2024should}.
Prior works have noted several limitations of existing modeling approaches in this domain \citep{keswani2025can}, making this is a relevant case study to analyze our proposed framework. 

% to learn computational models of people's moral judgments for kidney allocation, where
% % 
% % In this domain, 
% prior works have noted several modeling limitations of existing approaches
% %have studied
% study
% moral judgments regarding the allocation of kidneys to patients based on their medical attributes (e.g., transplant outcomes) and non-medical attributes (e.g., dependents and lifestyles) \citep{boerstler2024stability,chan2024should}.
% 
% We use data from these prior works to assess how well we can learn the underlying decision-making model that people use to make kidney allocation judgments.
% 
%For this assessment, we use
% We assess with
% both real-world and synthetic data.

\paragraph{Real-world dataset} We use the dataset %provided by
of \citet{boerstler2024stability},
% that contains kidney allocation judgments from several individuals.
% 
% This dataset 
which spans two studies where
%multiple
participants were presented with several kidney allocation scenarios
% across 10 days 
(15 participants in Study One and 40 in Study Two).
In both studies, each kidney allocation scenario comprises a pairwise comparison between
%two kidney patients,
%patients A and B
two patients
(see \cref{fig:example_ka} for an example).
% , both patients described using a set of features.
% 
Participants were instructed to choose which
%of two presented patients
patient
should be given an available kidney.
% , and their responses are used to learn their moral decision-making process for this setting.
% 
% The patient features differed across the two studies. 
In Study One, the patient features presented are the patient's \emph{number of dependents}, projected \emph{life years gained} from the transplant (LYG), \emph{%number of
alcoholic drinks per day}, and \emph{number of crimes committed}.
In Study Two, the patient features are the patient's \emph{number of elderly dependents}, %\emph{life years gained}
LYG, % from the transplant,
\emph{years waiting} for the transplant, \emph{weekly work hours post-transplant}, and \emph{obesity level}.
\iflongversion
Other dataset details are presented in \cref{sec:appendix_empirical}.
\else
We present further details in \cref{sec:appendix_empirical}.
\fi
\todo{ years vs decades?}

\paragraph{Synthetic dataset} We also created a synthetic dataset representing multiple simulated decision-makers.
% in the kidney allocation setting.
% 
This %synthetic
dataset contains pairwise comparisons of hypothetical kidney patients described using four features: %--- %the patient's
\emph{number of dependents}, LYG, %\emph{life years gained}, % from the transplant,
%years they've been waiting for the transplant list.
\emph{years waiting}, % for the transplant,
and \emph{number of crimes} committed.
The %purpose here
goal
is to simulate decision-makers who explicitly use the heuristics 
% and decision rules 
observed in prior studies to assess how well our model and other baselines %recover
recover these heuristics.
The heuristics 
% used by the simulated decision-makers 
are taken from \citet{keswani2025can}, who provide user-reported qualitative accounts of decision rules people use for kidney allocation decisions.
% In particular, qualitative study of kidney allocation decisions notes the various decision rules employed over patient features \citep{keswani2025can}.
% 
Using their observations, we create five simulated decision-makers, DM1--DM5, each using different decision rules over the presented patient features.
% 
\iffalse
For instance, the process of the first decision-maker (DM1) for a pairwise comparison between
%two patients, A and B,
patients A and B
is described by the following steps:
\fi
For instance, DM1 decides between A and B with the following process:
(a) assign 1 point to the patient with a higher %life years gained
LYG; (b) assign 1 point to each patient with dependents; (c) assign 1 point to each patient
%who has been
on the waiting list for
%more than 6 years
$\null{>}6$ years; (d) %compute difference in total points for A and B (with additive $\mathcal{N}(0,1)$ noise reflecting decision uncertainty)
%Apply additive
add $\mathcal{N}(0, 1)$ noise
(i.e., a homoskedastic Thurstone-Mosteller process) %(to reflect decision uncertainty)\todo{each or just once?} 
to the difference of patient scores, and choose A if %the difference is positive
difference$>0$ and B otherwise (DM1's process is mathematically described in \cref{fig:simulated_participant}).
Other simulated decision-makers (DM2--DM5) also use various heuristics, such as \textit{thresholding}, \textit{diminishing returns}, and \textit{tallying} (described in \cref{sec:appendix_empirical}).
% 
% All of these decision rules are observed in \citep{keswani2025can}, making these a good case study for our learning framework.
% 
% 
%
%We create a dataset with
Our dataset contains
1000 pairwise comparisons from each simulated DM.
% decision-maker.

\iflongversion
\textcolor{brown}{
CYRUS: can we say:
\[
%\Prob(\text{choose A over B}) =
H(A, B) = \Phi( \sgn(A_{\yearsg} - B_{\yearsg}) + \1_{1+}(A_{\deps}) - \1_{1+}(B_{\deps}) + \1_{7+}(A_{\yearsw}) - \1_{7+}(B_{\yearsw})  )
\]
NB due to the $\sgn(A_{\yearsg} - B_{\yearsg})$ term above, which applies an irreducible nonlinearity to a feature of both $A$ and $B$, this rule does not admit an atomic model, thus by the contrapositive of \cref{thm:factoring}~item~\cref{thm:factoring:disc}, and observing that this process clearly respects complementarity, we may conclude it does not respect weak transitivity.
In hindsight, this is obvious, as if patients $A$, $B$, and $C$ are identical except possibly in %life years gained,
LYG, if $H(A, B) = H(A, C) = \Phi(1) \approx 0.841$, i.e., if $A$ stands the most to gain, we still can not determine which of $B$ or $C$ has higher LYG, thus we may only conclude $H(B, C) \in \{\Phi(-1), \frac{1}{2}, \Phi(1) \}$.
\\
Big question: if $\sgn$ were instead an invertible function, would we still have WT, thus would there exist an equivalent atomic model? Yes, but $\Phi$ and the other feature responses may need to change?
}
% , such as thresholding the patient's number of dependents.
% or assuming diminishing returns with the years on transplant list feature .
\fi

%\paragraph{Baselines}
% \paragraph{Methodology and Baselines} We compare our model to several standard learning approaches, namely logistic and elastic-net classifiers, SVM, decision trees, multi-layer perceptron, $k$-NN, and random forests.  Among these, we will use logistic and decision tree models as the ``interpretable" methods we will compare our model's interpretability to. 
\paragraph{Methodology and Baselines} \ccolor{We compare our model to several learning approaches, namely Bradley-Terry and Drift Diffusion models from cognitive science literature,\todo{citations?} and common supervised learning approaches, such as logistic and elastic-net classifiers, SVM, %generalized additive models (GAM)
GAM with spline terms, decision trees, multi-layer perception (MLP), $k$-NN\todo{what $k$?}, and random forests.
% 
%Note that the
Drift Diffusion is the only model we consider that requires %access to
reaction times per scenario,  %(and hence, is not run for the simulated DMs),
thus it is not run for the simulated DMs.
%, while all other models (including ours) do not use reaction time information.
}%
% 
%Among these, we will 
We use logistic models and decision tree models as the ``interpretable" methods to compare our model's interpretability to. %(\cref{fig:example_models}). 
% 
% \ccolor{We also show 
% 
Since moral judgments
%can differ significantly
vary substantially
across individuals,
%we run and test all models 
all experiments operate
over individual-participant-level data. 
For each decision-maker (real or simulated), we use a
%random
70-30 train-test split, and report %the summary statistics of %predictive accuracy over test partitions across
test accuracy over 20 repetitions.
\ccolor{%
For our framework, we minimize the predictive loss to learn the editing functions $\inner^{\cdot, \cdot}$ (which assign score$\in \R$ to each feature value), constraining all $\inner^{\cdot, \cdot}$ to be monotonic (since all features in our dataset can be seen to impact the final choice monotonically)
We implement two variants of this framework: (A) with cross-entropy loss and $\sigma(x) = (1+e^{-x})^{-1}$, and (B) with hinge loss and $\sigma$ as the identity function.\todo{???}
% For our framework, we minimize the hinge loss to learn the editing functions $\inner^{\cdot, \cdot}$ (which assign score$\in \R$ to each feature value), constraining all $\inner^{\cdot, \cdot}$ to be monotonic (since all features in our dataset can be seen to impact final choice monotonically). This implicitly assigns $\sigma$ a scalar variant of the identity function. Context $\omega$ is limited to one feature for the real-world dataset, chosen using cross-validation.
% 
The first variant is aligned with the probabilistic framework of \cref{thm:factoring}, while the second framework is better suited to assessing our learned model on the ``hard classification'' metric of 0-1 predictive accuracy.}
% 
% \ccolor{For our framework, we set $\sigma(x) = (1+e^{-x})^{-1}$ and minimize the cross-entropy loss to learn the editing functions $\inner^{\cdot, \cdot}$ (which assign score$\in \R$ to each feature value), constraining all $\inner^{\cdot, \cdot}$ to be monotonic (since all features in our dataset can be seen to impact the final choice monotonically). 
% }
% 
% This implicitly assigns $\sigma$ a scalar variant of the identity function. 
Context $\omega$ is limited to one feature for the real-world datasets, chosen using cross-validation, and $\emptyset$ for the synthetic dataset.
% 
%It is kept empty for the synthetic dataset.
% 
% Additional implementation details are provided in \cref{sec:appendix_empirical}.
\Cref{sec:appendix_empirical} provides additional implementation details.
\begin{table}
    \centering
    % \renewcommand{\quad}{\!\!}%TODO no!
    % \vspace{-0.25cm}\null
    % \small
    % \footnotesize
    % \setstretch{0}
    % \renewcommand{\arraystretch}{0.8}
    % \addtolength{\tabcolsep}{-0.35em}    
    % \null\!\!\!%
    \begin{tabular}{lccc}
    \toprule
        & \multicolumn{3}{c}{%\!\!\!
        Average Accuracy (stddev% in brackets
        )
        }\\
         Model & \scalebox{1}[1]{\!\!Study One} & \scalebox{1}[1]{\!\!Study Two} & %\!\!\scalebox{0.97}[0.99]{Simulated DM}\!\! \\
         \scalebox{1}[1]{\!\!\!Simulated\!} \\
    \midrule
        \textbf{Cognitive models} \hspace{-1cm} \\
         \quad \scalebox{1}[1]{Drift-Diffusion} & .89 (.05) & .88 (.05) & -- \\
         \quad \scalebox{1}[1]{Bradley-Terry} & \textbf{.90} (.06) & .78 (.06)  & .77 (.06)\\
    % \midrule
        \textbf{Supervised learning models}\hspace{-2cm} \\
         \quad Logistic Clf & \textbf{.90} (.06) & {.89} (.05) & .85 (.07) \\
         \quad Elastic Net & .89 (.04) & .88 (.05) & .85 (.07) \\
         \quad SVM & .89 (.06) & .89 (.05) & .85 (.07) \\
         \quad GAM & .87 (.09) & .84 (.11)  & .88 (.08) \\
         % \scalebox{0.93}[0.99]{Decision Tree\!\!} & .83 (.06) & .79 (.06) & .82 (.11) \\
         \quad \scalebox{1}[1]{Decision Tree\!\!} & .83 (.06) & .79 (.06) & .82 (.11) \\
         \quad k-NN & .85 (.06) & .82 (.05)  & .79 (.08) \\
         \quad MLP & .89 (.05) & .86 (.06) & .87 (.08) \\
         % \scalebox{0.85}[0.98]{Random Forest\!\!\!} & .86 (.05) & .85 (.04) & .87 (.08) \\
         \quad \scalebox{1}[1]{Random Forest\!\!\!} & .86 (.05) & .85 (.04) & .87 (.08) \\
    \midrule
        \textbf{Our Models}\\
         \quad %w/
            \scalebox{1}[1]{\!cross-entropy loss\!\!\!\!} & \textbf{.90} (.06) & \textbf{.90} (.05) & \textbf{.89} (.10) \\
         \quad %w/
            hinge loss & \textbf{.90} (.06) & {.89} (.06) & \textbf{.89} (.08) \\
    \bottomrule 
    \end{tabular}
    % 
    % \vspace{-0.25cm}\null
    \caption{%Individual-level 
    % \small 
    Performance %Accuracy
    statistics of all individualized
    models on %the Kidney Allocation Study Datasets
    kidney allocation datasets. %The two-stage model we learn achieves the same or better accuracy than baselines for all datasets.
    \todo{TODO explicate}
    }
    \label{tab:agg_performance}
    %\vspace{-0.5in}\null
    % \vspace{-0.28in}\null
\end{table}
% \end{wraptable}

\paragraph{Results}
Across all datasets, we observe that
%the models returned by our strategy
our cognitively-motivated models
achieve high accuracy and provide deeper insight into decision-making processes than other baselines.
% 

% \begin{wraptable}{r}{6cm}
%     \centering
%     \vspace{-0.85cm}\null
%     \small
%     % \renewcommand{\arraystretch}{0.8}
%     \addtolength{\tabcolsep}{-0.2em}    
%     \null\!\!\begin{tabular}{lccc}
%     \toprule
%         & \multicolumn{3}{c}{%\!\!\!
%         Average Accuracy (stddev% in brackets
%         )
%         }\\
%          Model & \scalebox{0.97}[0.99]{\!Study One} & \scalebox{0.97}[0.99]{\!Study Two} & %\!\!\scalebox{0.97}[0.99]{Simulated DM}\!\! \\
%          \scalebox{0.97}[0.99]{\!Simulated\!} \\
%     \midrule
%          Logistic Clf & \textbf{.90} (.06) & \textbf{.89} (.05) & .85 (.07) \\
%          Elastic Net & .89 (.04) & .88 (.05) & .85 (.07) \\
%          SVM & .89 (.06) & .89 (.05) & .85 (.07) \\
%          \scalebox{0.93}[0.99]{Decision Tree\!\!} & .83 (.06) & .79 (.06) & .82 (.11) \\
%          k-NN & .85 (.06) & .82 (.05)  & .79 (.08) \\
%          MLP & .89 (.05) & .86 (.06) & .87 (.08) \\
%          \scalebox{0.85}[0.98]{Random Forest\!\!\!} & .86 (.05) & .85 (.04) & .87 (.08) \\
%     \midrule
%         Our Model & \textbf{.90} (.06) & \textbf{.89} (.06) & \textbf{.89} (.08) \\
%     \bottomrule 
%     \end{tabular}
%     \vspace{-0.25cm}\null
%     \caption{\small Performance of all models on %the Kidney Allocation Study Datasets
%     kidney allocation datasets. %The two-stage model we learn achieves the same or better accuracy than baselines for all datasets.
%     }
%     \label{tab:agg_performance}
%     \vspace{-0.2in}
% \end{wraptable}

\textit{Aggregated Performance.}
% 
%The performance of all models across the decision-maker pools
The mean accuracy of each model over all participants\todo{mean or weighted?} is 
%presented
shown
in \cref{tab:agg_performance}.
Overall,
%across
on
both real %-world
and synthetic data, our framework produces models %with %the same or better accuracies than the most accurate baselines.
\emph{at least %the accuracy of
as accurate} as all baselines.
% 
% But beyond sufficient accuracy, the primary advantage of our framework is that it 
Individual-level performance 
% for all participants 
is presented in \cref{sec:appendix_empirical}. 
Additionally, 
% our model also produces meaningful interpretations of the decision-maker's process.
% 
% This focus leads to our frameworks using a significantly smaller hypothesis class compared to the other frameworks.
% 
editing functions $\inner$ learned for the decision-makers provide insight into how they process the input features to reach their final decision. We demonstrate this interpretability through case studies of the editing functions learned for the simulated decision-maker DM1 and a real-world participant P4 from Study One.
% of the kidney allocations dataset below.
% 

% And more importantly, it provides deeper insight into how any participant processes the available input to reach their final decision. We demonstrate this advantage through a case study of two participants below.
% 

% \begin{wraptable}{r}{7.0cm}
%     \centering
%     \vspace{-0.5cm}\null
%     \small
%     \addtolength{\tabcolsep}{-0.4em}    
%     \begin{tabular}{lccc}
%     \toprule
%         & \multicolumn{3}{c}{\!Average Accuracy (std dev in brackets)}\\
%          Model & Study One & Study Two & \!\!\scalebox{0.97}[0.99]{Simulated DM}\!\! \\
%     \midrule
%          Logistic Clf & .90 (.06) & .89 (.05) & .74 (.02) \\
%          Elastic Net & .89 (.04) & .88 (.05) & .74 (.01) \\
%          SVM & .89 (.06) & .89 (.05) & .74 (.02) \\
%          Decision Tree & .83 (.06) & .79 (.06) & .68 (.02) \\
%          k-NN & .85 (.06) & .82 (.05)  & .72 (.01) \\
%          MLP & .89 (.05) & .86 (.06) & .74 (.02) \\
%          Random Forest\!\! & .86 (.05) & .85 (.04) & .74 (.01) \\
%     \midrule
%         Our Model & .90 (.06) & .89 (.06) & .77 (.02) \\
%     \bottomrule 
%     \end{tabular}
%     \vspace{-0.25cm}\null
%     \caption{\small Performance of all models on %the Kidney Allocation Study Datasets
%     kidney allocation datasets. The two-stage model we learn achieves similar or better accuracy than baselines for all datasets.}
%     \label{tab:agg_performance}
%     \vspace{-0.1in}
% \end{wraptable}

\textit{%Performance for participant
Participant P4 in Study One.}
% 
% Individual participant-level performance for the real-world data is presented in \cref{sec:appendix_empirical}. 
% Across all participants, our model provides similar or better performance than baselines while using hypothesis classes with much smaller complexities.
% 
% Using our modeling class provides insight into each participant's decision-making process.
% 
% 
The two-stage model learned for
%participant 4
P4 (w/ hinge loss) is illustrated in \cref{fig:example_models} (top row).
% 
% The two-stage model we learn for this participant 
The editing function plots show the following nuances of their decision-making process for pairwise comparisons of kidney patients: 
(a) \textit{number of dependents} and \textit{number of alcoholic drinks} are the two most important features, as reflected by the large choice contributions of these features; 
(b) \textit{number of past crimes} is considered mostly irrelevant by this participant; 
(c) \textit{life decades gained} is relevant only when the patient has zero dependents, indicating a conditional interaction between these two features;
% \textit{life decades gained} and \textit{number of dependents};
% 
(d) for \textit{number of dependents}, a difference of 1 vs 0 dependents is much more significant to the final choice than a difference of 2 vs 1 dependents --- approximately a threshold decision rule such that any non-zero number of dependents contributes equally;
(e) similarly for \textit{life decades gained} and  \textit{number of alcoholic drinks}, for certain conditions on \textit{number of dependents}, the participant employs threshold decision rules.
Our model is able to learn all of these decision-making nuances with an accuracy of 0.78 ($\pm$~0.05), but a trained logistic classifier (bottom-left in \cref{fig:example_models}) has a lower accuracy of 0.76 ($\pm$~0.04) and only uncovered points (a) and (b) above, while a trained decision-tree model (bottom-right in \cref{fig:example_models}) has an accuracy of 0.70 ($\pm$~0.03) and only uncovered points (a), (c), and (d).
% 
% Additionally, by learning structured decision rules, we obtain higher predictive performance; our model achieves 0.81 accuracy for this participant, while the Logistic classifier and Decision Tree achieve accuracies of 0.72 and 0.71, respectively.

\begin{figure}
    \centering
    \includegraphics[width=\linewidth]{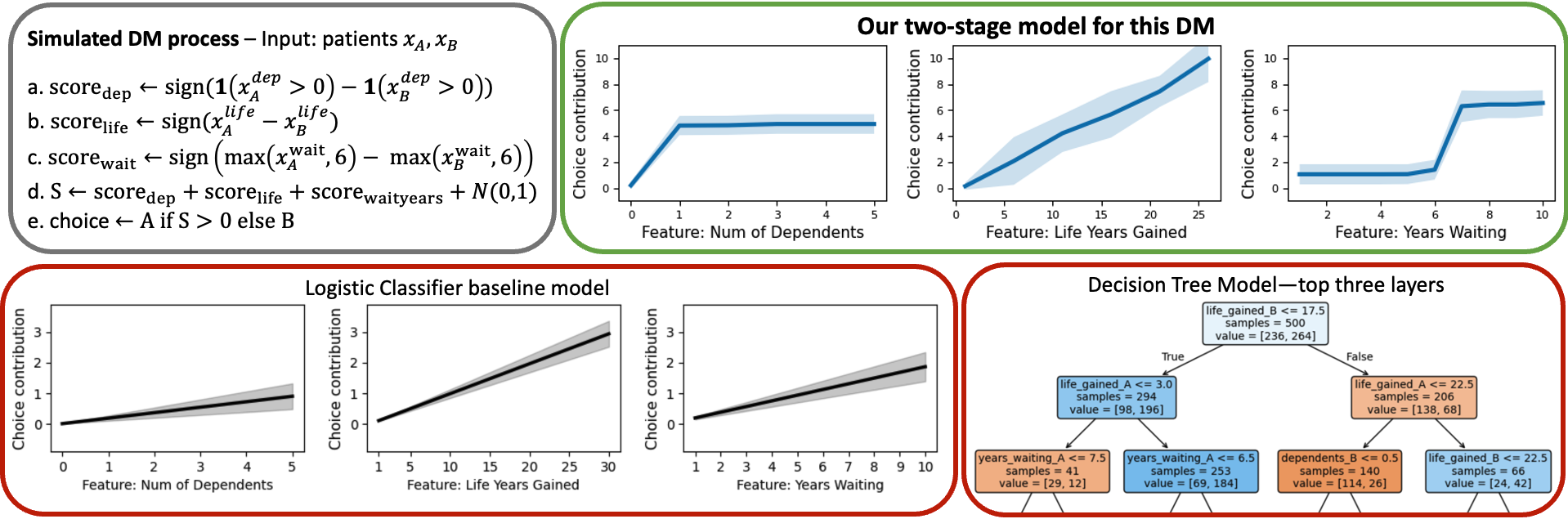}
    % \vspace{-0.65cm}\null
    \caption{
    % \small 
    Models learned for the simulated decision-maker DM1. On the top left, we present the mathematical description of the DM1's process. Once again, we present each model's interpretation of the process used by the decision-maker.
    %And one can see that our 
    Observe that our
    two-stage model most accurately captures the simulated process.}
    \label{fig:simulated_participant}
    % \vspace{-0.5cm}\null
    % \vspace{-0.75cm}\null
    % \vspace{-0.5cm}\null
\end{figure}

\textit{%Performance for the simulated decision-maker
Simulated DM1.}
% For the simulated decision-maker, 
% the decision process is described above and also presented in \cref{fig:simulated_participant}.
% 
% \cref{tab:agg_performance} shows that our model achieves higher predictive accuracy than other methods.
% 
% Beyond accuracy, \cref{fig:simulated_participant} (top right) further 
%For DM1, 
\Cref{fig:simulated_participant} (top right) shows how our model recovers the rules used by DM1. %this decision-maker.
E.g., for the feature \textit{number for dependents}, $\smash{\inner^{\deps}}$ captures the increase in choice contribution when this feature has a non-zero value, and for the \textit{years waiting} feature, $\smash{\inner^{\yearsw}}$ learns the %increase in contribution when the feature value is ``than 6
increased contribution of values exceeding 6 years.
Both successfully represent the threshold functions used in DM1.
Note that neither of these observations can be easily inferred from logistic or decision tree models (bottom row of \cref{fig:simulated_participant}).
Our model is also more accurate (0.76 $\pm$~0.02) than logistic (0.74 $\pm$~0.02) and decision tree (0.68 $\pm$~0.04) models.
% learned over this dataset.
% 
\Cref{sec:appendix_empirical} shows
%our model's performance
similar results
for other simulated decision-makers.
% DM2--DM5.
% 
Overall, our models provide a better understanding of the decision-making processes, without sacrificing predictive accuracy.

% \vspace{-0.05in}
\section{Discussion, Limitations, and Future Work} \label{sec:limitations}
% \vspace{-0.08in}

\todo{[fidelity adjective]

[decisions we make and the process by which we make them]

[cognitively-faithful: connect to the PROCESS]

[process matters: human-like errors?]
}

%In this paper, we have provided
This work provides
a modeling strategy to learn human decision processes from 
% human responses to 
pairwise comparisons.
The primary goal of this strategy is \textit{fidelity}, such that best-fit models from these classes are more faithful to humans' actual decision-making than
%other
standard hypothesis spaces.
% 
% The empirical results supported our hypothesis that greater \textit{fidelity} would also lead to greater model \textit{accuracy} in some cases.
% 
Our theoretical analysis demonstrates that these classes emerge from natural decision-making axioms, and our empirical analysis shows that greater \textit{fidelity} would lead to greater model \textit{accuracy} in some cases.
% highlights the advantages of our framework in learning accurate and interpretable models %for people's
% of
% kidney allocation judgments.
% in a kidney allocation study.
% 
% Overall, the framework forwarded in this paper advances computational 
%
The advantages of cognitively-faithful models will depend on the application. 
In applications like online recommender systems, it may be sufficient to predict human behavior accurately, without simulating their actual decision processes.
However, cognitively-faithful models would be desirable in other high-stakes AI domains (e.g., healthcare and sentencing), where the interpretability and alignment of AI’s decisions with the users is pivotal in establishing trust
% in AI 
\citep{jacovi2021formalizing}. 
Cognitively-faithful models would also be preferable in domains with no ``ground truth'' (like kidney allocation), where model validation requires the user to understand and evaluate the model’s decision process.
Lastly, cognitively-faithful models are essential for any use of AI in moral domains, since stakeholders expect an AI to justify its moral decisions in a similar manner 
% and to the same extent 
as humans \citep{lima2021human}. 
% 
%For any AI application where one or more
%For AI applications where any
%of the above criteria are satisfied,
\iflongversion
For AI applications that satisfy any of the above criteria,
% (i.e., when AI is used to make high-stakes decisions, used in domains with no “ground truth”, and/or used in a moral domain), 
%cognitively-faithful models would be preferable over
cognitively-faithful models are preferred over generic modeling methods.
\fi

\todo{[move limitations to intro if possible]}
\iflongversion
TODO

Yet, the choice to use cognitively-faithful models should be informed by stakeholders. Our axiomatic analysis can be viewed as a technical tool to allow practitioners to make better-informed choices about properties of the models they build to assist their users, as it is easy to qualitatively assess whether properties like transitivity and feature independence are desirable when simulating a user's decision processes.

At the same time, not all cognitive processes are suitable to replicate in a computational model, such as heuristics that people employ when facing decision fatigue \citep{pignatiello2020decision} and cognitive or framing biases that negatively affect their judgments
% , such as framing or priming biases 
\citep{tversky1981framing}. 
While these heuristics should not be a part of the computational model that assists the user in future decision-making, our two-stage framework could still be used to learn which heuristic rules are present in the data, to be able to ``correct'' for them in the learned computational model, e.g., by replacing the heuristic 
in the learned model 
with a more appropriate user-specified decision rule. 
Hence, even in this case, learning the cognitive processes is useful  
% using our approach could help 
in correcting the identified biases when creating an “ideal” cognitively faithful model.
% for the user.
% Encoding human decision-making heuristics in an AI model may not always be necessary.
% for every application.
%
\fi

% There are also certain limitations of our approach.
% 
%The
%As for limitations,
Our mathematical model is motivated by human behavior, but due to the complexity of the latter, any such effort is inherently limited.
The interpretability features of our framework need further validation through real-world user studies.
While this work provides a proof-of-concept for a computational framework to learn from human decisions, future user studies can help assess the extent to which users understand and trust the learned two-stage models. 
% 
% Encoding human decision-making heuristics in an AI model may not always be necessary.
% % for every application.
% % 
% In applications like online recommender systems, it may be sufficient to predict human behavior accurately, without simulating their actual decision processes.
% % 
% In domains where decision heuristics reflect cognitive biases, one might not want an AI to simulate the heuristics that propagate such biases.    
% % 
% Nevertheless, even in these cases, identifying humans' actual decision models could help 
% % explain and predict their decisions more accurately and 
% facilitate correction for identified biases.
% that also are related to the complexities of human decision-making processes.
% 
% Our axioms can also capture the ``ideal'' decision-making process of an individual; i.e., the one they would prefer to follow in the absence of any internal/external constraints.
% % 
% However, human decisions are known to deviate from the ideal in many ways, e.g., transitivity and complementarity violations \citep{tversky1990causes}, among others.
% % 
% Nevertheless, our focus in this paper was to characterize the ideal human decision-making process, and future work can explore ways to extend our hypothesis classes to incorporate known violations of the stated axioms.
% 
Our axioms characterize the ``ideal'' decision-making process of an individual; i.e., the process they would prefer to follow in the absence of any internal/external constraints, or the process that describes their judgments about what ``should be done.''
However, human decisions are known to deviate from the ideal in many ways, e.g., %well-known
sometimes exhibiting transitivity and complementarity violations \citep{tversky1990causes}, among others.
One thus wonders, does %idealized elicitation processes %are aligning
precluding such deviations
aligns an AI system to the processes humans actually use, or instead to some idealized version thereof?
% 
% 
% Our axioms can accommodate learning of the `ideal'' decision-making process of an individual; i.e., the one they would prefer to follow in the absence of any internal/external constraints, when learning from data that describes participants' judgments about what one ``should do".  They can also identify the actual decision-making process of an individual when learning from data about one's actual implemented choices. 
Both have their virtues, but we must be aware of their differences.
We hope that future work can study this issue, and explore ways to extend our two-stage hypothesis classes to quantify deviations of one's ``ideal" judgments from their implemented choices.
\iflongversion
Our empirical analysis of kidney allocation judgments may also be limited, since these are hypothetical moral judgments, which could differ from real-world moral decisions \citep{yu2022neural}.
While the modeling strategy should still be applicable for real-world decisions, the impact of this practical limitation can be investigated in future work.
% , as hypothetical moral judgments are also known to differ from real moral decisions in certain settings \citep{yu2022neural}.
% 
\fi
\iflongversion
Beyond these, our work focuses primarily on pairwise comparisons,
% . Preference elicitation literature has studied a variety of other elicitation methods (e.g. ranking) and 
and future work can explore extensions of our approach to other elicitation frameworks (e.g., ranking) as well.
% can similarly be studied in future work.
\fi

% \begin{table}[]
%     \centering
%     \begin{tabular}{lcc}
%     \toprule
%         & \multicolumn{2}{c}{Average Accuracy (std dev in brackets)}\\
%          Model & Study One & Study Two \\
%     \midrule
%          Logistic Classification & 0.89 (0.06) & 0.89 (0.05)  \\
%          SVM & 0.89 (0.06) & 0.89 (0.05)  \\
%          Decision Trees & 0.83 (0.06) & 0.79 (0.06) \\
%          MLP & 0.89 (0.05) & 0.86 (0.06) \\
%          Random Forests & 0.86 (0.05) & 0.85 (0.04) \\
%     \midrule
%         Our Two-Stage Framework & 0.90 (0.06) & 0.89 (0.06) \\
%     \bottomrule \\
%     \end{tabular}
%     \caption{Performance of all models on the Kidney Allocation Preferences Study Datasets.}
%     \label{tab:agg_performance}
% \end{table}

\vfill
\pagebreak[4]
\clearpage

% \paragraph{Acknowledgments}
\section*{Acknowledgments}

VK, CC, JSB, and WSA are grateful for the financial support from OpenAI and Duke University.

HH acknowledges support from the CMU-NIST Cooperative Research Center on AI Measurement Science \& Engineering (AIMSEC), and the AI Research Institutes Program funded by the National Science Foundation under AI Institute for Societal Decision Making (AI-SDM), Award No. 2229881. Any opinions, findings, conclusions, or recommendations expressed in this material are those of the authors and do not reflect the views of the funding agencies.

% WSA discloses that he is owner and founder of Patient Preference Predictors, Inc., and a member of the Grow Therapy AI-Advisory Panel.

\section*{Reproducibility statement}
For theoretical results, full proofs of all results are provided in Appendix~\ref{appx:proofs}.
To facilitate reproducibility of the empirical results, the primary methodological details of our simulations are described in Section~\ref{sec:experiments}, and additional descriptions, including simulated DM construction processes, training procedures, and hyperparameter values, are given in Appendix~\ref{sec:appendix_empirical}.
Additionally, source code for these simulations is available at \url{https://github.com/vijaykeswani/Cognitively-Faithful-Decision-Models/}.
% 

% \clearpage
% \small
\bibliography{references}
\bibliographystyle{iclr2026_conference}

\clearpage

\appendix

\pagebreak[4]

\input{appendix}

\input{empirical_appendix}

\end{document}

%% file: appendix.tex
\section{Proof Compendium}
\label{appx:proofs}

\iffalse
\begin{restatable}[Transitive Symmetry Laws]{lemma}{lemma:tsl}

Suppose CS, then PR and symmetry.

\end{restatable}
\begin{proof}

\end{proof}

\begin{restatable}[Transitive Symmetry Laws]{lemma}{lemma:tsl}

Suppose CS, then PR and symmetry.

\end{restatable}
\begin{proof}

\end{proof}
\fi

\thmfactoring*
\begin{proof}

The structure of this result is a bit convoluted, due to its presentation in four parts.
We first show item~\ref{thm:factoring:disc} in isolation using a straightforward explicit construction.
\todo{We then show that the conditions of item~\ref{thm:factoring:trans} imply the conditions of item~~\ref{thm:factoring:disc}, i.e., \emph{symmetry}, and the assumed addition of \emph{continuity} of $f(\cdot, \cdot)$ is sufficient to yield item~\ref{thm:factoring:trans}.}%
We then show that the conditions of item~\ref{thm:factoring:trans} imply \emph{symmetry} of the transitivity law, and the assumed addition of \emph{continuity} of $f(\cdot, \cdot)$ is sufficient to yield item~\ref{thm:factoring:trans}.

From there, we address items~\ref{thm:factoring:uncond}~\&~\ref{thm:factoring:cond}.
Note that because NC is a special case of CIC with condition set $\omega=\emptyset$, we show only item~\ref{thm:factoring:cond}, as item~\ref{thm:factoring:uncond} is a special case of item~\ref{thm:factoring:cond}.

\medskip

\todo{Weaker form only assuming countability:}

We first show item~\ref{thm:factoring:disc}.
For countable $\X$, it's trivially true that each $x \in \X$ can be encoded as a power of $2$.
Then, $\inner(x_{1}) - \inner(x_{2})$ is then 0 iff $x_{1} = x_{2}$, which by axiom~1, necessitates $\outer(0) = \frac{1}{2}$.
Otherwise, $x_{1} \neq x_{2}$, and %can otherwise
we can recover $x_{1}$ and $x_{2}$ uniquely (the larger %can be recovered as $\inf_{\varepsilon> 0} \log_{2}(\abs{\inner(x_{i}) - \inner(x_{j})} + \varepsilon) $
by rounding the absolute difference up to a the nearest power of 2, i.e., by taking $\log_{2}(\abs{\inner(x_{i}) - \inner(x_{j})})$, the smaller via subtraction, and the sign indicating which is which).
\iffalse
Specifically, it holds that
\[
\inner(x_{i}) - \inner(x_{j}) = 2^{i} - 2^{j} \implies 
\]
\fi%
Therefore, using this construction of $\inner$, for an arbitrary mapping between $\X$ and $\mathbb{Z}$, we can construct a corresponding $\outer$ that reconstructs any possible decision rule $H(\cdot, \cdot)$ over $\X$.
%}

\todo{From rebuttal:
"The proof for the "Atomic Model" item 1 appears invalid...": Apologies, the result was confusingly presented. We initially thought axiom 2 (weak transitivity) was required, but later found that it was not, yet we neglected to remove the assumption from the statement. Technically speaking, axiom 1 (complementarity) is stronger than necessary, though it can’t be omitted entirely. Assuming complementarity ensures the reflexive property that ½
, which our construction relies upon (we will clarify this detail in the proof), as our construction would not permit 
, and omitting axiom 1 as an assumption would allow for this. We will clarify that this assumption yields this property.

We do see the perspective that the argument feels somewhat trivial, however we don’t think a straightforward proof should detract from the utility of the claim. The straightforward constructive structure of item (1) is not so rare for results on countable domains, such as the famous discrete Deep Sets theorem. The continuous version of the theorem is much deeper (item 2), and relies on complex properties of transitive symmetric functions: the point of item (1) is really to show what is still available under minimal assumptions in the discrete case. We hope that the result is clearer and more satisfying when viewed in this light.
}

\todo{Add the advanced versions}
\todo
{NB: This didn't need WT, nor complementarity, but we can also conclude that $\outer(u) = 1 - \outer(-u)$.}

\todo
{TODO: What about symmetry?
TODO Assuming that $f(u, v)$ is a symmetric monotonically increasing (in both arguments) continuous function, may be enough?
\\
Yes, then apply cont DS theorem. We don't get strict invertibility though.
\\
What if a simple construction works?
\\
}

\todo{Problems enountered below:}

%
\medskip

We now show item~\ref{thm:factoring:trans}.

We first show that these assumptions imply $f(u, v) = f(v, u)$ for all $u, v \in (0, 1)$, i.e., the transitivity law is symmetric.
%To see this, observe that by weak transitivity, it holds
Assume WLOG $u \geq \frac{1}{2}$ and $v \geq \frac{1}{2}$, as by complementarity and WT, the remaining cases can be derived\todo{How?}.
Now, select an arbitrary $x_{0}$ in $\X$.
For any $\varepsilon \in (0, \frac{1}{2})$, by \emph{codomain span}, there exists an infinite sequence $x_{1}, x_{2}, \dots$ such that $H(x_{i-1}, x_{i}) = \frac{1}{2} + \varepsilon$, and moreover, let $k_{1},k_{2}$ be the smallest integers such that $H(x_{0}, x_{k_{1}}) \geq u$ and $H(x_{0}, x_{k_{1}}) \geq v$, respectively.
Note that for any $k$,
\todo{We haven't shown this, but will later find:
\[
H(x_{0}, x_{k}) = \sigma(k \sigma^{-1}(\frac{1}{2} + \varepsilon)
\enspace,
\]}%
$H(x_{0}, x_{k})$ is an increasing function, achieving $\lim_{k \to \infty} H(x_{0}, x_{k}) = 1$.
Furthermore, by continuity, in the limit as $\varepsilon \to 0$, $H(x_{0}, x_{k_{1}})$ and $H(x_{0}, x_{k_{2}})$ converge to $u$ and $v$, respectively.
However, the grouping of this sequence into a segment of $k_{1}$ points approximating $u$ and $k_{2}$ points approximating $v$ was arbitrary, and the weak transitivity operator $f(\cdot, \cdot)$, by construction, must be associative.
We may thus conclude that, due to continuity and the limit as $\varepsilon \to 0$, $f(u, v) = f(v, u)$.

\todo{Why is k finite sufficient?}

\iffalse
Chain form

\begin{align*}
H(x_{1}, x_{3}) &= f( H(x_{1}, x_{2}), H(x_{2}, x_{3}) ) & \\
 &= \rho( \phi \circ H(x_{1}, x_{2}) + \phi \circ H(x_{2}, x_{3}) ) & \textsc{Deep Sets Theorem} \\
 &= 
\end{align*}
\fi

Now observe that, by the continuous deep sets theorem (Theorem~7 of \citet{zaheer2017deep}), since $f(\cdot, \cdot)$ is symmetric, %it holds
there exist continuous functions $\rho,\phi: \R \to \R$ such that
\[
H(x_{1}, x_{3}) = f( H(x_{1}, x_{2}), H(x_{2}, x_{3}) ) = \rho( \phi \circ H(x_{1}, x_{2}) + \phi \circ H(x_{2}, x_{3}) )
\enspace.
\]

Assume WLOG that $\phi(\frac{1}{2}) = 0$ (as any shifting  to $\phi$ can be absorbed by $\rho$ in the composition).

Complementarity implies reflexivity, i.e., it holds that $H(x_{1}, x_{1}) = 1 - H(x_{1}, x_{1}) = \frac{1}{2}$.
Consequently, we observe the f
Now, again by reflexivity, and applying our derived weak transitive law, it holds
\begin{align*}
H(x_{1}, x_{2}) &= \rho( \phi \circ H(x_{1}, x_{2}) + \phi \circ H(x_{2}, x_{2}) ) & \\
 &= \rho( \phi \circ H(x_{1}, x_{2}) ) & \phi \circ H(x_{2}, x_{2}) = \phi(\frac{1}{2}) = 0 \enspace. \\
\end{align*}
Thus $\rho(\phi(u)) = u$ for all $u \in (0, 1)$.

Now, we argue that $\phi(u) + \phi(1 - u) = 0$, which implies $\phi(u) = -\phi(1 - u)$.
Observe the following cancellation: $H(x_{1}, x_{1}) = f(H(x_{1}, x_{2}), H(x_{2}, x_{1})) = \rho(\phi(u) + \phi(1 - u)) = \rho(0) = \frac{1}{2}$.
By continuity, and the property $\rho(\phi(u)) = u$, this implies that $\phi$ is either an increasing or a decreasing function.
Assume now that, WLOG, $\phi$ (and consequently $\rho$) are both increasing functions, as output-negation of $\phi$ may be counterbalanced by input-negation of $\rho$.

To show that $\rho$ and $\phi$ are true inverses, we require also that $\phi(\rho(v)) = v$ for all $v \in \R$.
We show this by proving that $\rho(\cdot)$ is a strictly increasing continuous function on $\R \to (0, 1)$.
Because $\phi$ is a strictly increasing function, and $\rho(\phi(1)) = 1$, clearly $\rho(\infty) = 1$ and $\rho$ is a strictly increasing continuous function \emph{on the domain} $(\phi(0), \phi(1))$.
The only issue is that, if $\phi(1) = -\phi(0) < \infty$, then
$\rho$ is only \emph{weakly increasing} (constant) 
%the behavior of $\rho$ is not specified
over the rest of $\R$\todo{Check that}.
Now, suppose BWOC that $\phi(1) = c $ for some $c < \infty$.
By codomain span and increasingness of $\rho$, there exist $x_{1}, x_{2}, x_{3}, x_{\infty}$, such that $H(x_{1}, x_{2}) \neq H(x_{1}, x_{3})$, $H(x_{i}, x_{\infty}) = 1$ for each $i$, and the weak transitivity law is violated, as $H(x_{1}, x_{2})$ and $H(x_{1}, x_{3})$ \emph{can not possibly} both be computed as $f(1, 1)$.
NB this impossibility does not depend on the functional form we derive: Rather, there is a fundamental incompatibility between weak transitivity and prediction with certainty.
We thus conclude that the image of $\phi(\cdot)$ is $\R$, hence $\phi$ and $\rho$ are true inverses, and both strictly monotonically increasing continuous functions.
Moreover, $\rho$ takes the form of the CDF of a symmetric continuous random variable with support $\R$.
%there exists some $c$ such that $\$
\iffalse
Moreover, if $\phi$ has image $(-c, c)$ then $\rho$ is strictly increasing over domain $(-c, c)$, so we need only show that $c = \infty$.
Suppose BWOC that $c < -\infty$. Then there exists some $x_{1}, x_{2}, x_{3}$ such that
%$H(x_{1}, x_{2}) = \rho(\phi(x_{1}) - \phi(x_{2})) = \rho(\infty) = 1$, and moreover,
\[
H(x_{1}, x_{3}) = \rho( \phi( H(x_{1}, x_{2}) ) + \phi( H(x_{2}, x_{3})
\]
\fi

Henceforth, we have operated purely in terms of the predicted probabilities, i.e., $H(\cdot, \cdot)$, but the result requires us to draw conclusions in terms of some $\inner(x)$.
We now argue for the existence of such a decomposition.
Essentially, we ``eliminate the middle man,'' as
\begin{align*}
H(x_{1}, x_{3}) &= f( H(x_{1}, x_{2}), H(x_{2}, x_{3}) ) & \textsc{Weak Transitivity} \\
&= \rho( \phi \circ H(x_{1}, x_{2}) + \phi \circ H(x_{2}, x_{3}) ) & \textsc{See Above} \\
&= \rho( \phi \circ H(x_{1}, x_{2}) - \phi \circ H(x_{3}, x_{2}) ) & \textsc{Complementarity} \\
&= \rho\left( \left( \inner(x_{1}) - \inner'(x_{2}) \right) - \left( \inner(x_{3}) - \inner'(x_{2}) \right) \right) & \textsc{See Below} \\
&= \rho\left( \inner(x_{1}) - \inner(x_{3}) \right) \enspace. & %\\
\end{align*}
The step marked \textsc{See Below} is rather subtle, but observe that it must hold BWOC for some functions $\inner,\inner': \X \to \R$\todo{Complex?}, as if it did not, then there would exist some $x_{2}$, $x_{2}'$ such that $H(x_{1}, x_{3}) = f( H(x_{1}, x_{2}), H(x_{2}, x_{3}) ) \neq f( H(x_{1}, x'_{2}), H(x'_{2}, x_{3}) )$, which violates weak transitivity.
Essentially, due to invertibility, $\phi \circ H(x_{1}, x_{3}) = \left( \inner(x_{1}) - \inner'(x_{2}) \right) - \left( \inner(x_{3}) - \inner'(x_{2}) \right)$ must be invariant under choice of $x_{2}$. 
In the subsequent step, $\inner'$ is eliminated, which effectively illustrates that $\inner = \inner'$.
%
%Exists some $\inner, \inner'$. (BWOC)
%
Finally, let $\outer(u) = \phi(u)$, and %the result
item~\ref{thm:factoring:disc} is complete.

\todo{Talk about CDF properties!}

\todo{Old junk:
Half Identity: Transitive law $f$ implies
$H(x_{1}, x_{2}) = f(H(x_{1}, x_{2}), H(x_{2}, x_{2})) = f(H(x_{1}, x_{2}), \frac{1}{2}) = H(x_{1}, x_{2})$, thus 
$\phi(p) = $
\\
Cancellation: $H(x_{1}, x_{1}) = f(H(x_{1}, x_{2}), H(x_{2}, x_{1})) = \rho(\sigma(u) + \sigma(1 - u)) = \rho(0) = \frac{1}{2}$
\\
Limit Condition:
Suppose BWOC that $\rho(u) = 1$ for some finite $u$.
\\
Due to half-identity, we may conclude that $\rho(\phi(u) + \phi(\frac{1}{2})) = u$ for any $u \in [0, 1]$ that arises as $u = H(x_{1}, x_{2})$ for some $x_{1}, x_{2} \in \R$.
Thus for all valid $u$, $\phi(u)$ is unique?
PROBLEM: this is true, but $u' = \phi(u_{1}) + \phi(u_{2})$ can be an invalid $u$, necessitating $\rho(u') = \rho(u_{3})$ for some valid $u_{3} \neq u'$. So it's a one-way inverse: $\rho(\phi(u)) = u$, but $\phi(\rho(u)) \neq u$ is possible, i.e., $\rho(u) = \rho(v)$ does not imply $u = v$.
\\
Therefore, WLOG, assume $\phi(\frac{1}{2}) = 0$ and $\rho$ and $\phi$ are inverses of one another (not necessarily continuous or monotonic).
TODO They are monotonic
\\
Now, observe that $\rho(0) = \frac{1}{2}$, and by invertibility and cancellation, we may conclude that $\sigma(u) = -\sigma(1 - u)$.
}

\bigskip

We now show item~\ref{thm:factoring:cond} (noting again that item~\ref{thm:factoring:uncond} is a special case).

\todo{For the form from item 1: it's missing something? item 1 needs a rho(phi( transitivity law form?. Not necessarily invertible.}

We first show a technical lemma: we assume $x'''$ can be obtained from $x_{1}$ by changing two independent features not in the condition set $\omega$, each partial change resulting in $x_{2}'$ and $x_{2}''$ (as in the CIC axiom).
% seek
% \[
% H(x_{1}, x_{2}''') = \rho\left( \phi \circ H(x_{1}, x_{2}') + \phi \circ H(x_{1}, x_{2}'') \right)
% \]
Observe that\todo{Add detail!}\todo{Use $\rho = \sigma$ and $\phi = \sigma^{-1}$}
{\footnotesize
\begin{align*}
\null\hspace{-1cm}\sigma^{-1} \! \circ \! H(x_{1}, x_{2}''') \hspace{-2cm} & \null & \null \\
%&= \sigma^{-1} \! \circ \! & \\
&=
  \sigma^{-1} \! \circ \! H(x_{1}, x_{2}') + \sigma^{-1} \! \circ \! H(x_{2}', x_{2}''') +
  \sigma^{-1} \! \circ \! H(x_{1}, x_{2}'') + \sigma^{-1} \! \circ \! 
H(x_{2}'', x_{2}''') - \sigma^{-1} \! \circ \! H(x_{1}, x_{2}''') & \hspace{-0.5cm}\textsc{$\sigma$-Transitivity} \\
 &= \left( \sigma^{-1} \! \circ \! H(x_{1}, x_{2}') +  \sigma^{-1} \! \circ \! H(x_{1}, x_{2}'') \right) + \left( \sigma^{-1} \! \circ \! H(x_{2}', x_{2}''') + \sigma^{-1} \! \circ \! 
H(x_{2}'', x_{2}''') + \sigma^{-1} \! \circ \! H(x_{2}''', x_{1}) \right) & \textsc{Algebra} \\
 &= \left( \sigma^{-1} \! \circ \! H(x_{1}, x_{2}') +  \sigma^{-1} \! \circ \! H(x_{1}, x_{2}'') \right) + \null \\
  & \quad \mathsmaller{\frac{1}{2}}\! \left( \sigma^{-1} \! \circ \! H(x_{2}', x_{2}''') + \sigma^{-1} \! \circ \! 
H(x_{2}'', x_{2}''') + \sigma^{-1} \! \circ \! H(x_{2}', x_{1}) + \sigma^{-1} \! \circ \! H(x_{2}'', x_{1}) \right) & \textsc{Algebra} \\
 &= \mathsmaller{\frac{1}{2}} \left( \sigma^{-1} \! \circ \! H(x_{1}, x_{2}') +  \sigma^{-1} \! \circ \! H(x_{1}, x_{2}'') \right) + \mathsmaller{\frac{1}{2}} \left( \sigma^{-1} \! \circ \! H(x_{2}', x_{2}''') + \sigma^{-1} \! \circ \! 
H(x_{2}'', x_{2}''') \right) & \textsc{Algebra} \\
 &= \sigma^{-1} \! \circ \! H(x_{1}, x_{2}') +  \sigma^{-1} \! \circ \! H(x_{1}, x_{2}'') & \hspace{-1cm}\textsc{See Below} %\\
\end{align*}%
}%
The final step applies the assumed CIC axiom:
The only way this assumption can hold is if $\left(\sigma^{-1} \circ H(x_{1}, x_{2}') + \sigma^{-1} \circ H(x_{1}, x_{2}'') \right) = \left( \sigma^{-1} \circ H(x_{2}', x_{2}''') + \sigma^{-1} \circ H(x_{2}'', x_{2}''') \right)$.
This is because of symmetry, if we negate the equation, flipping $x_{1}$ and $x_{2}'''$, the same must hold, hence the conclusion of the technical lemma.

%Thus we obtain the desideratum.
%TODO: we just need $\outer$ to be an invertible function (one to one)???

We now chain the above result to derive the desideratum, i.e., item~\ref{thm:factoring:trans}.
Given $x'_{0}, x'_{1}, x'_{2}, x'_{3}, \dots, x'_{d}$, define the \emph{transitive chain operator} $f(\cdot, \cdot, \cdots, \cdot)$ as
\[
f(x'_{0}, x'_{1}, x'_{2}, x'_{3}, \dots, x'_{d}) = f( \cdots ( f(f(x'_{0}, x'_{1}), x'_{2}), \cdots ), x'_{d}) = \outer \left( \sum_{i=1}^{d} \inner(x'_{i-1}) - \inner(x'_{i}) \right)
\enspace,
\]
where the RHS applies the structure of $\sigma$-transitivity.
%TODO NB this requires something about an inverse relationship.

Now, take $x'^{i}_{\omega}$ to be $x^{i}_{2}$ if $i \in \omega$, $x^{i}_{1}$ otherwise. Define a sequence $x'_{j}$ such that each $x'_{j}$ takes the previous $x'_{j-1}$ (starting with $x'_{\omega}$ for $j=1$) and changes one feature $k \not \in \omega$ from $x_{1}^{k}$ to $x_{2}^{k}$, such that $x'_{d-\abs{\omega}} = x_{2}$.

We then chain the result over all items not in the condition set $\omega$ to obtain
\begin{align*}
H(x_{1}, x_{2}) &= \outer \left( \inner(x_{1}) - \inner(x_{2}) \right) & \textsc{By Assumption} \\
 &= f(x_{1}, x_{2,\omega}, x_{2,(1)}, x_{2,(2)}, \dots, x_{2,(d-\abs{\omega})}) & \textsc{Transitivity Chain} \\
 &= \outer \left( \inner^{\omega,x_{1}^\omega}(x_{1}) - \inner^{\omega,x_{2}^{\omega}}(x_{2}) + \sum_{i \not \in \omega} \inner^{i,\omega_{i}}(x_{1}^{i}) - \inner^{i,\omega_{i}}(x_{2}^{i}) \right) & \textsc{$\sigma$-Transitivity} \\
 &= \outer \left( \sum_{i \not \in \omega} \inner^{i,\omega_{i}}(x_{1}^{i}) - \inner^{i,\omega_{i}}(x_{2}^{i}) \right) & \textsc{See Below} \\
 &= \outer \left( \sum_{i=1}^{d} \inner^{i,\omega_{i}}(x_{1}^{i}) - \inner^{i,\omega_{i}}(x_{2}^{i}) \right) \enspace. & \hspace{-2cm} \textsc{Let } \inner^{i,\omega_{i}}(x) = 0 \text{ for all } i \in \omega
\end{align*}
%TODO don't need $\omega_{i}$
Note that the function $\inner$ may change from step to step above, the conclusion is only that such a decomposition exists.
To see the step marked \textsc{See Below}, observe that the terms that condition on $\omega$ are capable of additively representing any function over $\omega$, so the explicit first term, i.e., that involving $\inner^{\omega}$, may be omitted.
Finally, observe that for the special case of noninteractivity, we have $\omega=\emptyset$, thus this first term always cancels out, and the above telescoping decomposition consists of exactly $d$ terms, one per feature.
%
\end{proof}

\bigskip

We now show \cref{thm:models}.
\thmmodels*
\begin{proof}

We first show item~\ref{thm:models:lin}.
We then find that it is more straightforward to show item~\ref{thm:models:mmono}, and finally to conclude item~\ref{thm:models:umono} as a corollary.

\medskip

We first show item~\ref{thm:models:lin}.
Recall that the linearity assumption requires that there exists some $\wv$ such that for all $x_{1}, x_{2}$, it holds
\[
\sigma^{-1} \circ H(x_{1}, x_{2}) = \wv \cdot (x_{1} - x_{2} )
\enspace.
\]

Consequently, the set of all feasible decision rules is parameterized by $\wv$, in particular, applying $\sigma(\cdot)$ to both sides of the above, it holds
\[
\HC = \left\{ x_{1}, x_{2} \mapsto \sigma(\bm{w} \cdot (x_{1} - x_{2}) ) \,\middle|\, \bm{w} \in \R^{d} \right\}
\enspace.
\]
Finally, observe that this multivariate hypothesis class decomposes as
\[
\HC = \left\{ x_{1}, x_{2} \mapsto \sigma\left( \sum_{i=1}^{d} \inner^{i}(x_{1}^{i}) - \inner^{i}(x_{2}^{i}) \right) \,\middle|\, \inner^{i}(x) = \bm{w}_{i} x, \bm{w} \in \R^{d} \right\}
\enspace.
\]
From this, we conclude both the form of $\HC$ and of each $\HCinner^{(i)}$.

NB $\sigma$ of $\sigma$-linearity and $\sigma$-transitivity must be identical over the domain (up to isomorphism), as $\sigma$-linearity and $\sigma$-transitivity would otherwise be mutually incompatible.

\todo{Lipschitz / constrained version? Monotonic version?}
\todo{Kernelized version?}

\medskip

We now show item~\ref{thm:models:mmono}.

Recall that we assumed $\sigma$-transitivity (\cref{thm:factoring}~item~\ref{thm:factoring:trans}).
It thus holds that
\[
\HC = \left\{ x^{(1)}, x^{(2)} \mapsto
  %\outer
  \sigma\left ( \vphantom{\sum}\smash{\sum_{i=1}^{d}} \inner^{i}(x^{(i)}_{1}) - \inner^{i}(x^{(i)}_{2}) \right)
  %\outer\left( \sum_{i=1}^{d} h_{i}(x^{(1)}_{i}) - h_{i}(x^{(2)}_{i}) \right)
  \,\middle|\, \inner^{i} \in \HCinner^{(i)} \right\}
\enspace.
\]

Monotonicity requires that if $x_{1} \preceq x_{2}$, then
\[
H(x_{1}, x_{2}) \leq \frac{1}{2}
\enspace.
\]
In our factoring, this is equivalent to
\[
\inner(x_{1}) \leq \inner(x_{2})
\]
\todo{Why?}
Consequently, the set of all such functions satisfying monotonicity is
\[
\HC = \{ x_{1}, x_{2} \mapsto \sigma( h(x_{1}) - h(x_{2}) ) \,|\, h: \X \to \R \text{ s.t. } \vec{x} \preceq \vec{y} \implies h(\vec{x}) \leq h(\vec{y}) \}
\enspace.
\]
NB: this does not specify $\sigma$, but due to the assumed factoring, $\sigma$ must be a symmetric continuous CDF with full support.

\medskip

We now show item~\ref{thm:models:umono}.

As we now restore NC, we now begin with the stronger factored form
\[
\HC = \left\{ x^{(1)}, x^{(2)} \mapsto
  %\outer
  \sigma\left ( \vphantom{\sum}\smash{\sum_{i=1}^{d}} \inner^{i}(x^{(i)}_{1}) - \inner^{i}(x^{(i)}_{2}) \right)
  %\outer\left( \sum_{i=1}^{d} h_{i}(x^{(1)}_{i}) - h_{i}(x^{(2)}_{i}) \right)
  \,\middle|\, \inner^{i} \in \HCinner^{(i)} \right\}
\enspace.
\]
We then apply the same reasoning as in \cref{thm:models:mmono}, now applied to $x_{1}$, $x_{2}$ that differ in only dimension $i$, to conclude that the $i$th model factor obeys
\[
\HCinner^{(i)} = \{ h: \R \to \R \,|\, x \leq y \implies h(x) \leq h(y) \}
\enspace.
\]
This concludes the result.
\todo{This proof lacks detail.}
\end{proof}

\todo{commented notes on probabilistic rankings}

We now show \cref{thm:prop}.
\thmprop*
\begin{proof}
\todo{Clearer notatition.}
Suppose items $A$, $B$, and $C \in \X$.
Essentially, we use the telescoping product
\[
\frac{\Prob(A \text{ best})}{\Prob(C \text{ best})}
= \frac{\Prob(A \text{ best})}{\Prob(B \text{ best})} \frac{\Prob(B \text{ best})}{\Prob(C \text{ best})}
\]

Substituting in the proportionality axiom:

\[
\frac{1}{\sigma(\sigma^{-1}(\Prob(C > B)) + \sigma^{-1}(\Prob(B > A)))} - 1 = \left(\frac{1}{\Prob(B > A)} - 1\right)\left(\frac{1}{\Prob(C > A)} - 1\right)
\]

Rearranging the equation:
\begin{align*}
\sigma(\sigma^{-1}(\Prob(C > B)) + \sigma^{-1}(\Prob(B > A))) &= \frac{1}{1 + \left(\frac{1}{\Prob(B > A)} - 1\right)\left(\frac{1}{\Prob(C > A)} - 1\right)} \\
 &= \frac{1}{1 + \exp \circ \ln \left(\left(\frac{1}{\Prob(B > A)} - 1\right)\left(\frac{1}{\Prob(C > A)} - 1\right) \right)} \\
 &= \logistic\left(\ln \left(\frac{1}{\frac{1}{\Prob(B > A)} - 1}\right) + \ln \left(\frac{1}{\frac{1}{\Prob(C > A)} - 1}\right) \right) \\
 &= \logistic\left(\logit\left(\Prob(B > A)\right) + \logit \left(\Prob(C > A)\right) \right)
 \enspace.
\end{align*}
We thus conclude that $\sigma$ is the logistic function.
\end{proof}

\todo{Notes on other probabilisic rankings, Gaussian model, etc.}

%% file: empirical_appendix.tex
\section{Empirical Analysis -- Other Details} \label{sec:appendix_empirical}

\subsection{Dataset and Preprocessing}

\paragraph{Kidney Allocation Real-World Datasets.}
As mentioned earlier, the real-world dataset we used was collected by \citet{boerstler2024stability} to test moral judgments \textit{stability}.
Across two studies, multiple participants were presented with several kidney allocation scenarios across 10 days (60 scenarios per day). 
In Study One, each kidney allocation scenario contained profiles of two kidney patients (Patient A and Patient B) described by four features: (a) number of dependents (0, 1, 2), (b) life decades gained from kidney transplant (1, 2, 3), (c) number of alcoholic drinks per day (0, 2, 4), and (d) number of crimes committed (0, 1, 2).
In Study Two, each patient is described by five features: (a) number of elderly dependents (0, 1, 2), (b) life years gained from kidney transplant (1, 2, 3), (c) years waiting for the transplant (0, 2, 4), (d) weekly work hours post-transplant (0, 1, 2), and (e) obesity level (0, 1, 2, 3, 4).
Participants were asked to decide who should get the kidney when only one was available and both patients were eligible.
An example of one such pairwise comparison is presented in Figure~\ref{fig:example_ka}.

\begin{figure}
    \centering
    \includegraphics[width=0.5\linewidth]{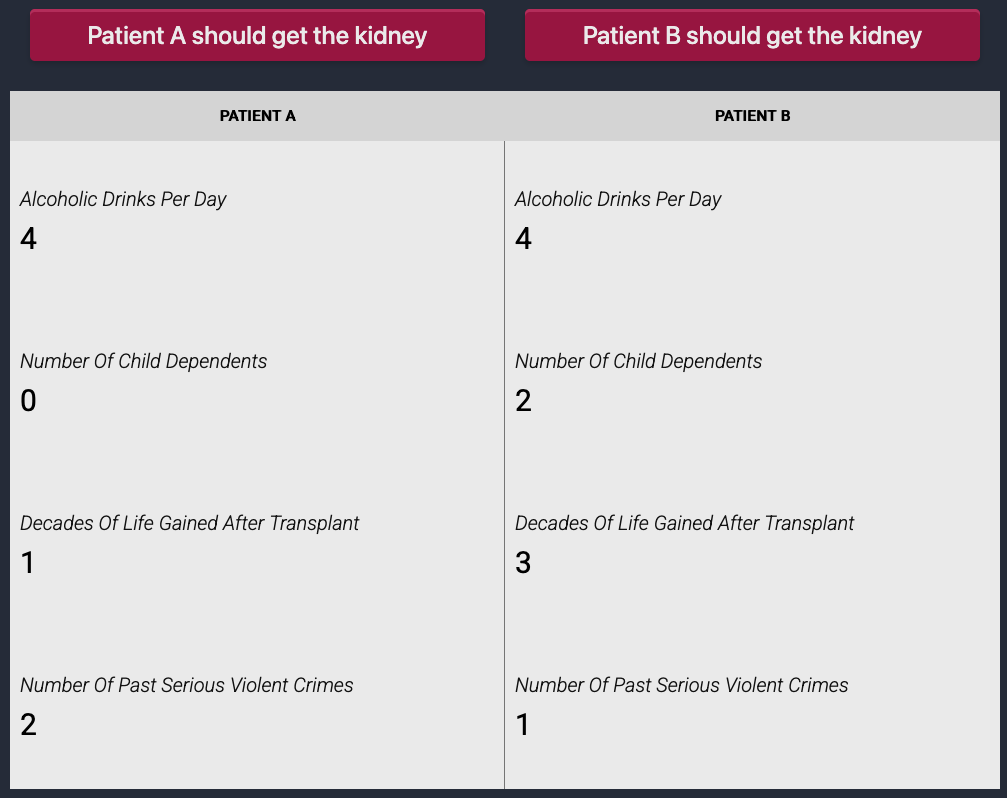}
    \caption{An example pairwise comparison presented to participants of Study One of the Kidney Allocation dataset.}
    \label{fig:example_ka}
\end{figure}

Out of the 60 scenarios presented to each participant per day, several scenarios were repeated.
This includes 6 scenarios that were repeated twice per day across all days to test for response stability and 2 scenarios that were used for attention checks.
The remaining scenarios were randomly chosen.
Since the repeated scenarios were non-random, we removed them from the dataset to ensure that the underlying data distribution for each participant corresponds to the uniform distribution over the input space.
Sessions where participants failed attention checks were also removed from the dataset. 
% 

% For Study Two, we also performed minor feature preprocessing. In particular, we pre-transformed the feature obesity level by taking its absolute value.
% % 
% This was done because the obesity value -1 denotes ``underweight'', 0 denotes ``normal weight'', and 1 denotes ``overweight''. 
% % 
% By taking the absolute value of this feature, the transformed feature denotes \textit{difference from ``normal weight''}, which is more likely to be monotonic and hence is more suitable for our experimental purposes.
% % 
% Note that this step is taken mainly for the convenience of empirical analysis and may not correspond to true human decision-making regarding this feature, which could introduce some errors in the learned performance. 

After the above steps, we had around 383 average responses from 15 participants in Study One and around 330 average responses from 40 participants in Study Two.
This dataset is available under the CC-BY 4.0 license and is included with the Supplementary for the sake of completeness.
% \textbf{License.......}
% 
% For this set of participants across two studies, the number of responses ranged from 285 to 627. 

\paragraph{Synthetic Datasets.} The synthetic dataset is constructed by simulating five decision-makers DM1--DM5, each using a different set of heuristics.
The process for DM1 is described in Section~\ref{sec:experiments}.
For DM2--DM5, the simulated decision processes are noted below.
Recall that each decision-maker is presented with a pairwise comparison between two kidney patients, A and B, with features $x_A$ and $x_B$. Each patient is described using the patient's number of dependents, life years gained from the transplant, years waiting for the transplant, and their past number of crimes.

\textit{DM2 simulation.} For DM2, the decision process for any pairwise comparison between A and B can be described as follows: (a) choose the patient with non-zero dependents; (b) if both patients have non-zero dependents, then choose the patient with greater value for life years gained; (c) if equal, then choose the patient who has been waiting longer for the transplant; (d) if equal, then choose randomly.
Hence, this decision-maker also uses thresholds on the number of dependents and employs other features mainly for tie-breaking.

\textit{DM3 simulation.} For DM3, the decision process for any pairwise comparison between A and B can be described as follows: (a) transform life years gained feature to reflect \textit{diminishing returns}, i.e. $z_A^{\text{life gained}} = \lfloor \log(x_A^{\text{life gained}})\rfloor$ and $z_B^{\text{life gained}} = \lfloor \log(x_B^{\text{life gained}})\rfloor$; (b) assign $z_A^{\text{life gained}}$ points to patient A and $z_B^{\text{life gained}}$ to patient B; (c) assign $x_A^{\text{dependents}}$ points to patient A and $x_B^{\text{dependents}}$ points to patient B; (d) assign one point each patient who was been waiting for 5 years or more; (e) sum up the points for each patient and choose one with greater number of points (ties broken randomly).
The log-transformation used in this decision-making process captures the diminishing returns property associated with the life years gained feature. Additionally, this process also uses the tallying heuristic to essentially count the number of factors favoring each patient.

\textit{DM4 simulation.} For DM4, the decision process for any pairwise comparison between A and B can be described as follows: (a) transform life years gained feature to reflect \textit{diminishing returns}, i.e. $z_A^{\text{life gained}} = \lfloor \log(x_A^{\text{life gained}})\rfloor$ and $z_B^{\text{life gained}} = \lfloor \log(x_B^{\text{life gained}})\rfloor$; (b) choose patient with greater $z_\cdot^{\text{life gained}}$ value; (c) if equal, then choose the patient with more dependents; (d) if equal, then choose the patient who has been waiting longer for the transplant; (e) if equal, then choose randomly.
This decision-maker also uses log-transformation for life years gained and other features for tie-breaking.

\textit{DM5 simulation.} For DM5, the decision process for any pairwise comparison between A and B can be described as follows: (a) count how many features favor each patient and choose the patient favored by more features; (b) if equal, choose randomly.
This decision-maker simply uses the tallying heuristic, choosing the option favored by more factors.

\begin{figure}
    \centering
    \includegraphics[width=\linewidth]{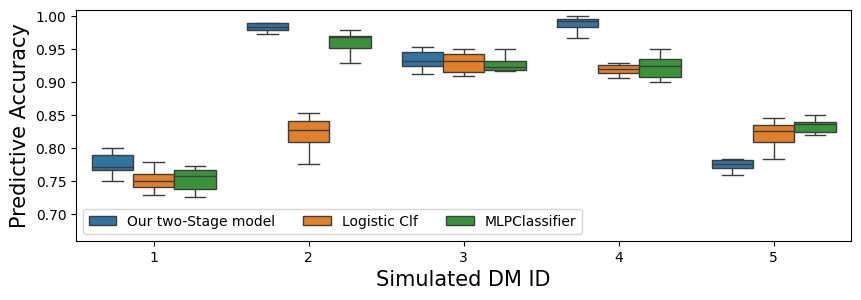}
    \caption{Comparison of our model vs baseline linear regression model for the synthetic dataset containing responses from five simulated decision-makers.}
    \label{fig:study_simulated_factor_vs_baselines}
\end{figure}

\begin{figure}
    \centering
    \includegraphics[width=\linewidth]{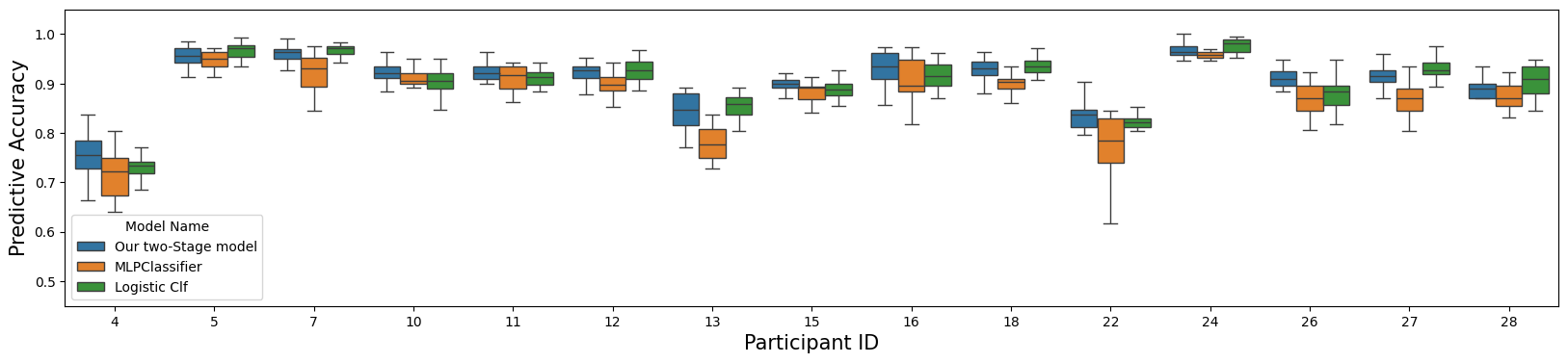}
    \caption{Comparison of our model vs baseline linear regression model for participants in Study One of the Kidney Allocation dataset.}
    \label{fig:study_one_factor_vs_baseline_tallying_up}
\end{figure}

\begin{figure}
    \centering
    \includegraphics[width=\linewidth]{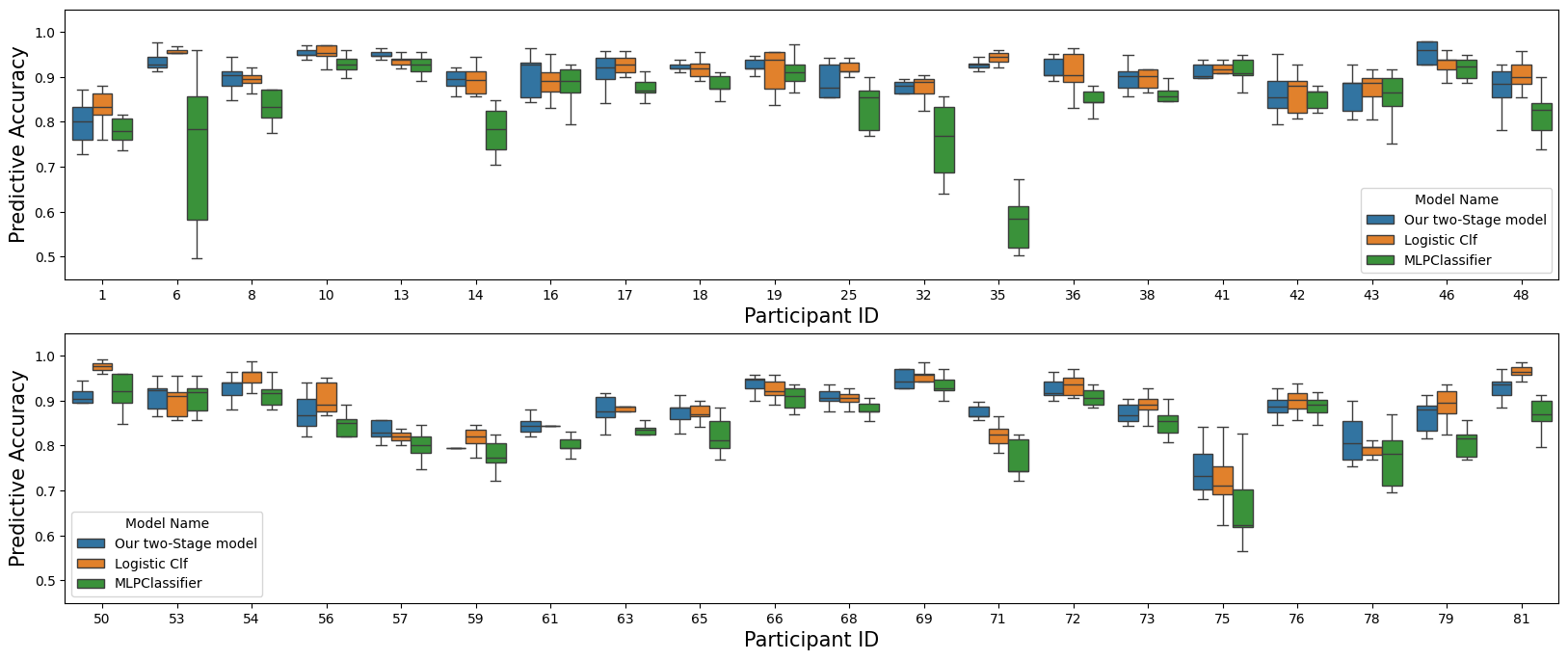}
    \caption{Comparison of our model vs baseline linear regression model for participants in Study Two of the Kidney Allocation dataset.}
    \label{fig:study_two_factor_vs_baseline_tallying_up}
\end{figure}

\begin{figure}
    \centering
    \includegraphics[width=\linewidth]{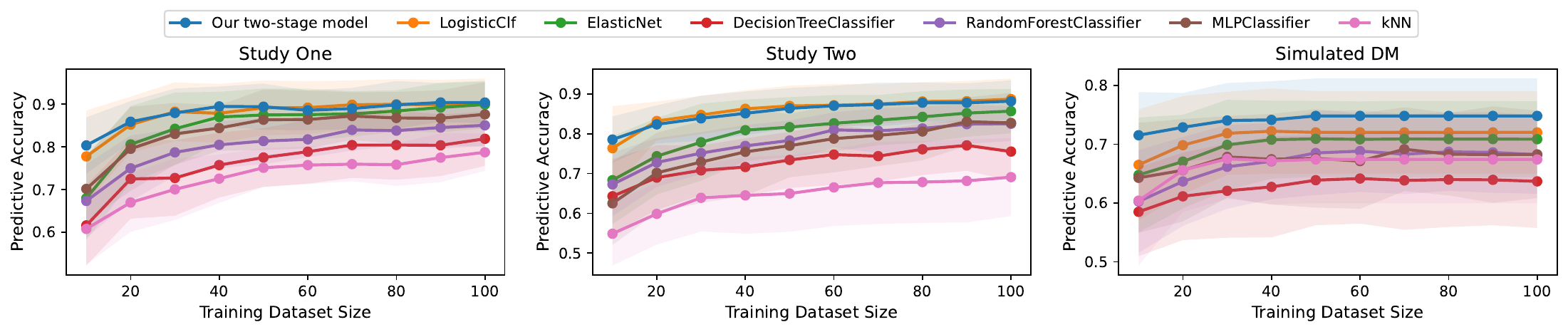}
    \caption{Comparison of our model vs baselines across increasing training size.}
    \label{fig:performance_by_training_size}
\end{figure}

\subsection{Implementation Details}
For each decision-maker (real or simulated), we use a random 70-30 train-test split and report the summary statistics of predictive accuracy over test partitions across 20 repetitions.

\paragraph{Our learning framework.}
For our framework, we learn editing functions $\inner^{i, \cdot}$ for feature $i$ that are designed to assign a score$\in \R$ to each value of feature $i$.
As noted earlier, all $\inner^{\cdot, \cdot}$ are constrained to be monotonic.
Context $\omega$ is limited to one feature for the real-world dataset.
The specific context is chosen using cross-validation. 20\% of the training data is held out to use for validation purposes and the conditional two-stage model is learned for each setting of $\omega = \set{i}$ for all $i\in[d]$.
The context $\omega=\set{i^\star}$ for which we achieve the smallest predictive error over the validation set is chosen as the final context for the corresponding participant.
Context $\omega$ is kept empty for the synthetic dataset.
For the dominance testing function $\outer(\cdot)$, we implement the simple tallying heuristic (i.e., the outputs from each feature-level editing function $\inner^{\cdot, \cdot}$ are simply averaged).
The final binary prediction is basically whether the $\outer(\cdot)$ is greater than 0 or not (with positive value favoring the patient on the left and negative value favoring the one on the right). 

The two-stage model is trained by minimizing the chosen loss function (cross-entropy or hinge) with regularization and constraints.
% 
% The loss function contains the standard Hinge loss plus 
The regularization term quantifies the difference between editing functions learned for different values of the context variable $\omega$.
We use this regularizer to ensure that editing functions for any feature $i$ corresponding to different context values are not too far from each other.
As we noted earlier, $\omega$ is set to contain only one feature in our experiments on real-world datasets. For any two feature values $x^\omega = a$ and $x^\omega = b$, and for any other feature $i \notin \omega$, the regularization term measures $||\inner^{i, a} - \inner^{i, b}||$, for all $a,b$. The difference between the two editing functions is calculated by taking the squared norm of the difference between the outputs of these functions over all values in $\X_i$.
Hence, the overall regularization term we use is 
\[\lambda \cdot \sum_{i\notin\omega}\sum_{a,b} ||\inner^{i, a} - \inner^{i, b}||.\]
Here $\lambda$ is the regularization parameter and is set to be 1e-3 in our experiments.
Additionally, we impose monotonicity constraints on each of $\inner$ functions while optimizing the above loss; i.e., for all features $i$, $\inner^i(a) \geq \inner^i(b), \forall a > b \in \X_i$ OR  $\inner^i(a) \leq \inner^i(b), \forall a > b \in \X_i$ (same constraints for context-based inner functions as well).
We solve this constrained optimization problem using the Python SLSQP library (options 'ftol' and 'maxiter' are set to 1e-7 and 300, respectively).

% This will form our dataset $S_k$ for the $k-$th participant.
% 

\paragraph{Baseline details.}
The drift diffusion model was implemented using the Python PyDDM library with linear drift functions.
The Bradley-Terry framework was implemented using the Python choix library, with a two-layer neural network scoring function.

We also implemented the following supervised modeling strategies to compare our method against. 
(a) Logistic Classifier -- we implement the standard logistic classification approach, but impose a symmetry constraint by regressing over feature differences across the pairwise comparison; 
(b) Elastic-net Classifier -- using the standard logistic classification approach over all given features with L1 and L2 regularization, with L1 ratio (scaling between L1 and L2 penalties) set to be 0.5.
(c) Decision Tree Classifier -- over all features with Gini splitting criterion;
(d) Linear SVM -- over all features with L2 penalty and squared hinge loss; 
(e) kNN -- over all features with $n=5$ neighbors;
(f) Random Forest Classifier -- over all features with 100 estimators;
(g) MLP Classifier -- over all features with two hidden layers of 10 nodes each.

\paragraph{Computing resources used.} All experiments were run on a MacBook M2 system with a 16GB of random-access memory.

\subsection{Additional Empirical Results}

\paragraph{Performances and models for simulated decision-makers DM2--DM5}.
As mentioned earlier, we also created a synthetic dataset containing responses from five simulated decision-makers.
The descriptions of the simulated DMs are provided above.
In this section, we report additional results for these simulations.

First, the performance of our model, Logistic Classifier, and MLP Classifier are presented in Figure~\ref{fig:study_simulated_factor_vs_baselines}.
Note for all but DM5, our model has better predictive accuracy than baselines.

\paragraph{Individual participant-level performances.}
Across all participants, the performance of our two-stage model, Logistic Classifier, and MLP Classifier are reported in Figure~\ref{fig:study_one_factor_vs_baseline_tallying_up} for Study One and in Figure~\ref{fig:study_two_factor_vs_baseline_tallying_up} for Study Two (other baselines excluded here for presentation clarity).
The plots show that our model has comparable accuracy to the Logistic Classifier for all participants across the real-world study datasets.

\begin{figure}
    \centering
    \includegraphics[width=\linewidth]{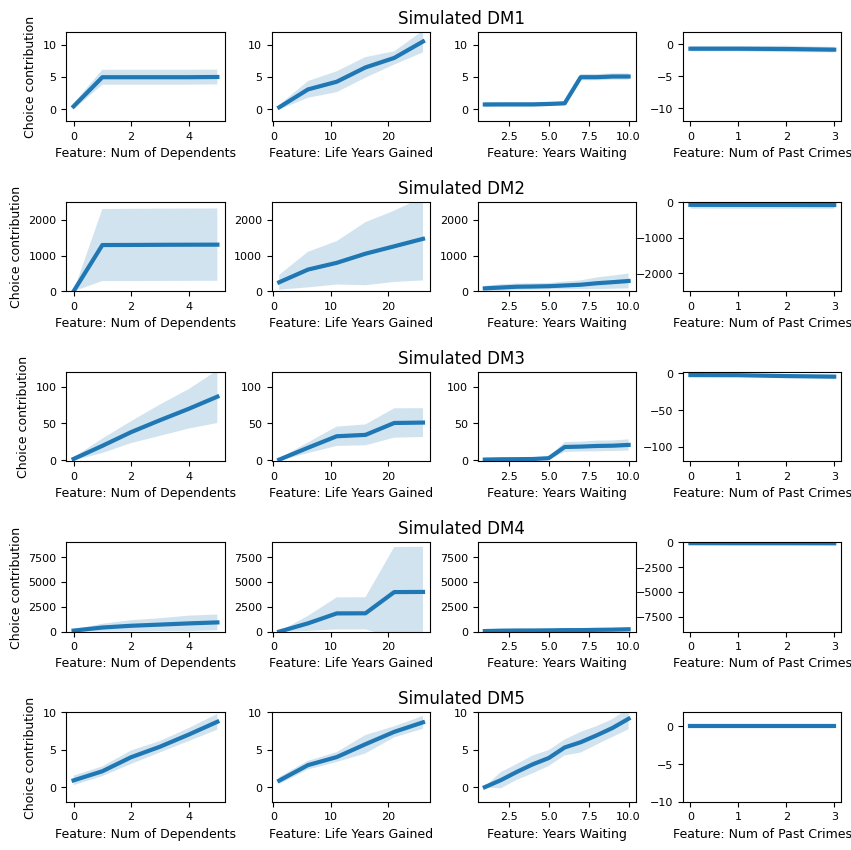}
    \caption{Editing rules learned by our approach for all simulated decision-makers DM1--DM5.}
    \label{fig:editing_rules_simulated}
\end{figure}

\paragraph{Performance variation across training data sizes.}
We also assess the performance of each model across variations of training size from 10 to 100.
The results are reported in Figure~\ref{fig:performance_by_training_size}.
Our model reaches high accuracy levels faster than other models, with significant improvement in training performance compared to baselines for the simulated decision-makers.

%% file: references.bib
@incollection{montgomery1983decision,
	title        = {Decision rules and the search for a dominance structure: Towards a process model of decision making},
	author       = {Montgomery, Henry},
	year         = 1983,
	booktitle    = {Advances in psychology},
	publisher    = {Elsevier},
	volume       = 14,
	pages        = {343--369}
}

@incollection{kahneman2013prospect,
	title        = {Prospect theory: An analysis of decision under risk},
	author       = {Kahneman, Daniel and Tversky, Amos},
	year         = 2013,
	booktitle    = {Handbook of the fundamentals of financial decision making: Part I},
	publisher    = {World Scientific},
	pages        = {99--127}
}

@article{payne1976task,
	title        = {Task complexity and contingent processing in decision making: An information search and protocol analysis},
	author       = {Payne, John W},
	year         = 1976,
	journal      = {Organizational behavior and human performance},
	publisher    = {Elsevier},
	volume       = 16,
	number       = 2,
	pages        = {366--387}
}

@article{ajzen1996social,
	title        = {The social psychology of decision making},
	author       = {Ajzen, Icek},
	year         = 1996,
	journal      = {Social psychology: Handbook of basic principles},
	pages        = {297--325}
}

@article{kubanek2017optimal,
	title        = {Optimal decision making and matching are tied through diminishing returns},
	author       = {Kubanek, Jan},
	year         = 2017,
	journal      = {Proceedings of the National Academy of Sciences},
	publisher    = {National Academy of Sciences},
	volume       = 114,
	number       = 32,
	pages        = {8499--8504}
}

@article{shah2008heuristics,
	title        = {Heuristics made easy: an effort-reduction framework.},
	author       = {Shah, Anuj K and Oppenheimer, Daniel M},
	year         = 2008,
	journal      = {Psychological bulletin},
	publisher    = {American Psychological Association},
	volume       = 134,
	number       = 2,
	pages        = 207
}

@book{gigerenzer2000simple,
	title        = {Simple heuristics that make us smart},
	author       = {Gigerenzer, Gerd and Todd, Peter M and ABC Research Group, the and others},
	year         = 2000,
	publisher    = {Oxford University Press}
}

@article{tversky1972elimination,
	title        = {Elimination by aspects: A theory of choice.},
	author       = {Tversky, Amos},
	year         = 1972,
	journal      = {Psychological review},
	publisher    = {American Psychological Association},
	volume       = 79,
	number       = 4,
	pages        = 281
}

@incollection{czerlinski1999good,
	title        = {How good are simple heuristics?},
	author       = {Czerlinski, Jean and Gigerenzer, Gerd and Goldstein, Daniel G},
	year         = 1999,
	booktitle    = {Simple heuristics that make us smart},
	publisher    = {Oxford University Press},
	pages        = {97--118}
}

@article{gigerenzer2011heuristic,
	title        = {Heuristic decision making},
	author       = {Gigerenzer, Gerd and Gaissmaier, Wolfgang},
	year         = 2011,
	journal      = {Annual review of psychology},
	publisher    = {Annual Reviews},
	volume       = 62,
	number       = 2011,
	pages        = {451--482}
}

@article{persson2022prominence,
	title        = {The prominence effect in health-care priority setting},
	author       = {Persson, Emil and Erlandsson, Arvid and Slovic, Paul and V{\"a}stfj{\"a}ll, Daniel and Tingh{\"o}g, Gustav},
	year         = 2022,
	journal      = {Judgment and Decision Making},
	volume       = 17,
	number       = 6,
	pages        = {1379--1391}
}

@article{tversky1988contingent,
	title        = {Contingent weighting in judgment and choice.},
	author       = {Tversky, Amos and Sattath, Shmuel and Slovic, Paul},
	year         = 1988,
	journal      = {Psychological review},
	publisher    = {American Psychological Association},
	volume       = 95,
	number       = 3,
	pages        = 371
}

@article{bradley1952rank,
	title        = {Rank analysis of incomplete block designs: I. The method of paired comparisons},
	author       = {Bradley, Ralph Allan and Terry, Milton E},
	year         = 1952,
	journal      = {Biometrika},
	publisher    = {JSTOR},
	volume       = 39,
	number       = {3/4},
	pages        = {324--345}
}

@article{chen2004survey,
	title        = {Survey of preference elicitation methods},
	author       = {Chen, Li and Pu, Pearl},
	year         = 2004
}

@inproceedings{kalloori2018eliciting,
	title        = {Eliciting pairwise preferences in recommender systems},
	author       = {Kalloori, Saikishore and Ricci, Francesco and Gennari, Rosella},
	year         = 2018,
	booktitle    = {Proceedings of the 12th acm conference on recommender systems},
	pages        = {329--337}
}

@inproceedings{guo2010real,
	title        = {Real-time multiattribute Bayesian preference elicitation with pairwise comparison queries},
	author       = {Guo, Shengbo and Sanner, Scott},
	year         = 2010,
	booktitle    = {Proceedings of the Thirteenth International Conference on Artificial Intelligence and Statistics},
	pages        = {289--296},
	organization = {JMLR Workshop and Conference Proceedings}
}

@article{rafailov2023direct,
	title        = {Direct preference optimization: Your language model is secretly a reward model},
	author       = {Rafailov, Rafael and Sharma, Archit and Mitchell, Eric and Manning, Christopher D and Ermon, Stefano and Finn, Chelsea},
	year         = 2023,
	journal      = {Advances in Neural Information Processing Systems},
	volume       = 36,
	pages        = {53728--53741}
}

@article{jiang2024survey,
	title        = {A survey on human preference learning for large language models},
	author       = {Jiang, Ruili and Chen, Kehai and Bai, Xuefeng and He, Zhixuan and Li, Juntao and Yang, Muyun and Zhao, Tiejun and Nie, Liqiang and Zhang, Min},
	year         = 2024,
	journal      = {arXiv preprint arXiv:2406.11191}
}

@article{ziegler2019fine,
	title        = {Fine-tuning language models from human preferences},
	author       = {Ziegler, Daniel M and Stiennon, Nisan and Wu, Jeffrey and Brown, Tom B and Radford, Alec and Amodei, Dario and Christiano, Paul and Irving, Geoffrey},
	year         = 2019,
	journal      = {arXiv preprint arXiv:1909.08593}
}

@article{ouyang2022training,
	title        = {Training language models to follow instructions with human feedback},
	author       = {Ouyang, Long and Wu, Jeffrey and Jiang, Xu and Almeida, Diogo and Wainwright, Carroll and Mishkin, Pamela and Zhang, Chong and Agarwal, Sandhini and Slama, Katarina and Ray, Alex and others},
	year         = 2022,
	journal      = {Advances in neural information processing systems},
	volume       = 35,
	pages        = {27730--27744}
}

@article{lichtenstein2006construction,
	title        = {The construction of preference: An overview},
	author       = {Lichtenstein, Sarah and Slovic, Paul},
	year         = 2006,
	journal      = {The construction of preference},
	volume       = 1,
	pages        = {1--40}
}

@article{ben2019foundations,
	title        = {Foundations of stated preference elicitation: Consumer behavior and choice-based conjoint analysis},
	author       = {Ben-Akiva, Moshe and McFadden, Daniel and Train, Kenneth and others},
	year         = 2019,
	journal      = {Foundations and Trends{\textregistered} in Econometrics},
	publisher    = {Now Publishers, Inc.},
	volume       = 10,
	number       = {1-2},
	pages        = {1--144}
}

@article{charness2013experimental,
	title        = {Experimental methods: Eliciting risk preferences},
	author       = {Charness, Gary and Gneezy, Uri and Imas, Alex},
	year         = 2013,
	journal      = {Journal of economic behavior \& organization},
	publisher    = {Elsevier},
	volume       = 87,
	pages        = {43--51}
}

@article{ji2023ai,
	title        = {Ai alignment: A comprehensive survey},
	author       = {Ji, Jiaming and Qiu, Tianyi and Chen, Boyuan and Zhang, Borong and Lou, Hantao and Wang, Kaile and Duan, Yawen and He, Zhonghao and Zhou, Jiayi and Zhang, Zhaowei and others},
	year         = 2023,
	journal      = {arXiv preprint arXiv:2310.19852},
	eprint       = {2310.19852},
	archiveprefix = {arXiv},
	primaryclass = {cs.AI}
}

@inproceedings{feffer2023preference,
	title        = {From preference elicitation to participatory ML: A critical survey \& guidelines for future research},
	author       = {Feffer, Michael and Skirpan, Michael and Lipton, Zachary and Heidari, Hoda},
	year         = 2023,
	booktitle    = {Proceedings of the 2023 AAAI/ACM Conference on AI, Ethics, and Society},
	pages        = {38--48}
}

@article{mukherjee2024optimal,
	title        = {Optimal design for human preference elicitation},
	author       = {Mukherjee, Subhojyoti and Lalitha, Anusha and Kalantari, Kousha and Deshmukh, Aniket Anand and Liu, Ge and Ma, Yifei and Kveton, Branislav},
	year         = 2024,
	journal      = {Advances in Neural Information Processing Systems},
	volume       = 37,
	pages        = {90132--90159}
}

@inproceedings{keswani2025can,
  author       = {Vijay Keswani and
                  Vincent Conitzer and
                  Walter Sinnott{-}Armstrong and
                  Breanna K. Nguyen and
                  Hoda Heidari and
                  Jana Schaich Borg},
  title        = {Can {AI} Model the Complexities of Human Moral Decision-making? {A}
                  Qualitative Study of Kidney Allocation Decisions},
  booktitle    = {Proceedings of the 2025 {CHI} Conference on Human Factors in Computing
                  Systems},
  pages        = {249:1--249:17},
  publisher    = {{ACM}},
  year         = {2025},
  url          = {https://doi.org/10.1145/3706598.3714167},
  doi          = {10.1145/3706598.3714167},
  bibsource    = {dblp computer science bibliography, https://dblp.org}
}

@inproceedings{lima2021human,
	title        = {Human perceptions on moral responsibility of {AI}: A case study in {AI}-assisted bail decision-making},
	author       = {Lima, Gabriel and Grgi{\'c}-Hla{\v{c}}a, Nina and Cha, Meeyoung},
	year         = 2021,
	booktitle    = {Proceedings of the 2021 CHI conference on human factors in computing systems},
	pages        = {1--17}
}

@article{fintz2022using,
	title        = {Using deep learning to predict human decisions and using cognitive models to explain deep learning models},
	author       = {Fintz, Matan and Osadchy, Margarita and Hertz, Uri},
	year         = 2022,
	journal      = {Scientific reports},
	publisher    = {Nature Publishing Group UK London},
	volume       = 12,
	number       = 1,
	pages        = 4736
}

@article{lin2022predicting,
	title        = {Predicting human decision making in psychological tasks with recurrent neural networks},
	author       = {Lin, Baihan and Bouneffouf, Djallel and Cecchi, Guillermo},
	year         = 2022,
	journal      = {PloS one},
	publisher    = {Public Library of Science San Francisco, CA USA},
	volume       = 17,
	number       = 5,
	pages        = {e0267907}
}

@article{erev2010choice,
	title        = {A choice prediction competition: Choices from experience and from description},
	author       = {Erev, Ido and Ert, Eyal and Roth, Alvin E and Haruvy, Ernan and Herzog, Stefan M and Hau, Robin and Hertwig, Ralph and Stewart, Terrence and West, Robert and Lebiere, Christian},
	year         = 2010,
	journal      = {Journal of Behavioral Decision Making},
	publisher    = {Wiley Online Library},
	volume       = 23,
	number       = 1,
	pages        = {15--47}
}

@inproceedings{plonsky2017psychological,
	title        = {Psychological forest: Predicting human behavior},
	author       = {Plonsky, Ori and Erev, Ido and Hazan, Tamir and Tennenholtz, Moshe},
	year         = 2017,
	booktitle    = {Proceedings of the AAAI Conference on Artificial Intelligence},
	volume       = 31,
	number       = 1
}

@article{peterson2021using,
	title        = {Using large-scale experiments and machine learning to discover theories of human decision-making},
	author       = {Peterson, Joshua C and Bourgin, David D and Agrawal, Mayank and Reichman, Daniel and Griffiths, Thomas L},
	year         = 2021,
	journal      = {Science},
	publisher    = {American Association for the Advancement of Science},
	volume       = 372,
	number       = 6547,
	pages        = {1209--1214}
}

@inproceedings{bourgin2019cognitive,
	title        = {Cognitive model priors for predicting human decisions},
	author       = {Bourgin, David D and Peterson, Joshua C and Reichman, Daniel and Russell, Stuart J and Griffiths, Thomas L},
	year         = 2019,
	booktitle    = {International conference on machine learning},
	pages        = {5133--5141},
	organization = {PMLR}
}

@incollection{crockett2016computational,
	title        = {Computational modeling of moral decisions},
	author       = {Crockett, Molly J},
	year         = 2016,
	booktitle    = {The social psychology of morality},
	publisher    = {Routledge},
	pages        = {71--90}
}

@article{freedman2020adapting,
	title        = {Adapting a kidney exchange algorithm to align with human values},
	author       = {Freedman, Rachel and Borg, Jana Schaich and Sinnott-Armstrong, Walter and Dickerson, John P and Conitzer, Vincent},
	year         = 2020,
	journal      = {Artificial Intelligence},
	publisher    = {Elsevier},
	volume       = 283,
	pages        = 103261
}

@inproceedings{boerstler2024stability,
	title        = {On the stability of moral preferences: A problem with computational elicitation methods},
	author       = {Boerstler, Kyle and Keswani, Vijay and Chan, Lok and Borg, Jana Schaich and Conitzer, Vincent and Heidari, Hoda and Sinnott-Armstrong, Walter},
	year         = 2024,
	booktitle    = {Proceedings of the AAAI/ACM Conference on AI, Ethics, and Society},
	volume       = 7,
	pages        = {156--167}
}

@article{sinnott2021ai,
	title        = {How {AI} can AID bioethics},
	author       = {Sinnott-Armstrong, Walter and Skorburg, Joshua August and others},
	year         = 2021,
	journal      = {Journal of Practical Ethics},
	publisher    = {Michigan Publishing},
	volume       = 9,
	number       = 1
}

@article{lee2019webuildai,
	title        = {{WeBuildAI}: Participatory framework for algorithmic governance},
	author       = {Lee, Min Kyung and Kusbit, Daniel and Kahng, Anson and Kim, Ji Tae and Yuan, Xinran and Chan, Allissa and See, Daniel and Noothigattu, Ritesh and Lee, Siheon and Psomas, Alexandros and others},
	year         = 2019,
	journal      = {Proceedings of the ACM on human-computer interaction},
	publisher    = {ACM New York, NY, USA},
	volume       = 3,
	number       = {CSCW},
	pages        = {1--35}
}

@inproceedings{johnston2023deploying,
	title        = {Deploying a Robust Active Preference Elicitation Algorithm on MTurk: Experiment Design, Interface, and Evaluation for COVID-19 Patient Prioritization},
	author       = {Johnston, Caroline M and Vossler, Patrick and Blessenohl, Simon and Vayanos, Phebe},
	year         = 2023,
	booktitle    = {Proceedings of the 3rd ACM Conference on Equity and Access in Algorithms, Mechanisms, and Optimization},
	pages        = {1--10}
}

@article{awad2018moral,
	title        = {The moral machine experiment},
	author       = {Awad, Edmond and Dsouza, Sohan and Kim, Richard and Schulz, Jonathan and Henrich, Joseph and Shariff, Azim and Bonnefon, Jean-Fran{\c{c}}ois and Rahwan, Iyad},
	year         = 2018,
	journal      = {Nature},
	publisher    = {Nature Publishing Group},
	volume       = 563,
	number       = 7729,
	pages        = {59--64}
}

@article{Liscio2024,
	title        = {Value Preferences Estimation and Disambiguation in Hybrid Participatory Systems},
	author       = {Enrico Liscio and Luciano Cavalcante Siebert and Catholijn M. Jonker and Pradeep K. Murukannaiah},
	year         = 2024,
	journal      = {CoRR},
	volume       = {abs/2402.16751},
	doi          = {10.48550/ARXIV.2402.16751},
	url          = {https://doi.org/10.48550/arXiv.2402.16751},
	eprinttype   = {arXiv},
	eprint       = {2402.16751},
	bibsource    = {dblp computer science bibliography, https://dblp.org}
}

@inproceedings{noothigattu2018voting,
	title        = {A voting-based system for ethical decision making},
	author       = {Noothigattu, Ritesh and Gaikwad, Snehalkumar and Awad, Edmond and Dsouza, Sohan and Rahwan, Iyad and Ravikumar, Pradeep and Procaccia, Ariel},
	year         = 2018,
	booktitle    = {Proceedings of the AAAI Conference on Artificial Intelligence},
	volume       = 32,
	number       = 1
}

@article{noothigattu2020axioms,
	title        = {Axioms for learning from pairwise comparisons},
	author       = {Noothigattu, Ritesh and Peters, Dominik and Procaccia, Ariel D},
	year         = 2020,
	journal      = {Advances in Neural Information Processing Systems},
	volume       = 33,
	pages        = {17745--17754}
}

@article{ge2024axioms,
	title        = {Axioms for {AI} alignment from human feedback},
	author       = {Ge, Luise and Halpern, Daniel and Micha, Evi and Procaccia, Ariel D and Shapira, Itai and Vorobeychik, Yevgeniy and Wu, Junlin},
	year         = 2024,
	journal      = {arXiv preprint arXiv:2405.14758}
}

@article{tversky1974judgment,
	title        = {Judgment under Uncertainty: Heuristics and Biases: Biases in judgments reveal some heuristics of thinking under uncertainty.},
	author       = {Tversky, Amos and Kahneman, Daniel},
	year         = 1974,
	journal      = {Science},
	publisher    = {American association for the advancement of science},
	volume       = 185,
	number       = 4157,
	pages        = {1124--1131}
}

@article{tversky1981framing,
	title        = {The framing of decisions and the psychology of choice},
	author       = {Tversky, Amos and Kahneman, Daniel},
	year         = 1981,
	journal      = {science},
	publisher    = {American Association for the Advancement of Science},
	volume       = 211,
	number       = 4481,
	pages        = {453--458}
}

@article{chan2024should,
	title        = {Should Responsibility Affect Who Gets a Kidney?},
	author       = {Chan, Lok and Sinnott-Armstrong, Walter and Borg, Jana Schaich and Conitzer, Vincent},
	year         = 2024,
	journal      = {Responsibility and Healthcare},
	publisher    = {Oxford University Press},
	pages        = 35
}

@inproceedings{kim2018computational,
	title        = {A computational model of commonsense moral decision making},
	author       = {Kim, Richard and Kleiman-Weiner, Max and Abeliuk, Andr{\'e}s and Awad, Edmond and Dsouza, Sohan and Tenenbaum, Joshua B and Rahwan, Iyad},
	year         = 2018,
	booktitle    = {Proceedings of the 2018 AAAI/ACM Conference on AI, Ethics, and Society},
	pages        = {197--203}
}

@article{noothigattu2019teaching,
	title        = {Teaching {AI} agents ethical values using reinforcement learning and policy orchestration},
	author       = {Noothigattu, Ritesh and Bouneffouf, Djallel and Mattei, Nicholas and Chandra, Rachita and Madan, Piyush and Varshney, Kush R and Campbell, Murray and Singh, Moninder and Rossi, Francesca},
	year         = 2019,
	journal      = {IBM Journal of Research and Development},
	publisher    = {IBM},
	volume       = 63,
	number       = {4/5},
	pages        = {2--1}
}

@inproceedings{capel2023human,
	title        = {What is human-centered about human-centered {AI}? {A} map of the research landscape},
	author       = {Capel, Tara and Brereton, Margot},
	year         = 2023,
	booktitle    = {Proceedings of the 2023 CHI conference on human factors in computing systems},
	pages        = {1--23}
}

@inproceedings{brighton2006robust,
	title        = {Robust inference with simple cognitive models.},
	author       = {Brighton, Henry},
	year         = 2006,
	booktitle    = {AAAI spring symposium: Between a rock and a hard place: Cognitive science principles meet AI-hard problems},
	pages        = {17--22}
}

@article{holte1993very,
	title        = {Very simple classification rules perform well on most commonly used datasets},
	author       = {Holte, Robert C},
	year         = 1993,
	journal      = {Machine learning},
	publisher    = {Springer},
	volume       = 11,
	pages        = {63--90}
}

@article{csimcsek2015learning,
	title        = {Learning from small samples: An analysis of simple decision heuristics},
	author       = {{\c{S}}im{\c{s}}ek, {\"O}zg{\"u}r and Buckmann, Marcus},
	year         = 2015,
	journal      = {Advances in neural information processing systems},
	volume       = 28
}

@article{brandstatter2006priority,
	title        = {The priority heuristic: making choices without trade-offs.},
	author       = {Brandst{\"a}tter, Eduard and Gigerenzer, Gerd and Hertwig, Ralph},
	year         = 2006,
	journal      = {Psychological review},
	publisher    = {American Psychological Association},
	volume       = 113,
	number       = 2,
	pages        = 409
}

@article{kahneman2005model,
	title        = {A model of heuristic judgment},
	author       = {Kahneman, Daniel and Frederick, Shane},
	year         = 2005,
	journal      = {The Cambridge handbook of thinking and reasoning},
	volume       = 267,
	pages        = 293
}

@book{mccullagh1989generalized,
  title     = {Generalized Linear Models},
  author    = {McCullagh, Peter and Nelder, John A.},
  year      = {1989},
  edition   = {2nd},
  publisher = {Chapman and Hall/CRC},
  address   = {London},
  series    = {Monographs on Statistics and Applied Probability},
  volume    = {37}
}

@inproceedings{ng2000algorithms,
	title        = {Algorithms for inverse reinforcement learning.},
	author       = {Ng, Andrew Y and Russell, Stuart and others},
	year         = 2000,
	booktitle    = {Icml},
	volume       = 1,
	number       = 2,
	pages        = 2
}

@article{christiano2017deep,
	title        = {Deep reinforcement learning from human preferences},
	author       = {Christiano, Paul F and Leike, Jan and Brown, Tom and Martic, Miljan and Legg, Shane and Amodei, Dario},
	year         = 2017,
	journal      = {Advances in neural information processing systems},
	volume       = 30
}

@article{kaufmann2023survey,
	title        = {A survey of reinforcement learning from human feedback},
	author       = {Kaufmann, Timo and Weng, Paul and Bengs, Viktor and H{\"u}llermeier, Eyke},
	year         = 2023,
	journal      = {arXiv preprint arXiv:2312.14925},
	volume       = 10
}

@article{stray2021you,
	title        = {What are you optimizing for? {Aligning} recommender systems with human values},
	author       = {Stray, Jonathan and Vendrov, Ivan and Nixon, Jeremy and Adler, Steven and Hadfield-Menell, Dylan},
	year         = 2021,
	journal      = {arXiv preprint arXiv:2107.10939}
}

@article{leike2018scalable,
	title        = {Scalable agent alignment via reward modeling: A research direction},
	author       = {Leike, Jan and Krueger, David and Everitt, Tom and Martic, Miljan and Maini, Vishal and Legg, Shane},
	year         = 2018,
	journal      = {arXiv preprint arXiv:1811.07871}
}

@inproceedings{pan2022effects,
	title        = {The Effects of Reward Misspecification: Mapping and Mitigating Misaligned Models},
	author       = {Pan, Alexander and Bhatia, Kush and Steinhardt, Jacob},
	year         = 2022,
	booktitle    = {International Conference on Learning Representations}
}

@article{hubinger2019risks,
	title        = {Risks from learned optimization in advanced machine learning systems},
	author       = {Hubinger, Evan and van Merwijk, Chris and Mikulik, Vladimir and Skalse, Joar and Garrabrant, Scott},
	year         = 2019,
	journal      = {arXiv preprint arXiv:1906.01820}
}

@article{zhuang2020consequences,
	title        = {Consequences of misaligned {AI}},
	author       = {Zhuang, Simon and Hadfield-Menell, Dylan},
	year         = 2020,
	journal      = {Advances in Neural Information Processing Systems},
	volume       = 33,
	pages        = {15763--15773}
}

@inproceedings{jacovi2021formalizing,
	title        = {Formalizing trust in artificial intelligence: prerequisites, causes and goals of human trust in {AI}},
	author       = {Jacovi, Alon and Marasovi{\'c}, Ana and Miller, Tim and Goldberg, Yoav},
	year         = 2021,
	booktitle    = {Proceedings of the 2021 ACM conference on fairness, accountability, and transparency},
	pages        = {624--635}
}

@article{wiedeman2020modeling,
	title        = {Modeling of moral decisions with deep learning},
	author       = {Wiedeman, Christopher and Wang, Ge and Kruger, Uwe},
	year         = 2020,
	journal      = {Visual Computing for Industry, Biomedicine, and Art},
	publisher    = {Springer},
	volume       = 3,
	number       = 1,
	pages        = 27
}

@article{ueshima2024discovering,
	title        = {Discovering Novel Social Preferences Using Simple Artificial Neural Networks},
	author       = {Ueshima, Atsushi and Takikawa, Hiroki},
	year         = 2024,
	journal      = {Collabra: Psychology},
	publisher    = {University of California Press},
	volume       = 10,
	number       = 1
}

@article{tversky1990causes,
	title        = {The causes of preference reversal},
	author       = {Tversky, Amos and Slovic, Paul and Kahneman, Daniel},
	year         = 1990,
	journal      = {The American Economic Review},
	publisher    = {JSTOR},
	pages        = {204--217}
}

@article{yu2022neural,
	title        = {Neural and cognitive signatures of guilt predict hypocritical blame},
	author       = {Yu, Hongbo and Contreras-Huerta, Luis Sebastian and Prosser, Annayah MB and Apps, Matthew AJ and Hofmann, Wilhelm and Sinnott-Armstrong, Walter and Crockett, Molly J},
	year         = 2022,
	journal      = {Psychological Science},
	publisher    = {Sage Publications Sage CA: Los Angeles, CA},
	volume       = 33,
	number       = 11,
	pages        = {1909--1927}
}

@book{luce1959individual,
	title        = {Individual choice behavior},
	author       = {Luce, R Duncan and others},
	year         = 1959,
	publisher    = {Wiley New York},
	volume       = 4
}

@article{zaheer2017deep,
	title        = {Deep sets},
	author       = {Zaheer, Manzil and Kottur, Satwik and Ravanbakhsh, Siamak and Poczos, Barnabas and Salakhutdinov, Russ R and Smola, Alexander J},
	year         = 2017,
	journal      = {Advances in neural information processing systems},
	volume       = 30
}

@article{ganzach2001nonlinear,
	title        = {Nonlinear models of clinical judgment: Communal nonlinearity and nonlinear accuracy},
	author       = {Ganzach, Yoav},
	year         = 2001,
	journal      = {Psychological Science},
	publisher    = {SAGE Publications Sage CA: Los Angeles, CA},
	volume       = 12,
	number       = 5,
	pages        = {403--407}
}

@article{zorman1997limitations,
	title        = {The limitations of decision trees and automatic learning in real world medical decision making},
	author       = {Zorman, Milan and {\v{S}}tiglic, Milojka Molan and Kokol, Peter and Mal{\v{c}}i{\'c}, Ivan},
	year         = 1997,
	journal      = {Journal of medical systems},
	publisher    = {Springer},
	volume       = 21,
	number       = 6,
	pages        = {403--415}
}

@article{payne1988adaptive,
	title        = {Adaptive strategy selection in decision making.},
	author       = {Payne, John W and Bettman, James R and Johnson, Eric J},
	year         = 1988,
	journal      = {Journal of experimental psychology: Learning, Memory, and Cognition},
	publisher    = {American Psychological Association},
	volume       = 14,
	number       = 3,
	pages        = 534
}

@article{rosenfeld2012combining,
	title        = {Combining psychological models with machine learning to better predict people’s decisions},
	author       = {Rosenfeld, Avi and Zuckerman, Inon and Azaria, Amos and Kraus, Sarit},
	year         = 2012,
	journal      = {Synthese},
	publisher    = {Springer},
	volume       = 189,
	number       = {Suppl 1},
	pages        = {81--93}
}

@article{dawes1979robust,
	title        = {The robust beauty of improper linear models in decision making.},
	author       = {Dawes, Robyn M},
	year         = 1979,
	journal      = {American Psychologist},
	volume       = 34,
	number       = 7
}

@article{pignatiello2020decision,
	title        = {Decision fatigue: A conceptual analysis},
	author       = {Pignatiello, Grant A and Martin, Richard J and Hickman Jr, Ronald L},
	year         = 2020,
	journal      = {Journal of health psychology},
	publisher    = {SAGE Publications Sage UK: London, England},
	volume       = 25,
	number       = 1,
	pages        = {123--135}
}

@article{hastie2017generalized,
	title        = {Generalized additive models},
	author       = {Hastie, Trevor J},
	year         = 2017,
	journal      = {Statistical models in S},
	publisher    = {Routledge},
	pages        = {249--307}
}

@article{thurstone1927law,
	title        = {A law of comparative judgment},
	author       = {Thurstone, LL},
	year         = 1927,
	journal      = {Psychological Review},
	publisher    = {Psychological Review Company},
	volume       = 34,
	number       = 4,
	pages        = 273
}

@article{mosteller1951remarks,
	title        = {Remarks on the method of paired comparisons: I. The least squares solution assuming equal standard deviations and equal correlations},
	author       = {Mosteller, Frederick},
	year         = 1951,
	journal      = {Psychometrika},
	publisher    = {Springer-Verlag},
	volume       = 16,
	number       = 1,
	pages        = {3--9}
}

@article{carnal1989decomposition,
	title        = {On a decomposition problem for multivariate probability measures},
	author       = {Carnal, H and Dozzi, M},
	year         = 1989,
	journal      = {Journal of multivariate analysis},
	publisher    = {Elsevier},
	volume       = 31,
	number       = 2,
	pages        = {165--177}
}

@article{ewerhart2025possibility,
	title        = {On the (im-)possibility of representing probability distributions as a difference of {i.i.d.} noise terms},
	author       = {Ewerhart, Christian and Serena, Marco},
	year         = 2025,
	journal      = {Mathematics of Operations Research},
	publisher    = {INFORMS},
	volume       = 50,
	number       = 1,
	pages        = {390--409}
}

@book{vonNeumann1944theory,
	title        = {Theory of Games and Economic Behavior},
	author       = {von Neumann, John and Morgenstern, Oskar},
	year         = 1944,
	publisher    = {Princeton University Press}
}

@article{tversky1969intransitivity,
	title        = {Intransitivity of preferences},
	author       = {Tversky, Amos},
	year         = 1969,
	journal      = {Psychological Review},
	volume       = 76,
	number       = 1,
	pages        = {31--48}
}
